\documentclass[Afour,sageh,times]{sagej}

\usepackage[bookmarksopen,bookmarksnumbered]{hyperref}

\usepackage{moreverb,url}

\usepackage{amsmath}
\usepackage{amssymb}
\usepackage{amstext}

\usepackage{color}
\usepackage{url}
\usepackage{overpic}

\renewcommand{\vec}[1]{ {\mathbf{#1}}}
\usepackage{multirow}
\newcommand{\btau}   {\mbox{\boldmath $\tau$}}

\newcommand{\T}{^\top}

\newcommand{\bb}{\mbox{\boldmath $b$}}

\newcommand{\bzero}{\mbox{\boldmath $0$}}

\usepackage{mdframed}

{\end{tabbing} \end{mdframed} \vspace{5px}}

\newcommand{\vc}{\vec{c}}
\newcommand{\vg}{\vec{g}}

\newcommand{\vp}{\vec{p}}

\newcommand{\vI}{\vec{I}}

\newcommand{\vk}{\vec{k}}

\newcommand{\BM}{\begin{bmatrix}}
\newcommand{\EM}{\end{bmatrix}}
\newcommand{\vPhi}{ \boldsymbol{\Phi}}

\newcommand{\vphi}{\boldsymbol{\phi}}

\newcommand{\bone}{{\bf 1}}
\newcommand{\vzero}{{\vec{0}}}

\newcommand{\vA}{\vec{A}}
\newcommand{\va}{\vec{a}}
\newcommand{\vB}{\vec{B}}

\newcommand{\vC}{\vec{M}}
\newcommand{\vf}{\vec{f}}

\newcommand{\vG}{\vec{G}}

\newcommand{\vH}{\vec{H}}

\newcommand{\vJ}{\vec{J}}

\newcommand{\vK}{\vec{K}}

\newcommand{\vM}{\vec{M}}
\newcommand{\vN}{\vec{N}}

\newcommand{\vW}{\vec{W}}

\newcommand{\vY}{\vec{Y}}
\newcommand{\vq}{\vec{q}}
\newcommand{\vpi}{\boldsymbol{\pi}}

\newcommand{\vqd}{\dot{\vq}}

\newcommand{\vqdd}{\ddot{\vq}}

\newcommand{\qd}{\dot{q}}
\newcommand{\qdd}{\ddot{q}}

\newcommand{\vR}{\vec{R}}
\newcommand{\Rot}[2]{{}^{#1}\vR_{{\mkern .5\thinmuskip} #2}}

\newcommand{\vS}{\vec{S}}

\newcommand{\vT}{\vec{T}}

\newcommand{\vv}{\vec{v}}
\newcommand{\vV}{\vec{V}}

\usepackage[table]{xcolor}
\newcommand{\subsub}[1]{\vspace{1ex}
\noindent{\bf #1}:}

\newcommand{\ba}{{\bf a}}

\newcommand{\vX}{\vec{X}}

\newcommand{\XJ}{{\vX\!_J}}

\newcommand{\XM}[2]{{}^{#1}\vX_{#2}{\mkern-.75\thinmuskip}}

\newcommand{\TM}[2]{{}^{#1}\vT_{#2}{\mkern-.75\thinmuskip}}

\newcommand{\XMT}[2]{{}^{#1}\vX_{#2}\T}

\newcommand{\vx}{\vec{x}}
\newcommand{\vy}{\vec{y}}

\newcommand{\beq}{\begin{equation}}
\newcommand{\eeq}{\end{equation} }
\usepackage{tikz}
\newcommand{\Nw}[1]{n^{w}_{#1}}
\newcommand{\Nv}[1]{n^{v}_{#1}}

\newcommand{\TsetInertia}{\overline{\mathcal{T}}}
\newcommand{\NsetInertia}{\overline{\mathcal{N}}}

\newcommand{\MomentumCondition}[1]{\vC(\vV_{#1}, \vPhi_{#1})\T \dpi_{#1}}
\newcommand{\MomentumColumns}[1]{\vC(\vV_{#1}, \vPhi_{#1})}
\newcommand{\MomentumRows}[1]{\vC(\vV_{#1}, \vPhi_{#1})\T}
\newcommand{\ObservabilityColumns}[1]{\beforeJoint{ \vK}{#1}}
\newcommand{\ObservabilityRows}[1]{\beforeJoint{ \vK}{#1}\T}

\definecolor{colork}{RGB}{240, 240, 240}
\newcommand{\algdesc}[1]{ \textcolor{white}{\colorbox{colork}{\textcolor{black}{\em #1}}} }

\newcommand{\Bi}{ {}^{\pred{i}}\vB_{\hspace{0.05em} i} }
\newcommand{\Bj}{ {}^{\pred{j}}\vB_{\hspace{0.05em} j} }

\newcommand{\vKJiT}{ \vKJi\T }

\newcommand{\vVJi}{\wellBeforeJoint{\vV}{i}}
\newcommand{\vVim}{\beforeJoint{ \vV}{i}}

\newcommand{\vKJi}{\wellBeforeJoint{\vK}{i}}
\newcommand{\vKim}{\beforeJoint{ \vK}{i}}

\newcommand{\beforeJoint}[2]{
#1_{#2^-} 
}

\newcommand{\wellBeforeJoint}[2]{
#1_{\pred{#2}^+} 
}

\newcommand{\Iset}{\mathbb{I}}

\usepackage{empheq}
\usepackage{amsmath}
\usepackage{amsthm}
\usepackage{amssymb}
\usepackage{color}
\usepackage{mathrsfs}
\usepackage{nicefrac}
\usepackage{enumitem}
\usepackage{verbatim}
\usepackage{theoremref}

\usepackage{lipsum}

\usepackage{graphicx}
\usepackage{empheq}
\usepackage{amsthm}
\usepackage{bbding}
\usepackage{cite}
\usepackage{colortbl}
\usepackage{pifont}% http://ctan.org/pkg/pifont

\graphicspath{{./Figures/}{./}}

\usepackage[linesnumbered,boxruled]{algorithm2e}

\usepackage{setspace}

\newcommand{\pred}[1]{ #1{\text-}1 }

\newcommand{\nbod}{N}

\newcommand{\Ichange}[2]{ (#2 \times)\T\, #1 + #1 \, (#2 \times)  }

\usepackage[normalem]{ulem}       % for strikeout \sout

\newcommand{\vvJi}{\vv_{\pred{i}}}

\newcommand{\qb}{q_\jb}
\newcommand{\qdb}{\qd_\jb}

\newcommand{\Xba}{\XM{\bb}{\ba}}
\newcommand{\Xbaq}{\Xba(\qb)}
\newcommand{\Xbaz}{\Xba(0)}

\newcommand{\XTba}{\XMT{\bb}{\ba}}
\newcommand{\XTbaq}{\XTba(\qb)}
\newcommand{\XTbaz}{\XTba(0)}

\newcommand{\vPhib}{\vPhi_\bb}
\newcommand{\dIb}{\dI_\bb}
\newcommand{\dIa}{\dI_\ba}

\newcommand{\vvb}{\vv_\bb}
\newcommand{\vva}{\vv_\ba}

\renewcommand{\ba}{{\pred{i}}}
\renewcommand{\bb}{i}
\newcommand{\jb}{i}

\newcommand{\Wspan}[1]{\mathcal{W}_{#1}}
\newcommand{\Vspan}[1]{\mathcal{V}_{#1}}
\newcommand{\Vset}[1]{\mathcal{V}_{#1}^*}

\newcommand{\dI}{\delta \vI}
\newcommand{\dpi}{\delta \vpi}

\usepackage{tikz}
\usepackage{pgffor}
\usepackage{diagbox}

\DeclareMathOperator{\spn}{span}
\DeclareMathOperator{\trace}{Trace}

\newtheorem{theorem}{Theorem}
\newtheorem{lemma}{Lemma}

\newtheorem{innercustomthm}{Theorem}
\newenvironment{customthm}[1]
  {\renewcommand\theinnercustomthm{#1}\innercustomthm}
  {\endinnercustomthm}

\newtheorem{proposition}{Proposition}
\newtheorem{remark}{Remark}

\usepackage{xspace}
\newcommand{\mapping}{mapping\xspace}
\newcommand{\Mapping}{Mapping\xspace}

\usepackage{scalerel}

\newcommand{\vel}{\mbox{\boldmath $v$}}

\renewcommand{\XJ}[1]{\XM{#1}{\pred{#1}}(0)}
\newcommand{\XJT}[1]{\XMT{#1}{\pred{#1}}(0)}
\newcommand{\DI}[1]{\boldsymbol{\Delta}_{#1}}
\title{A Geometric Characterization of Observability in Inertial Parameter Identification }

\author{Patrick~M.~Wensing, G{\"u}nter  Niemeyer, Jean-Jacques E. Slotine%\\[-5ex]~
                                                                    }
\definecolor{light-gray}{gray}{0.95}

\begin{document}

\newbox\one
\newbox\two
\long\def\loremlines#1{%
    \setbox\one=\vbox{
    \vspace{3px}
      \lipsum\footnote{Another footnote.}%
     }
   \setbox\two=\vsplit\one to #1\baselineskip
   \unvbox\two}

%!TEX root = MinimalParameters_v5.tex

 \begin{abstract}

This paper presents an algorithm to geometrically characterize inertial parameter identifiability for an articulated robot. 
The geometric approach tests identifiability across the infinite space of configurations using only a finite set of conditions and without approximation.  It
can be applied to general open-chain kinematic trees ranging from industrial manipulators to legged robots, and it is the first solution for this broad set of systems that is provably correct.  The high-level operation of the algorithm is based on a key observation: Undetectable changes in inertial parameters can be represented as sequences of inertial transfers across the joints. Drawing on the exponential parameterization of rigid-body kinematics, undetectable inertial transfers are analyzed in terms of observability from linear systems theory. This analysis can be applied recursively, and lends an overall complexity of $O(N)$ to characterize parameter identifiability for a system of $N$ bodies. {\sc Matlab} source code for the new algorithm is provided.

%We present an algorithm to characterize the space of identifiable inertial parameters in system identification of an articulated robot. This problem has been considered in the literature for decades; however, existing methods suffer from common drawbacks. Methods either rely on symbolic techniques that are not general, or rely on numerical techniques that are not provably correct. The contribution of this work is to propose a new recursive algorithm for this problem with a provably correct output. This Recursive Parameter Nullspace Algorithm (RPNA) can be applied to general open-chain kinematic trees and does not rely on symbolic techniques. Drawing on the exponential parameterization of rigid-body kinematics, classical linear controllability and observability results are applied to analyze inertial parameter identifiability. The high-level operation of the RPNA is based on a key observation --  undetectable changes in inertial parameters can be interpreted as sequences of inertial transfers across the joints. This observation can be applied recursively, and lends an overall complexity of $O(N)$ to determine the minimal parameters for a system of $N$ bodies.  {\sc Matlab} source code for the new algorithm is provided.
\vspace{-10px}
\end{abstract}

\maketitle
% \begin{IEEEkeywords}
% Dynamics, Calibration and Identification
% \end{IEEEkeywords}

% wing to these drawbacks, many rules of thumb exist in the literature for special articulation topologies appearing commonly in manipulators.

%\newcommand{\citep}[1]{\cite{#1}}
%\newcommand{\citet}[1]{\cite{#1}}
\renewcommand{\cite}[1]{\citep{#1}}

\SetAlgoNoLine

%\todo{Referne Nori, slotine li}

\section{Introduction}

A classic problem in robotics is the identification of inertial parameters (mass, center of mass, and inertia) for each link of a mechanism. This problem has received attention through multiple decades, with original work on the identification of manipulators \citep{Atkeson86,Khalil95,Swevers97} seeing extensions to the identification of mobile robots and humans in more recent applications \citep{Ayusawa08,Ayusawa10,Ayusawa14,Traversaro13,Jovan16,lee2020geometric}. 
Across domains, an enabling property is that the inverse dynamics of a rigid-body system are linear in its inertial parameters, motivating least-squares solutions to identify parameters from the measurement of joint kinematics, joint torques, and external forces. 

It is well known, however, that not all inertial parameters are identifiable from these measurements \cite{Atkeson86}. This observation has motivated studying which parameters (or combinations thereof) can theoretically be identified from an infinite amount of data \cite{Gautier90b,Mayeda90,sheu1991identifying,Kawasaki91,Khalil95,Chen02,Ros12,Ayusawa14}. This problem is one of characterizing the {\em structural identifiability} \cite{Bellman70} of a model form. It is emphasized that this problem is distinct from the problem of fitting a model to experimental data. Naturally, the quality of any identified model depends both on the accuracy of the model form assumed and the quality of the data collected. Thus, the related problem of {\em practical identifiability} addresses how uncertainty in measurements relates to uncertainty in the identified parameters (e.g., \cite{Calafiore00,Poignet05}). In short, structural identifiability considers what model properties can be identified from data, whereas practical identifiability considers how non-ideal aspects of that data impact the inferred model.  

This paper presents new tools for characterizing the structural identifiability of a rigid-body system via a geometric approach to the problem. We provide an algorithm that takes a kinematic model as an input, and returns the linear combinations of the system's inertial parameters that are identifiable from measurements of the joint kinematics, joint torques, and external forces. The approach taken offers an alternative to previous numerical (e.g., \cite{Gautier91b}) or symbolic approaches (e.g., \cite{Khosla89,Gautier90b}) to the problem. The main benefit of our approach is that the algorithm requires only a finite set of conditions to test the infinite space of configurations.  It also 
comes with a theoretical proof of correctness that holds for arbitrary open-chain fixed- or floating-base systems with generic joint types. This proof of correctness is achieved by working with spatial (6D) inertias in a parameterization-free sense and only adopting inertial parameters for the implementation of the theory.

We show that undetectable changes in the inertial parameters can be represented by exchanges of mass and inertia between neighboring bodies. Exchanges are considered undetectable if they leave the system dynamics unchanged across the entire state space. 
%This condition requites that exchanges at each joint must not change the associated joint torque, and further must not change the total kinetic energy. These conditions must hold no matter how the robot is moving. 
Addressing this infinity of possibilities is the main hurdle that has previously prevented a rigorous treatment of structural identifiability for general robot systems.
%We leverage the exponential parameterization of rigid body kinematics coupled with geometric characterizations of observability from linear systems to address this challenge. 
%Once undetectable changes are fully characterized, the identifiable parameter combinations are readily established.

%, and has likewise enabled adaptive robot control schemes \citep{Slotine87,Slotine89}.

% \begin{table*}
% \center
% \vspace{6px}
% \scaleobj{.85}{
% \begin{tabular}{ccccccc}
% \hline 
%                        &      & General   & Fixed   & Floating & Closed & Provably \\
%                        & Type & Joints    & Base   & Base     & Chains & Correct  \\ \hline
% \citet{Ayusawa14}   & Symbolic  & \checkmark     &    & \checkmark  & & \checkmark          \\
% \rowcolor{light-gray}  \citet{Gautier90b}& Symbolic &  & \checkmark &  & & \\
% \citet{Gautier91b} & Numeric & \checkmark  & \checkmark & \checkmark & \checkmark & \\
% \rowcolor{light-gray}  \citet{Ros15} & Symbolic &  & \checkmark &  & \checkmark & \\ 

% {\bf This Paper	}		    & \bf{Geometric} & \pmb{\checkmark} & \pmb{\checkmark} & \pmb{\checkmark} & Rotors Only & \pmb{\checkmark} \\ \hline\\[-1ex]
% \end{tabular}
% }
% %\caption{Comparison of features amongst parameter identification optimization problems formulated in the literature. }
% \caption{Feature comparison. \revg{[SHALL WE LEAVE OUT THE CLOSED CHAIN COLUMN?]}}
% %
% \label{tab:compare}
% %\vspace{-10px}
% \end{table*}

\begin{table*}
\center
\vspace{6px}
\scaleobj{.85}{
\begin{tabular}{cccccc}
\hline 
                       &      & General   & Fixed   & Floating  & Provably \\
                       & Type & Joints    & Base   & Base     &  Correct  \\ \hline
\rowcolor{light-gray} \quad \citet{Niemeyer90,niemeyer1991performance}  & \hspace{1em} Geometric \hspace{1em} & & & \checkmark & \checkmark \\
\citet{Ayusawa14}   & Symbolic  & \checkmark     &    & \checkmark  & \checkmark          \\
\rowcolor{light-gray}  \citet{Gautier90b}& Symbolic &  & \checkmark &  & \\
\citet{Gautier91b} & Numeric & \checkmark  & \checkmark & \checkmark & \\
\rowcolor{light-gray}  \citet{Ros15} & Symbolic &  & \checkmark &  & \\ 
{\bf This Paper	}		    & \bf{Geometric} & \pmb{\checkmark} & \pmb{\checkmark} & \pmb{\checkmark} & \pmb{\checkmark} \\ \hline\\[-1ex]
\end{tabular}
}
%\caption{Comparison of features amongst parameter identification optimization problems formulated in the literature. }
\caption{Feature comparison.}
\label{tab:compare}
%\vspace{-10px}
\end{table*}

\subsection{Related previous work}

\noindent {\bf Motivation:} 
Early work on system identification emphasized determining minimal sets of parameters (base parameters) to be identified \cite{Gautier91b}. Base parameters group together linear combinations of link parameters that appear together in the equations of motion. Isolating these groupings allows using a reduced parameter set for faster dynamics computations \cite{khalil1987minimum}, and enables identification to have a unique solution. 

Since this work back in the 1980s and 90s, increases in computing power and new identification methodologies suggest revisiting these original motivations. Regarding computation power, the Recursive-Newton-Euler Algorithm (RNEA) 
can now be carried out with the full parameter set in microseconds for complex systems  \cite{carpentier2019pinocchio}. In terms of methodology, recent advances in enforcing physical consistency of the parameters \cite{Sousa14,Traversaro13, WensingKimSlotine17} and geometric regularization \cite{Lee18,lee2020geometric} jointly suggest the benefits of considering the full parameter set when carrying out identification.  

Despite these recent advances, structural identifiability considerations remain fundamental for robot system identification. Recall that structural identifiability analysis characterizes which parameters (or combinations) can be deduced from an {\em infinite} amount of training data. Methods that design exciting trajectories \citep{Swevers97,Gautier92,calafiore2001robot,Jovic15,lee2021optimal} seek to maximally identify model information with a {\em finite} amount of training data. It is only possible to certify that a given dataset is maximally exciting, however, via comparison to a structural identifiability analysis. Another motivation for characterizing base parameters comes from instrumental variable identification techniques \citep{Janot14,Janot14b} that address model bias from noisy data, but currently require the use of a non-redundant model parameterization. Methods that characterize uncertainty in the parameter estimates likewise only do so through considering uncertainty in the base parameters \citep{Calafiore00,Poignet05}.   Yet other threads of work have exploited how the full parameters map onto a base parameter set for payload identification \citep{gaz2017payload}, or have sought to invert this relationship for inferring the full parameters from the base set \citep{gaz2016extracting}.

As a practical matter, mobile legged systems are often identified in restricted setups (e.g., with the torso fixed \cite{WensingKimSlotine17,Focchi17b,Bonnet18}), and so it is important to understand how these setups affect identifiability. When the accuracy of certain joint torques is affected by these restricted setups, it also has implications for which joints are best used for proprioceptive detection of external contacts \citep{de2006collision,haddadin2017robot,bledt2018contact}.

Considerations of structural identifiability likewise find applications in the area of adaptive control (e.g., \cite{slotine1987adaptive,pucci2015collocated,wang2012recursive}), where so-called persistency of excitation (PE) conditions (i.e., ensuring maximal excitation over recurring finite intervals over time) can only be satisfied when adopting base-parameters. 
These considerations hold regardless of whether one adopts a direct adaptive control strategy (e.g., \citep{slotine1987adaptive,chung2009cooperative,pucci2015collocated,garofalo2021adaptive}), an indirect adaptive control strategy \cite{li1989indirect}, or a composite of the two \citep{slotine1989composite,wang2013recursive,o2022neural}. More recently, less strict alternatives to PE conditions, known as interval excitation (IE) conditions, have been proposed \citep{pan2018composite} (see, e.g., \citet{ortega23} for broader context beyond robot control). Again, however, IE conditions can only be satisfied when adopting a minimal parameterization.

\noindent {\bf Approaches:}
Previous methods for characterizing identifiability
include symbolic and numerical approaches. Some symbolic approaches consider finding common groupings of parameters within the equations of motion \cite{khosla1986real, khalil1986automatic,Khosla89}, which can become untenable at a system-level scale due to the complexity of the equations of motion. Thus, other work focused on carrying out regroupings of parameters by hand for special cases such as revolute manipulators with parallel or perpendicular joints \cite{Gautier90b,Mayeda90,Kawasaki91}. Many special cases still have to be considered separately, for example, whether joints are revolute or prismatic, parallel or orthogonal to each other, parallel or orthogonal to gravity, and combinations of the above. When regroupings are missed in the application of these rules, it leads one to believe that there are more base parameters than there truly are.
%Each time a new symbolic grouping is isolated, it reduces the upper bound on the number of base parameters of the model. If all groupings are discovered, then the identifiability of the model is correctly characterized. Yet, if regroupings are missed, it leads one to believe that more base parameters are uniquely identifiable than actually are. The complexity of the many cases to be considered makes arriving at the correct result a non-trivial task. %even when the model fits into the restricted classes covered by these methods. 
By comparison, the approach herein borrows inspiration from these methods but does not require any special cases due to the generality of its geometric treatment of the problem.

On the flip side, numerical methods for assessing identifiability (e.g., \cite{Gautier91b}) provide a lower bound on the number of base parameters of a model. Such methods often generate a finite set of random data, which is assumed to be maximally exciting. If the data is indeed maximally exciting, and there are no numerical issues, then correct conclusions can be drawn regarding the identifiability of the model. In cases when the data is not maximally exciting, the number of base parameters is underestimated. Since numeric methods provide a lower bound on the number of base parameters, while symbolic methods provide an upper bound, they can be combined together  \cite{Gautier91b} to mutually certify their outputs. 
Overall, neither the symbolic nor numeric approaches individually are provably correct for general mechanisms.

A complementary approach to
finding common parameter groupings is to consider inconsequential/undetectable transfers of inertia between pairs of bodies across each joint.  This strategy was developed originally by \citet{Niemeyer90} for use with floating-base systems with revolute and prismatic joints. It was later independently discovered by Chen and colleagues for 2D \citep{Chen02} and 3D \citep{Chen02b} mechanisms, and developed further by \citet{Ros12,Ros15} and \citet{ Iriarte13}. These latter papers represent the state of the art in base parameter determination. However, much like original symbolic approaches (e.g., \cite{Gautier90b}) they require a skilled individual to exercise discretion in applying a set of special rules for determining identifiable parameters of links with motion restrictions close to the ground. By comparison, our algorithm is fully automatic. The user provides a model description (e.g., specifying the kinematic data found in a URDF file), and our algorithm provides a provably correct description of which parameter combinations are identifiable, and which are not.

The closest work is from \citet{niemeyer1991performance} and \citet{Ayusawa14}, where they carried out provably correct identifiability analysis for floating-base robots. Results from \citet{Ayusawa14} treated general joint models, and also contributed a surprising result regarding the ability to identify floating-base models without measurement of joint torques. Beyond our extensions to fixed-base robots, we revisit the main theorem of \citet{Ayusawa14} to provide a shorter proof from a new angle.

\subsection{Contribution}

The main contribution of this paper is the first provably correct algorithm to characterize the identifiable inertial parameter combinations for general fixed- and floating-base open-chain systems (Tab.~\ref{tab:compare}). The algorithm is named the Recursive-Parameter-Nullspace Algorithm (RPNA) and has a structure reminiscent of the outward kinematics pass of the RNEA. Rather than computing the velocities of each link on the outward pass, we geometrically characterize all possible velocities of each link. This information enables the algorithm to automatically detect motion restrictions for each body and to assess how they influence parameter identifiability. In this regard, we provide a modern update to past symbolic identifiability work, which leverages a geometric treatment to avoid the many special cases previously needed. 

%, and correspondingly to set conditions on inertia transfers around each joint. 
%More specifically, by drawing upon the exponential characterization of kinematics, the approach employs controllability/reachability analysis from linear systems theory to characterize possible motions experienced by each link.
%Drawing on the exponential characterization of kinematics, the approach employs controllable subspaces from linear systems theory to characterize achievable motions, 
%It then uses observability analysis to determine exchanges of mass and inertia between pairs of bodies that can be carried out without impact on the system dynamics. Chaining these results recursively allows us to provably describe the identifiable parameter combinations for the system as a whole.

%Our theoretical development proceeds parameterization-free while employing inertial parameters only in implementation. Correspondingly, as a minor contribution, we provide an updated treatment of identifiability employing tools from spatial vector algebra \cite{Featherstone08}, which can be readily translated to Lie group/algebra notation. We hope that this more modern treatment will be an asset to the reader toward more rapidly interpreting previous strategies (e.g., \cite{Gautier91b,Mayeda90}, etc.) that focused on detailed symbolic expressions. 

\subsection{Organization}

The paper is organized as follows. Section \ref{sec:Concepts} introduces the main concepts of an inertial transfer and how it relates to identifiability. We then develop the identifiability theory by focusing on a pair of bodies connected via a joint, in cases where the combined bodies can move freely in 6D space (Sec.~\ref{sec:Floating_Two_Body}) vs. when one body has additional motion restrictions (Sec.~\ref{sec:Two_Fixed}). We develop this theory in a parameterization-free sense, but rephrase the results in terms of inertial parameters in Sec.~\ref{sec:Parameters}.
These developments then enable a recursive treatment of identifiability in chains of bodies in Sec.~\ref{sec:fb}, where we present the RPNA. 
Extensions of the basic algorithm 
to kinematics trees, multi-DoF joints, and floating-base systems 
are provided in Sec.~\ref{sec:extensions}. Sec.~\ref{sec:closedChain} provides a further extension to joint motors, an example of simple closed chains.
Results in Sec.~\ref{sec:results} consider system-level identifiability for classical manipulators, the PUMA \& SCARA, as well as a mobile robot, the MIT Cheetah 3. As a key practical takeaway, we illustrate the pitfalls of constrained identification experiments often used to identify individual limbs of mobile legged systems. Concluding comments are provided in Sec.~\ref{sec:conclusions}.

\section{Main Concepts}
\label{sec:Concepts}

\noindent {\bf Dynamics and Unidentifiable Parameters:}
We consider identifying a rigid-body robot with $N$ bodies whose joint-space dynamics take the standard form
\begin{equation}
 \vH(\vq)\, \vqdd  + \vc(\vq,\vqd) + \vg(\vq) = \btau_{\rm total}
\label{eq:eom}
\end{equation}
with $\vq$ the configuration variable, $\vH$ the mass matrix (also known as the joint-space inertia matrix), $\vc$ and $\vg$ the Coriolis and gravity forces, and $\boldsymbol{\tau}_{\rm total}$ the total generalized force
summing contributions from joint torques and any external contact forces.
It is well known that 
$\vH$, $\vc$, and $\vg$ 
can be expressed linearly in the inertial parameters $\vpi \in \mathbb{R}^{10 N}$ of the bodies, which include masses, first moments, and rotational inertias.
Thus, the dynamics \eqref{eq:eom} can be written as
\begin{equation}
\vY(\vq, \vqd, \vqdd)\, \vpi = \btau_{\rm total}
\label{eq:regressor_main}
\end{equation}
where $\vY$ is the classical regressor \citep{Atkeson86}.

Unidentifiable parameters of the system are then given by those that don't affect the generalized force in any case
\begin{align}
\mathcal{N} &= \{\dpi \in \mathbb{R}^{10 \nbod}~|~ \vY(\vq,\vqd, \vqdd) \, \dpi = \bzero, ~\forall\,\vq,\vqd,\vqdd \} \nonumber
\end{align}
For any parameter variations $\dpi \in \mathcal{N}$, we say that the change $\dpi$ does not affect the dynamics, or equivalently that it is not identifiable through measurement of the total generalized force $\boldsymbol{\tau}_{\rm total}$.
Likewise, 
$\dpi$ is undetectable from the measurement of the joint torques and any external contact forces in this case.

%\revg{
Since the first two terms in \eqref{eq:eom} are determined by the kinetic energy, we will depart from these vector equations and instead focus on scalar energy $T=\frac{1}{2} \vqd\T \vH(\vq) \vqd$.
We temporarily dismiss potential energy before introducing gravity in Sec.~\ref{sec:gravity} as an upward acceleration.
Rewriting $T= \vY_T(\vq, \vqd)\, \vpi$, we can then equivalently express
\begin{equation}
\mathcal{N} = \{\dpi \in \mathbb{R}^{10 \nbod}~|~ \vY_T(\vq,\vqd) \, \dpi = \bzero, ~\forall\,\vq,\vqd \} \nonumber
\end{equation}
%}

\begin{figure}
\center
\includegraphics[width=\columnwidth]{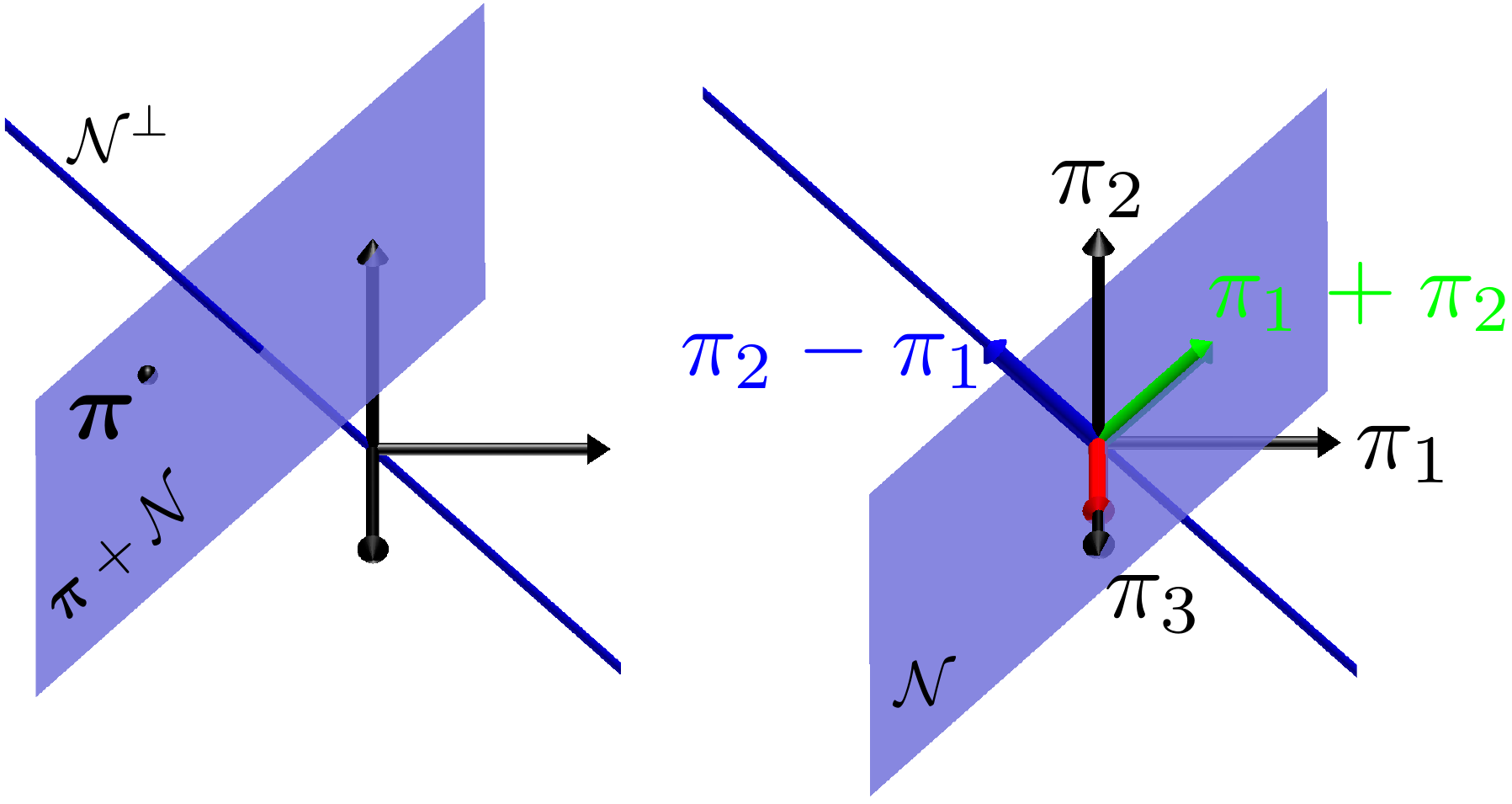}
\caption{(left) Example 3D parameter space showing the nominal parameter values $\vpi$. Any parameter vectors in the affine (i.e., offset) subspace $\vpi + \mathcal{N}$ give the same model. (right) Variations $\dpi$ split across unidentifiable  $\mathcal{N}$ (2D) and identifiable  $\mathcal{N}^\perp$ (1D) subspaces.  \\
\quad $\begin{array}{
r@{\ :\quad\dpi=[\ }c@{\quad}c@{\quad}c@{\quad}c@{\ ...]\ \in\ }l}
\pi_3           & 0 & 0 & 1 & 0 & \mathcal{N} \\
(\pi_1 + \pi_2) & \phantom{-} 1 & 1 & 0 & 0 & \mathcal{N} \\
(\pi_2 - \pi_1) & -1 & 1 & 0 & 0 & \mathcal{N}^\perp
\end{array}$}
\label{fig:pispace}
\end{figure}

\subsub{Identifiable Parameters and Combinations}
Generally, any variation $\dpi$ may affect a combination of parameters. For illustration, consider the $\vpi \in \mathbb{R}^{3}$ parameter space shown in Figure~\ref{fig:pispace}, decomposed into the unidentifiable subspace $\mathcal{N}$ as well as the identifiable orthogonal subspace $\mathcal{N}^\perp$. Axis $\pi_3$ lies within $\mathcal{N}$, making the parameter $\pi_3$ unidentifiable. But parameters $\pi_1$ and $\pi_2$ are combined: $(\pi_1\!+\!\pi_2)$ is unidentifiable while $(\pi_2\!-\!\pi_1)$ is identifiable. We can also equally imagine a $\pi_4$ axis within $\mathcal{N}^\perp$ so that $\pi_4$ is identifiable by itself.  So, any individual parameter may be identifiable in isolation, identifiable in combination, or unidentifiable.

Mathematically, if we can characterize $\mathcal{N} = {\rm Null}(\vN)$ as the nullspace of some matrix $\vN$, then the subspace $\mathcal{N}^\perp$ orthogonal to $\mathcal{N}$ is given by $\mathcal{N}^\perp = {\rm Range}( \vN\T)$ such that the rows of $\vN$ give the identifiable parameter combinations.

\subsub{Interpretation and Parameter Transfers}
The set $\mathcal{N}$ is a linear subspace of $\mathbb{R}^{10 \nbod}$, and any basis for it will generally consist of combinations of parameters from multiple bodies. In deriving a basis for $\mathcal{N}$, we will show it is enough to consider combinations of parameters from only two neighboring bodies at a time, and that these combinations represent a parameter transfer between the bodies.

Indeed, the approach derives two criteria on $\dpi$, which we interpret as two identification mechanisms:
%\cang{Toward arriving at this result, our approach will consider the parameters of a body as being identified via one of two mechanisms:}
%\revg{
(1) how a body's momentum is projected on the preceding joint (and hence identifiable via joint torques), 
%(1) how a body directly affects the body's momentum associated with the preceding joint and hence identifiable via joint torques
or (2) how a body's inertia combines with preceding (parent) bodies and whether it remains identifiable in the conglomeration of bodies.  For the latter, we notice the child body's parameters may be identified if they add to a parent differently, depending on joint configurations.  But if the child parameter adds to the parent in a fixed manner, independent of configuration, it remains undetectable via (2).

Figure \ref{fig:invariance} illustrates this idea at a high level, showing three potential point
%\cang{two}
masses to be attached to body 2.  Each point mass affects the body's total mass, center of mass, and inertia and presents a physical interpretation of a particular $\dpi$.  Adding either $m_b$ or $m_c$ ($\dpi_b$ or $\dpi_c$) affects the angular momentum seen on the joint (mechanism (1)), but, in the absence of gravity, the rotational joint alone can only determine their sum and not distinguish between them.  However, considering the conglomeration of both bodies, we see $m_b$ vs. $m_c$ at different locations with respect to body 1 based on the joint position, and hence can identify them distinctly (mechanism (2)).  So, both $\dpi_b\in\mathcal{N}^\perp$ and $\dpi_c\in\mathcal{N}^\perp$.  Meanwhile, $m_a$ ($\dpi_a$) does not affect the angular momentum about the joint axis and adds to the previous body in a fixed manner/location, making it unidentifiable by itself.  Indeed, we could equivalently rigidly attach $m_a$ to the parent body without any effect on the dynamics.  We consider this reassignment a transfer of $\dpi_a$ between the neighboring bodies.  The combination of all such undetectable exchanges will be shown to span $\mathcal{N}$.

\begin{figure}[t]
\center
\includegraphics[width=\columnwidth]{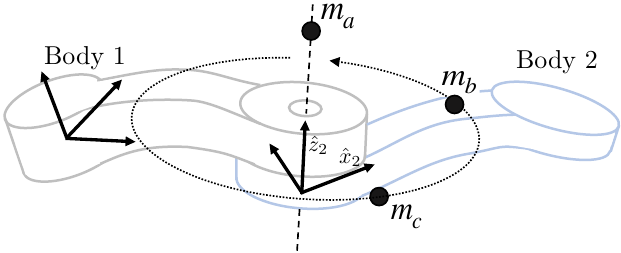}
\caption{Additional masses $m_a$, $m_b$, $m_c$ are rigidly attached to body 2.  The masses $m_b$, $m_c$ are detectable by the joint torque and appear in differing locations w.r.t.~body 1 based on the joint position.  The mass $m_a$ neither affects the angular momentum about the joint or joint torque, nor shifts location, and remains undetectable by itself.}
%\caption{Masses $m_a$ and $m_b$ are rigidly attached to body 2. The inertial properties of mass $m_a$ are felt by body 1 in a fixed way with changes in angle of the connecting joint, while the inertial properties of mass $m_b$ are felt in a variable way.}
\label{fig:invariance}
\end{figure}

\subsub{Spatial Notation}
To analyze the parameter combinations that do not affect the system dynamics, we will equivalently consider parameter changes that do not modify the kinetic energy.  A simple extension in Sec.~\ref{sec:fb} will address gravity. 
%\cang{or gravitational potential energy}\footnote{\cang{More precisely, we will consider parameter variations that do not modify the rate of change in gravitational potential energy.}}. \cang{To begin, focus is placed on the kinetic energy, since simple modifications to our developments address gravity.}
Effects of rotational and linear kinetic energy will be captured together using 6D spatial notation of  \citet{Featherstone08} (see Appendix~\ref{app:dynamics} for a short review). For example, the kinetic energy of a single body takes the form 
\newcommand{\m}{\phantom{-}}
%\begin{align*}
$T=\frac{1}{2} \vv\T \vI \vv$, 
%\end{align*}
with $\vv \in \mathbb{R}^6$ its spatial velocity and $\vI \subset \mathbb{R}^{6\times6}$ its spatial inertia, given by
\begin{align*}
\vv&=\left[ \begin{smallmatrix} \omega_x \\ \omega_y \\ \omega_z \\ v_x \\ v_y \\ v_z \end{smallmatrix} \right] &
\vI&= 
 \left[ \begin{smallmatrix} I_{xx} & I_{xy} & I_{xz} & 0      & -m c_z &  \m m c_y \\
                      I_{xy} & I_{yy} & I_{yz} &  \m m c_z & 0      & -m c_x \\
                      I_{xz} & I_{yz} & I_{zz} & -m c_y &  \m m c_x &      0 \\
                      0      &  \m m c_z & - m c_y & m & 0 & 0 \\
                        -m c_z & 0      & \m m c_x & 0 & m & 0 \\
                      \m m c_y & -m c_x & 0 & 0 & 0 & m \end{smallmatrix} \right]
\end{align*}
where $[\omega_x,\omega_y,\omega_z]\T$ is the angular velocity of a body-fixed coordinate system, $[v_x, v_y,v_z]\T$ the linear velocity of the origin of that system, $m$ the body mass, $[c_x,c_y,c_z]\T$ the CoM location in body coordinates, and $I_{xx}$, $I_{yy}$, etc. the mass moments and products of inertia about the body coordinate origin. For later use, the first moments are abbreviated as $h_{x} = m c_{x}$ for the $x$-axis and similarly for the others, while $
\Iset$ denotes the subspace of $6\times 6$ matrices taking the above form. Rather than working with $T = \frac{1}{2} \vqd\T \vH(\vq) \vqd$ in our analysis, we'll instead consider the system kinetic energy through $T=\sum_{i} \frac{1}{2} \vv_i \T \vI_i \vv_i$ where $i$ sums over bodies.

%Appendix~\ref{app:dynamics} quickly reviews spatial notation.% \citep{Featherstone08}.

Our preliminary development will consider the special case of rigid-body chains with bodies connected by single-degree-of-freedom joints. Bodies are numbered $1$ to $\nbod$ from the base to the end of the chain with a frame attached to each body (Fig.~\ref{fig:two-body}). Joint~$i$ connects body~$\pred{i}$ to body~$i$, with its configuration noted by $q_i\in \mathbb{R}$. Under these definitions, the spatial velocities of neighboring bodies are related as:
\begin{equation}
\vvb = \Xbaq \, \vva + \vPhib \, \qdb \label{eq:two_body_diff_kin}
\end{equation}
where $\Xba \in \mathbb{R}^{6\times 6}$ is a spatial transform between frames and the vector $\vPhib \in \mathbb{R}^{6 \times 1}$  describes the free motion for the joint. For instance, for a revolute joint about $\hat{z}_\bb$, $\vPhib = [0\,0\,1\,0\,0\,0]\T$.

\section{Two Body Case: General Spatial Motion}
\label{sec:Floating_Two_Body}
%{\color{red} Before this section, they need to 1) understand the high-level problem of identifiability (i.e., that it is a nullspace for the dynamics). 2) be comfortable with the concept of spatial inertias.  }

We begin by considering a single pair of bodies within the chain, and we use this simple case to mathematically develop the main ideas of the previous section. Consider bodies $\ba$ and $\bb$ where body $\ba$ is able to move with any general spatial velocity. This case could occur, for example, if body $\ba$ were a floating base or if it were a body in a fixed base system that was far enough away from the base of the mechanism to experience all 6 spatial degrees of freedom in its movement.

\begin{figure}[t]
	\center
	\includegraphics[width = \columnwidth]{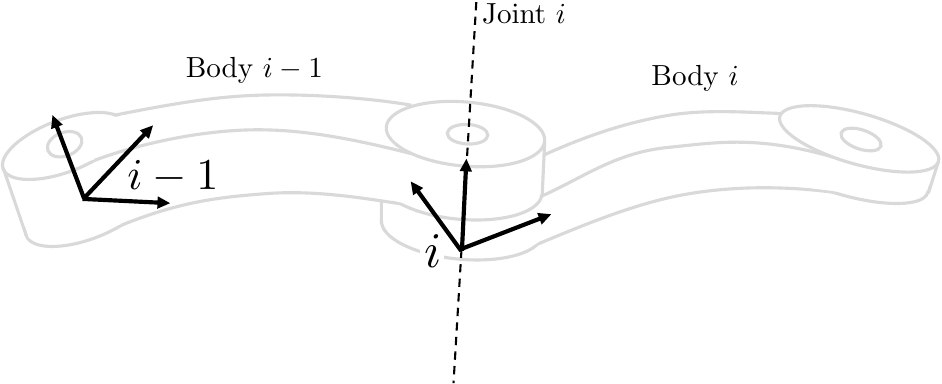}
	\caption{Example two-body system. Body $\ba$ can move freely, while body $\bb$ is attached via a single-degree-of-freedom joint (revolute as shown). Frames $\ba$ and $\bb$ are rigidly attached to bodies $\ba$ and $\bb$ respectively. %The frames coincide in the case when the joint angle $q$ satisfies $q=0$.
 }
	\label{fig:two-body}
\end{figure}

\subsection{Undetectable Changes in Inertia }
We proceed to examine conditions under which the inertial parameters of these bodies can be changed without affecting the system's dynamics. To do so, we 
%We proceed to characterize which inertia parameters affect the system's dynamics. The combinations of parameters that affect the dynamics are the same as those that would be identifiable from measurement of the system kinematics, joint torque, and any external contact forces on the system. 
consider a collection of changes in inertias $ \dIa$ and $\dIb$. For these changes to have no effect on the dynamics, they must not change the kinetic energy
\[
\delta T =\frac{1}{2} \vva\T\, \dIa \,\vva + \frac{1}{2} \vvb\T \,\dIb\, \vvb = 0
\]
 Using \eqref{eq:two_body_diff_kin}, the variation $\delta T$ can be factored as:
\begin{equation}
 \frac{1}{2} \begin{bmatrix} \vva \\ \qdb \end{bmatrix}\T \begin{bmatrix} \dIa + \XTba \,\dIb\, \Xba & \XTba \, \dIb\, \vPhib \\[.75ex] \vPhib\T\, \dIb\, \Xba & \vPhib\T\, \dIb\, \vPhib \end{bmatrix} \begin{bmatrix} \vva \\ \qdb \end{bmatrix}
 \label{eq:TwoBodyMassMatrix}
\end{equation}

For the kinetic energy variation $\delta T$ to always be zero, each entry of the above matrix must be zero. Considering the off-diagonal blocks, this condition requires
\begin{equation}
0 =  \vPhib\T \dIb \label{eq:revised_single_dof_cond2}%\\
%0 &= \delta \vPhi\T \dIb \vPhi \\
\end{equation}
since $\Xba$ is full rank. 
The condition $\vPhi_i\T \dIb \vPhi_i = 0$ for the lower-right block is redundant with \eqref{eq:revised_single_dof_cond2}.
For the upper-left block to be zero, it must be that
\begin{align}
\mathbf{0} &= \dIa +  \XTbaq\, \dIb\, \Xbaq && \forall \qb \label{eq:I1C}
\end{align}
 Toward more general results later the manuscript, we note that since $\vva$ and $\vvb$ can individually take any values, these two conditions are equivalent to:
% \begin{equation}
% \vPhi\T \dIb \vvb && \forall\, \vvb \label{eq:twoBodyTcond2}    
% \end{equation}
% the energy variation $\delta T$ can be rewritten as 
% \begin{align}
% \delta T &= \frac{1}{2}\dot{q}^2 \, \vPhi\T \,\dIb\, \vPhi +  \dot{q}\,\vPhi\T \,\dIb \, \Xbaq\, \vva\\
%  &~~~+\frac{1}{2} \vva\T \left[ \dIa + \XMT{}{}(q) \, \dIb \,\Xbaq \right] \vva \nonumber  \nonumber
% \end{align}
% In this case, enforcing $\delta T =0$ is equivalent to requiring
\newcommand{\myobjscale}{.98}
\begin{align}
\!\!\scaleobj{\myobjscale}{0} & \scaleobj{\myobjscale}{=\vPhib\T \dIb \vvb} && \scaleobj{\myobjscale}{\forall\, \vvb} \label{eq:twoBodyTcond2}\\
\!\!\scaleobj{\myobjscale}{0} & \scaleobj{\myobjscale}{=\vva\T \left[ \dIa +  \XTbaq\, \dIb\, \Xbaq \right] \vva} \!\!  &&\scaleobj{\myobjscale}{\forall \vva, q_i} \label{eq:twoBodyTcond1}
\end{align}
The first condition \eqref{eq:twoBodyTcond2} encodes that the change in inertia $\dIb$ must not change the projection of body~$\bb$'s momentum ($\vI_\bb \vvb$) along the joint free mode $\vPhib$. For example, for a revolute joint about $\hat{z}_\bb$, the angular momentum about $\hat{z}_\bb$ must not change due to $\delta \vI_i$.
% Noting that body~2 can experience any spatial velocity $\vvb\in \mathbb{R}^6$, this condition is equivalent to requiring
% \begin{equation}
% \vPhi\T\, \dIb = \bzero
% \end{equation}

% Similarly, since body 1 can experience any spatial velocity $\vva \in \mathbb{R}^6$, \eqref{eq:twoBodyTcond1} is equivalent to requiring
% \begin{align}
% \dI_{1} +\XMT{}{}(q)\, \dIb\,\Xbaq   = \bzero ~~\forall{q}
% \label{eq:I1C}
% \end{align}

Returning to condition \eqref{eq:I1C}, we see the first appearance of a condition regarding an inertia transfer. The sum $\dIa +  \XTbaq\, \dIb\, \Xbaq$ represents the change in the total inertia of the two bodies combined, where $\XTbaq\, \dIb\, \Xbaq$ gives how $\dIb$ maps back to its parent. As a result, \eqref{eq:I1C} requires that the combination of $\dIa$ and $\dIb$ must represent an even exchange of inertia between the bodies for all joint angles. %Although the inertial changes $\dI_{1}$ and $\dI_{2}$ represent fixed quantities, \eqref{eq:I1C} depends on configuration through $\XM{}{}(q)$. For \eqref{eq:I1C} to hold with changes in $q$ requires that the inertia transfer must remain equal and opposite as the joint configuration varies. 
In the case when  when $q_i=0$, \eqref{eq:I1C} requires
\begin{align}
\dIa = - \XTbaz \,\dIb \, \Xbaz && {\rm(Transfer~Assign.)}
\label{eq:Transfer_Assignment}
\end{align}
which encodes an assignment for a transfer of inertia between the bodies. Considering changes in configuration via a time derivative of \eqref{eq:I1C} then requires
\begin{align}
\bzero & = \frac{\rm d}{{\rm d} t} \left(-\XTbaq \, \dIb\, \Xbaq \right )
\label{eq:derivOfMap}
\end{align}
which enforces that the mapping of $\dIb$ to body $\ba$ must not change over time.  Consider the property from Eq.~(2.45) of \citet{Featherstone08}
\begin{equation}
\frac{\rm{d}}{{\rm d}t} \Xbaq = -(\vPhib \qdb)\times\Xbaq
\label{eq:change_of_X_main}
\end{equation}
where $(\vv)\times \in \mathbb{R}^{6 \times 6}$ gives the spatial cross-product matrix with its form noted further in Appendix \ref{app:dynamics}.
This property enables simplifying condition \eqref{eq:derivOfMap}  to:
\begin{equation}
\bzero = \XTba [\, (\vPhib \qdb \times)\T \dIb + \dIb (\vPhib \qdb \times)\, ]\,\Xba  \label{eq:InvWithX}
\end{equation}
 Since $\Xba $ is full rank and $\qdb$ can be chosen arbitrarily, \eqref{eq:InvWithX} is equivalent to 
\begin{align}
(\vPhib \times)\T\, \dIb + \dIb \, (\vPhib \times)  = \bzero 
\label{eq:InertiaChange2}
\end{align}

%\footnote{This equation may be recognized as requiring a zero time derivative of the spatial inertia $\dIb$ when moving with spatial velocity $\vPhi$ \cite[Eq. (2.65)]{Featherstone08}.}

\noindent {\bf Summary:} In the case of a body that experiences general spatial motion,  an exchange of inertia \eqref{eq:Transfer_Assignment} with its child will not affect the dynamics if and only if it 1) does not modify the projection of the child body's momentum along the joint and 2) has a variation to the child inertia that maps to its parent in a constant way across configuration
\begin{empheq}[box=\fbox]{align}
\bzero&= \vPhib\T \dIb &&{\rm(Momentum)}\label{eq:Momentum_Floating}\\
\bzero &= \Ichange{\dIb}{\vPhib} && {\rm(\Mapping)}
\label{eq:Invariance_Floating} 
\end{empheq}
%&& \!\!{\rm (Transfer~Cond.)}
%&&\!\!{\rm (Invariance~Cond.)}  
We name these conditions the momentum and \mapping conditions. Conditions \eqref{eq:Momentum_Floating} and \eqref{eq:Invariance_Floating} are equivalent to Eq.~(18) of \citet{niemeyer1991performance} (equiv., Eqs.~(4.13) and (4.14) of \citet{Niemeyer90}) and to Eq.~(41) of \citet{Ayusawa14}. 
%We name these conditions the momentum, and \mapping conditions. Conditions \eqref{eq:Momentum_Floating} and \eqref{eq:Invariance_Floating} are equivalent to Eq.~(18) of \citet{niemeyer1991performance} (equiv., Eqs.~(4.13) and (4.14) of \citet{Niemeyer90}) and to Eq.~(41) of \citet{Ayusawa14}. 
\begin{remark}
    \citet{Ayusawa14} includes an additional condition in their Eq.~(42), which comes from considering Coriolis/centripetal effects. However, since these effects are determined by the kinetic energy, the additional conditions are redundant with those herein, enabling a more compact treatment. 
\end{remark}
\begin{remark}
\label{rem:FirstRowOfMassMatrix}
%\rev{
Building on this simple case, we revisit the main result of \citet{Ayusawa14} on parameter identifiability for floating-base systems with only external force/torque measurements in Sec.~\ref{subsec:EnergyToDynamics} following the extension of our theory to multi-DoF joints.% \revg{[CUT HERE? YES I WOULD...]} \cang{The main idea of the argument therein is closely connected to what we saw in this section. Namely, considering the matrix that shows up in \eqref{eq:TwoBodyMassMatrix} a zero variation to the top block row occurs if and only if the matrix in total has zero variation.} %An analog of this property holds for floating-base systems, where the matrix analogous to the one in \eqref{eq:TwoBodyMassMatrix} is the system mass matrix $\vH$.
%}
\end{remark}
%; however, the above form is parameterization free, and thus may admit further physical intuition. %due to its expression with respect to inertia matrices instead of inertia parameters. %We name \eqref{eq:floatingExchange} the transfer condition, \eqref{eq:floatingCondition1} the transfer \mapping condition, and \eqref{eq:floatingCondition2} the momentum \mapping condition. 
Returning to the high-level description of how parameters are identified, the momentum condition \eqref{eq:Momentum_Floating} enforces that the changes $\dIb$ must not affect the \textbf{}local joint torque, where the \mapping condition \eqref{eq:Invariance_Floating} enforces that the changes $\dIb$ must not be detectable via how they are mapped to the parent. %When \eqref{eq:Momentum_Floating} and \eqref{eq:Invariance_Floating} hold, the parameter changes $\dIb$ map in a fixed way to the parent, and thus the transfer \eqref{eq:Transfer_Assignment} is undetectable.

%{\color{blue} Intuition into why these are correct.
%\loremlines{9}}
%\lipsum[1-1]
%The next subsection considers a few simple examples to give further intuition into these conditions.

\subsection{Example: Revolute Joint}
\label{sec:example_revolute}
% This subsection provides examples with different joint types for intuition into the parameter nullspace of a two-body system. We only consider parameter changes to body~2, since the corresponding changes to body~1 are determined with the transfer condition \eqref{eq:Transfer_Assignment}. 
This subsection gives an example of the momentum and mapping conditions with a revolute joint. We only consider parameter changes to body~$\bb$, since the corresponding changes to body~$\ba$ are determined with the transfer assignment \eqref{eq:Transfer_Assignment}. 
%\subsubsection{Revolute Joint}
Consider a revolute joint about the local $\hat{z}_i$ axis such that $\vPhib = [0~ 0~ 1~ 0~ 0~ 0]\T$.  In this case the momentum condition \eqref{eq:Momentum_Floating} imposes
%\revg{\begin{center} CHANGING ORDER TO MATCH m,h,I in (34) \end{center}}
\begin{equation}
 \delta h_x = \delta h_y =  \delta I_{xz} = \delta I_{yz} = \delta I_{zz} = 0 \label{eq:RevoluteMomentumConditions}
\end{equation}
for the second link. The last three restrictions ($\delta I_{xz} = \delta I_{yz} = \delta I_{zz}=0$) ensure that the angular momentum about the $\hat{z}_\bb$ axis will remain unchanged for pure angular velocities of the body frame, while the first two $\delta h_x =\delta h_y = 0$ likewise ensure the same for linear velocities. The implication is that, for the second body, $h_x$, $h_y$, $I_{xz}$, $I_{yz}$, and $I_{zz}$ are identifiable via the joint torque.

Likewise the \mapping condition \eqref{eq:Invariance_Floating} imposes
\begin{equation}
\delta h_x = \delta h_y = \delta I_{xy} = \delta I_{xz} = \delta I_{yz}=  \delta I_{xx} - \delta I_{yy} = 0 \label{eq:MappingRevolute}
\end{equation}
to ensure that $\dI_\bb$ maps in a fixed way to the parent (i.e., that $\XMT{i}{\pred{i}}(q_i)\, \dI_\bb \, \XM{i}{\pred{i}}(q_i)$ is configuration invariant). Physically, these conditions are satisfied for any mass distribution that is symmetric about the $\hat{z}_\bb$ axis. Comparing with \eqref{eq:RevoluteMomentumConditions}, \eqref{eq:MappingRevolute} implies that $I_{xy}$ and $I_{xx} - I_{yy}$ are also identifiable, but via configuration dependence in the way these parameters are combined with the parent inertia.

To satisfy both the momentum and \mapping conditions, $\dI_\bb$ is left with three degrees of freedom 
\begin{equation}
\delta m, \delta h_z, \delta I_{xx} = \delta I_{yy} \nonumber
\end{equation}
Notation for this third freedom signifies that changes in $I_{xx}$ must match those to $I_{yy}$. 
In this example, the inertia exchange has a physical interpretation. The transfer freedoms represent an exchange of any infinitely thin rod along $\hat{z}_\bb$.
The point mass $m_a$ in Fig.~\ref{fig:invariance} is a special case of such a rod.

In general, a rigid body attached to another by a revolute joint will add a maximum of seven parameters, in the absence of further motion restrictions.

\section{Two Bodies: Restricted Spatial Motion}
\label{sec:Two_Fixed}

Building on the previous section, we explore how additional spatial motion restrictions affect identifiability. Similar to the previous case, parameters of a body are identified via two mechanisms: (1) directly via how the parameters are sensed on the preceding joint, or (2) indirectly via variations in how the parameters map to their parent body. However, if the parent has motion restrictions, some of its parameters may not appear in the equations of motion and would be unidentifiable via the generalized force. Thus, for a child parameter to use the second identification mechanism, variations in its mapping to the parent must appear on identifiable parameters of the parent. %Conversely, an inertial transfer will be unobservable if (1) it cannot be directly sensed on the connecting joint and (2) the transfer maps in a fixed way onto the identifiable parameters of the parent. 
We develop these conditions mathematically and then work through examples to illustrate them.

\subsection{Coupled Conditions for $\delta T=0$}
\label{sec:Two_Bodies_Coupled}

In the previous development, body $\ba$ was assumed able to experience any velocity $\vv_{\pred{i}}$. This simplified analysis when reducing conditions
\eqref{eq:twoBodyTcond2} and \eqref{eq:twoBodyTcond1}
for $\delta T = 0$
\begin{align}
\!\!\scaleobj{\myobjscale}{0} & \scaleobj{\myobjscale}{=\!\vPhib\T \dIb \vvb} && \scaleobj{\myobjscale}{\forall\, \vvb} \label{eq:Momentum_General}\\
\!\!\scaleobj{\myobjscale}{0} & \scaleobj{\myobjscale}{=\!\vva\T \left[ \dIa \!+\!  \XTbaq\, \dIb\, \Xbaq \right] \vva} \!\!  &&\scaleobj{\myobjscale}{\forall \vva, q_i} \label{eq:Mapping_General_Precursor}
\end{align}
%\begin{align}
%\bzero &\equiv \vPhib \T \dIb \vvb \label{eq:Momentum_General} \\
%\bzero &\equiv \vva\T \left[ \dIa +  \XTbaq\, \dIb\, \Xbaq \right] \vva   \label{eq:Mapping_General_Precursor} 
%\end{align}
Now, instead, we consider these conditions in the context where $\vva$ does not freely take values in $\mathbb{R}^6$.

\subsection{Decoupling Effects of Inertia Changes in the Mapping Condition}
\label{sec:Two_Bodies_Simplify}
%satisfies the above conditions \eqref{eq:Momentum_General} and \eqref{eq:Mapping_General_Precursor}.

%\subsub{Treating $\qb$ dependence in the Mapping Condition (\ref{eq:Mapping_General_Precursor})} 
The condition \eqref{eq:Mapping_General_Precursor} is impractical to verify as written due to its dependence on $\qb$ and $\vva$. It is also complicated by the fact that both $\dIa$ and $\dIb$ appear, which we address first.
%
%Similarly, the conditions \eqref{eq:Mapping_General_Precursor} are impractical to verify as written due to their dependence on $\vva$ and $\qb$. While the momentum conditions are related to the direct identifiability of $\dIb$, the ones in \eqref{eq:Mapping_General_Precursor} are related to the indirect identifiability of $\dIb$ via torques earlier in the mechanism. This indirect nature makes these conditions more challenging to simplify. 
%We first address the coupled effects of $\dIa$ and $\dIb$ before then addressing dependence on $\qb$ and $\vva$. %before then addressing $\vva$.

For a given fixed $\vva$, \eqref{eq:Mapping_General_Precursor} must hold for all $\qb$. Equivalently, the expression on the right side of \eqref{eq:Mapping_General_Precursor} must be zero at some point, and have zero derivative everywhere. We again consider the transfer assignment (\ref{eq:Transfer_Assignment}):
\[
\dIa = - \XTbaz \,\dIb \, \Xbaz
\]
so that \eqref{eq:Mapping_General_Precursor} holds at $\qb=0$.  
Then, enforcing a zero derivative condition for \eqref{eq:Mapping_General_Precursor} gives:
\begin{equation}
\frac{\rm d}{{\rm d} q_i} \vva\T \left[ \XTbaq\, \dIb\, \Xbaq \right] \vva = 0 \label{eq:MappingCondition_derivative}
\end{equation}
which no longer has a dependence on $\dIa$.

\subsection{Simplification into a Finite Form}
\label{sec:Two_Body_Finite}

\subsub{Addressing dependence on $q_i$}
We can enforce a zero derivative everywhere by ensuring that all (first and higher-order) derivatives (w.r.t., $q_i$) are zero at $\qb=0$. For the first derivative:
\begin{align*}
0 &= \left. \frac{{\rm d}}{ {\rm d} q_i} \vva\T \left[ \dIa +  \XTbaq\, \dIb\, \Xbaq \right] \vva \right|_{\qb=0} \\
  &= \vva\T \XTbaz\, \left[  \left(\vPhib \times\right)\T \! \dIb + \dIb \left(\vPhib \times \right)  \right] \Xbaz \, \vva 
\end{align*}
While the analogous condition for the $k$-th derivative w.r.t.~$q_i$ is equivalent to:
\begin{equation} \vva\T\, \XTbaz\, \dIb^{(k)} \, \Xbaz \, \vva  = 0 \label{eq:Mapping_General}
\end{equation}
where $\dIb^{(0)} = \dIb$ and 
\[
\dIb^{(k+1)} = \left(\vPhib \times\right)\T \dIb^{(k)} + \dIb^{(k)} \left(\vPhib \times \right)
\]
The matrix $\dIb^{(k)}$ captures the $k$-th derivative in how $\dIb$ maps to the parent. The mapping condition for the case of unrestricted motion is equivalent to $\dIb^{(1)}= \bzero$. Each successive derivative is linear in the previous, and thus, for all derivatives after $k=10$, the condition \eqref{eq:Mapping_General}  can be guaranteed to be redundant via the Cayley-Hamilton Theorem \citep{Rugh96}.

%\revg{

%{\Large NEW - SEPARATING THE TWO SPANS}

\subsub{Addressing velocity dependence}
Repeating the momentum condition, the previous development is rewritten as
\begin{align}
0 &= \vPhib\T \dIb \vvb && \forall\,\vvb \label{eq:Momentum_General_2}\\
0 &= \vva\T \left[ \XTbaz\, \dIb^{(k)}\, \Xbaz \right] \vva && \forall \vva \label{eq:Mapping_General_2} \\
 &\qquad \quad \qquad \qquad \quad \forall\,k=1,\ldots,10 \nonumber
\end{align}
Both conditions remain impractical to verify as written since they give an infinite number of constraints.  To simplify them, we introduce spans and create an equivalent finite set of conditions.  This step provides the key to our contribution.  We address each condition separately.

\subsub{Linear velocity span for momentum condition \eqref{eq:Momentum_General_2}} 
First, let the set of possible velocities for $\vva$ be denoted by $\Vset{\ba}$. The set of possible velocities for body $\bb$ is
\begin{align}
\scaleobj{.95}{\!\!\Vset{\bb} = \{ \Xbaq\, \vva + \vPhib \, \qdb ~|~ \vva \in \Vset{\pred{i}},\,\qb,\qdb \in \mathbb{R}\}}
\end{align}
Then, consider the following linear velocity span
$\Vspan{i}$, which is subspace of $\mathbb{R}^6$ and characterizes the motion restrictions:
\begin{align}
\Vspan{\bb} &= \spn\{\vv ~|~ \vv \in \Vset{\bb} \}
\label{eq:VSpan}
\end{align}

\begin{remark}
Note that $\Vset{\bb} \subseteq \Vspan{i}$, since not all elements of $\Vspan{\bb}$ need be attainable velocities. Rather, all elements of $\Vspan{\bb}$ can be expressed as the linear combination of some attainable velocities. 
\end{remark}

Since the set $\Vspan{i}$ is a finite-dimensional subspace, we consider a set of basis vectors
\begin{align}
\Vspan{i} &= \spn \{ \vv_{i,1},\ldots, \vv_{i, \Nv{i}} \} 
\end{align}
where we collect the basis elements for $\Vspan{i}$ into a matrix $\vV_\bb = \left[ \vv_{i,1} ,\ldots,\vv_{i,\Nv{i}} \right]$ so that ${\rm Range}(\vV_{\bb}) = \Vspan{\bb}$. 
We will later provide methods to compute $\vV_{\bb}$.  With this, the infinite set of conditions \eqref{eq:Momentum_General_2} are equivalent to the finite set of conditions
\begin{align}
0 &= \vPhib\T \, \dIb \, \vv_{i,\ell} , \quad  \forall\,\ell=1,\ldots,\Nv{i} \label{eq:MomentumBasis}
\end{align}
which can easily be collected into the vector equation
\begin{align}
\bzero &= \vPhib\T \, \dIb \, \vV_i \label{eq:MomentumFinite}
\end{align}

\subsub{``Quadratic'' velocity span for mapping condition \eqref{eq:Mapping_General_2}} 
Unfortunately, the mapping condition depends on quadratic velocity terms.  We employ the $\trace()$ operator toward creating a ``quadratic'' velocity span.  Fortunately, leveraging the linear inertial parameterization in the next section will enable simplifying this step.

\newcommand{\W}{\mathbf{W}}

Using $\trace()$, we rewrite the mapping condition \eqref{eq:Mapping_General_2}
\begin{align}
0 &= \trace\!\left(\vva \,\vva\T \, \left[ \XTbaz\, \dIb^{(k)}\, \Xbaz \right] \right)\label{eq:Mapping_General_Trace} \\
       &\qquad \quad \qquad \qquad \quad \forall \vva \in\mathcal{V}^*_{\pred{i}}, \forall\,k=1,\ldots,10 \nonumber
\end{align}
and create the ``quadratic'' velocity span
\begin{align}
\Wspan{\pred{i}} &= \spn\{\vv\vv\T ~|~ \vv \in \Vset{\pred{i}} \}
\label{eq:WSpan}
\end{align}
where $\Wspan{\pred{i}}$ is a finite-dimensional subspace of $\mathbb{R}^{6\times 6}$.  We also consider the set of basis elements
\begin{align}
\Wspan{\pred{i}} &= \spn \{ \W_{\pred{i},1},\ldots, \W_{\pred{i}, \Nw{\pred{i}}} \}
\end{align}
for converting the mapping condition into the finite set of conditions
\begin{align}
0 &= \trace\!\left(\W_{\pred{i},j} \, \left[ \XTbaz\, \dIb^{(k)}\, \Xbaz \right] \right) \label{eq:mappingFinite} \\
        &\qquad \qquad \quad \forall\, k=1, \ldots, 10, \quad j = 1,\ldots, \Nw{\pred{i}} \nonumber
\end{align}
Unlike \eqref{eq:MomentumBasis}, the iteration over all basis elements can not be collected into a vector equation.  Sec.~\ref{sec:Parameters} will avoid this issue.

Again, physically, \eqref{eq:MomentumFinite} encodes that $\dIb$ must not change the components of body~$\bb$'s momentum associated with the joint $\vPhib$, while \eqref{eq:mappingFinite} encodes that $\dIb$ maps in a fixed way onto the identifiable parameters of the parent.

\subsection{Main Theorem}

Although we've developed this theory for a pair of bodies, an inductive argument (in the appendix) shows that we can consider the momentum and mapping conditions on inertia transfers joint-by-joint in chains of bodies to build out all unidentifiable changes to the inertia. 

\begin{theorem}
\label{thm:main}
(Main Result) Consider a serial-chain rigid-body system in the absence of gravity, with the following inertia transfer subspaces for each joint ($i \in \{1,\ldots,\nbod\}$):
\begin{align}
\TsetInertia_i =  \{ \dI_1, \ldots&, \dI_N \in \Iset ~| ~ \exists \dI_0 \in \Iset , \dI_j=\bzero {\rm{~if~}} j\notin\{i,\pred{i}\}, \nonumber\\
\dI_{\pred{i}} &= - \XJT{i}\, \dI_i \, \XJ{i}, \nonumber\\ 
\bzero &= \vPhi_i\T \, \dI_i \,\vV_{i} ,  \nonumber\\
\bzero &= \trace( \W_{\pred{i},j} \, \XJT{i}\, \dI_i^{(k)}\, \XJ{i} \,)\nonumber \\ 
& ~~~~~~~~~~~~~\forall\, {k = 1, \ldots, 10}, \quad j={1, \ldots, \Nw{\pred{i}}} \nonumber 
\nonumber
%\label{eq:finalTi}
~\}\nonumber
%\label{eq:finalTi}
\end{align}
% The structurally unobservable parameter subspace $\mathcal{N}$ satisfies
% \begin{equation}
% \mathcal{N} = \bigoplus_{i=1}^{\nbod} \mathcal{T}_i \nonumber
% \end{equation} 
Then, the set of all structurally unobservable inertia changes is given by $\TsetInertia_1 \oplus \cdots \oplus \TsetInertia_N$, where $\oplus$ denotes a direct sum of vector subspaces.
\end{theorem}
\begin{proof} See Appendix \ref{sec:Appendix:varH}.
\end{proof}

We proceed first with a set of examples before focusing on making this result amenable to practical computation with inertial parameters in the subsequent sections.

\subsection{Example}
\label{sec:simple_Examples}

\begin{figure}[t!]
\center
\includegraphics[width=.6\columnwidth]{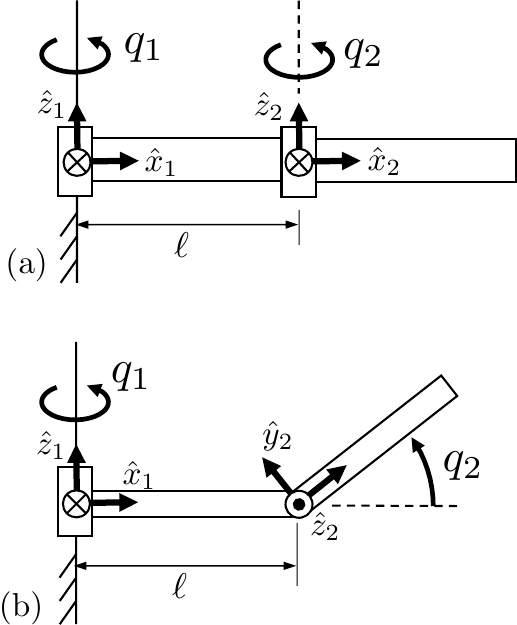}
\caption{Two simple RR manipulators.}
\label{fig:simpleManips}
\end{figure}

%To give additional intuition into the momentum and relaxed \mapping conditions, 
We work through these conditions within the context of the two 2R manipulators shown in Fig.~\ref{fig:simpleManips}. 

\subsubsection{Simple system with parallel joint axes}
\label{sec:Simple1}
The system in Figure \ref{fig:simpleManips}(a) is planar, with the spatial velocities of its bodies and a basis matrix $\vV_2$ given by:
\[
\vv_1 = \left[\begin{smallmatrix} 0 \\ 0 \\ \qd_1 \\ 0 \\ 0 \\0 \end{smallmatrix}\right] ~~~~~ \vv_2 = \left[\begin{smallmatrix} 0 \\ 0 \\ \qd_1+ \qd_2 \\ \ell s_2 \qd_1 \\ \ell c_2 \qd_1  \\ 0 \end{smallmatrix}\right] ~~~~~ 
\vV_2 = \left[\begin{smallmatrix} 0 & 0 & 0\\ 0 &0 &0 \\ 1 &0 &0 \\ 0 & 1 &0 \\ 0 & 0 & 1 \\0 & 0 & 0 \end{smallmatrix}\right]
\nonumber
\]
where $c_2 = {\rm cos}(q_2)$ and $s_2 = {\rm sin}(q_2)$. Note that the velocity span for body~2 includes linear velocities in both the $x$ and $y$ directions, since effects from $\dot{q}_1$ will be in the $\hat{x}_2$ or $\hat{y}_2$ direction depending on the value of $q_2$. Considering the span for $\vv_1 \vv_1\T$ we likewise have a basis for $\mathcal{W}_1$ consisting of the single element:
\[
\vW_{1,1}= \vPhi_1 \vPhi_1\T = 
 \left[ \begin{smallmatrix} 0 & 0 & 0 & 0 & 0 &  0 \\
                      0 & 0 & 0 & 0 & 0 &  0 \\
                      0 & 0 & 1 & 0 & 0 &  0 \\
                      0 & 0 & 0 & 0 & 0 &  0 \\
                      0 & 0 & 0 & 0 & 0 &  0 \\
                      0 & 0 & 0 & 0 & 0 &  0 \end{smallmatrix} \right]
\]
which has a $1$ in the slot of the matrix corresponding to $I_{zz}$ (the only identifiable parameter for body $1$).

% The parameter $I_{zz_1}$ is the only inertial parameter for body 1 that affects its kinetic energy and thus: 
% \[
% \vK_1 = \begin{bmatrix} 0 & 0 & 0 & 0 & 0 & 0 & 0 & 0 & 0 & 1 \end{bmatrix}\T
% \]
% where the one in the last entry corresponds to $I_{zz}$ being the last entry in $\vpi$ \eqref{eq:paramsDef}.

%Considering the transfer across joint 1 amounts to considering a transfer between body~1 and the ground. Since the set $\mathcal{K}_0 = \{\bzero\}$, the relaxed \mapping condition for joint 1 holds without restriction on $\dI_1$. The momentum condition enforces that transfers between Body $1$ and ground must only satisfy $\vV_{1}\T \, \dI_1\, \vPhi_1 = \delta I_{zz_1}=0$. Thus, all the other parameters of body~1 are unobservable, contributing nine unobservable parameters to $\mathcal{N}_{\vH}$. 

%Constraints on the undetectable inertial transfers across the first and second joints are given through the transfer sets:

%\begin{align}
% \mathcal{T}_1= \{& \dpi \in \mathbb{R}^{20}~|~0= \vV_{1}\T \, \dI_1\, \vPhi_1,~  \bzero=\dI_2
% \} \label{eq:SimpleTransfer1}\\
% \mathcal{T}_2= \{& \dpi \in \mathbb{R}^{20}~|~ \dI_1 = - \XMT{J_2}{1} \dI_2 \XM{J_2}{1}, \bzero= \vV_{2}\T \, \dI_2\, \vPhi_2 \nonumber \\
% &~~{\rm Tr}(\XM{J_2}{1}\,\vK_{1,1}\,\XMT{J_2}{1} \dI_2^{(k)}) = 0, k=1,...,10\} \label{eq:SimpleTransfer2}
% \end{align}
%

The body 1 momentum condition  $\bzero = \vPhi_1\T \, \dI_1\, \vV_1$ imposes 
\[
\delta I_{zz_1} = 0
\]
The body 1 mapping condition is empty as $\Wspan{0}$ is empty - any parameters transferred to ground are unidentifiable.  The body 2 momentum condition  $\bzero = \vPhi_2\T \, \dI_2\, \vV_{2}$ imposes 
\[
 \delta h_{x_2} = \delta h_{y_2} = \delta I_{zz_2} = 0
\]
And the body 2 mapping condition
 \[
0  = \trace(\vW_{1,1} \, \XMT{2}{1}(0) \,\dI_2^{(k)} \,\XM{2}{1}(0) \,)
 \]
 or, equivalently, using $\vW_{1,1} = \vPhi_1 \vPhi_1\T$, written as
 \[
0  = \vPhi_1\T \, \XMT{2}{1}(0) \,\dI_2^{(k)} \,\XM{2}{1}(0) \, \vPhi_1
 \]
 imposes
 \begin{align*}
   \delta h_{y_2} &= 0 & \text{when\ } k &= 1 \\
   \delta h_{x_2} &= 0 & \text{when\ } k &= 2
 \end{align*}
 The constraints are redundant for all $k$ higher. Thus, we find four base parameters: $I_{zz_1}$ for body 1,  $h_{x_2}$,  $h_{y_2}$, and $I_{zz_2}$ for body 2.

For reference, consider the mass matrix $\vH(\vq)$
\begin{equation*}
  \left[ \begin{smallmatrix}
    I_{zz_1}^\star + I_{zz_2} + 2 \ell\, h_{x_2} c_2 - 2 \ell\, h_{y_2} s_2
    \quad  
    &
    I_{zz_2} + \ell\, h_{x_2} c_2 - \ell\, h_{y_2} s_2
    \\[1ex]
    I_{zz_2} + \ell\, h_{x_2} c_2 - \ell\, h_{y_2} s_2
    &
    I_{zz_2}
    \end{smallmatrix} \right]
\end{equation*}
where $I_{zz_1}^\star = I_{zz_1} + m_2\ell^2$.  We see the expected four parameter groupings: $I_{zz_1} + m_2\ell^2$, $I_{zz_2}$, $h_{x_2}$, $h_{y_2}$.  Note the effect of $m_2$ is transferred from body 2 and combined with $I_{zz_1}$.

In more detail, body 1 experiences only a pure rotation, so the body 1 momentum condition shows we can only identify the rotational inertia about the joint axis.  Body 2 rotates in the same plane but may also translate.  So the body 2 momentum condition also considers the torque needed to linearly move the center of mass ($h_{x_2}$, $h_{y_2}$ depending on $q_2$).  The mapping condition also certifies the identifiability of $h_{x_2}$, $h_{y_2}$, since they map inconsistenly to body 1 (compare to $m_b$, $m_c$ in Fig.~\ref{fig:invariance}). But $m_2$ itself maps to body 1 in a fixed manner (compare to $m_a$ in Fig.~\ref{fig:invariance}). Overall, seven parameters of body 2 are not identifiable individually, and can be transferred to body 1.

% \revg{\Large Old Description:}

\newcommand{\XJJ}{\XM{}{}(0)}
\newcommand{\XJJT}{\XMT{}{}(0)}

\subsubsection{Simple system with perpendicular joint axes}
\label{sec:Simple2}

For the system in Fig.~\ref{fig:simpleManips}(b), the second body can move out of the $\hat{x}_1,\hat{y}_1$ plane, resulting in additional motion freedom. A full derivation for this example is in Appendix~\ref{app:example2}, with main results summarized here. The body 2 momentum condition provides:
\beq
\delta I_{xz_2} = \delta I_{yz_2} = \delta I_{zz_2} = 0 \label{eq:SimpleExample2Momentum}
\eeq
which imply that $ I_{xz_2}$, $ I_{yz_2}$, and $ I_{zz_2}$ can be identified through the second joint torque. The body 2 \mapping condition requires:
\beq
 \delta h_{x_2} = \delta h_{y_2} = \delta I_{xx_2} - \delta I_{yy_2} = \delta I_{xy_2} = 0 \label{eq:SimpleExample2Invariance}
\eeq
which imply that $h_{x_2}$, $h_{y_2}$, $I_{xx_2} -  I_{yy_2}$, and $I_{xy_2}$ can be identified via the first joint torque. 
Note again, due to motion restrictions, conditions \eqref{eq:SimpleExample2Momentum} and \eqref{eq:SimpleExample2Invariance} represent a subset of those in the free-floating case \eqref{eq:RevoluteMomentumConditions} and \eqref{eq:MappingRevolute} respectively. However, unlike the previous manipulator with parallel joints, the seven conditions from  \eqref{eq:SimpleExample2Momentum} and \eqref{eq:SimpleExample2Invariance} are independent. As a result, the unobservable transfers across Joint 2 have three degrees of freedom, which must coincide with those in the free-floating case. This gives eight base parameters for the mechanism (one from body 1, and seven from body 2).

\section{Momentum/Mapping Conditions in Terms of Inertial Parameters}
\label{sec:Parameters}

Up until this point, we have phrased the momentum and mapping conditions in terms of changes to the inertia matrices. In this section, we instead translate these conditions to ones on the inertial parameters, relying heavily on the linearity of the inertia matrices in the inertial parameters of each body: 
%\subsection{Inertial Parameters}
% \subsub{Relating to Parent Parameters} While the development has thus far been parameterization free, we now consider the inertial parameters to further simplify our new mapping condition \eqref{eq:mappingFinite} and clarify its physical meaning. %we will consider inertial parameters. 
%The inertia $\vI_i$ of any body $i$ is linear in parameters $\vpi_i$. 
\begin{equation}
\vpi_i = [m, h_x, h_y, h_z, I_{xx}, I_{xy}, I_{xz}, I_{yy}, I_{yz}, I_{zz}]\T  \label{eq:paramsDef}
\end{equation}
We switch from the matrix form of the spatial inertia to the parameter vector via the notation
%\begin{equation}
%\vI(\vpi_i) = [ \vpi_i]^{\wedge} {\textrm{~and~}} 
$\vpi_i = [ \vI_i ]^\vee$ 
%\label{eq:promote_demote}
%\end{equation}
where %the wedge $\wedge$ promotes a vector to an inertia matrix, while 
the vee $\vee$ demotes an inertia to a parameter vector. %The spatial inertia is linear in its inertia parameters, so each of these transformations is a linear operator.

\subsection{Momentum Condition}
\label{sec:MomentumParams}

To express the momentum condition $\vPhi_i\T\dI_i \vV_i = \bzero$ from \eqref{eq:MomentumFinite} using inertial parameters, we temporarily consider $\vpi\in \mathbb{R}^{10}$ for a single body with  and  define the matrix:
\begin{equation}
\vC(\vV, \vPhi) = \left[ \frac{\partial}{\partial \vpi} \vV\T\,  \vI(\vpi) \, \vPhi \right]\T
\label{eq:Cmom}
\end{equation}
Then, inertia variations $\dI_i$ satisfy the momentum condition if and only if the corresponding parameter changes $\dpi_i$ satisfy $\MomentumCondition{i} = \bzero$.

\subsection{Mapping Condition}
\label{sec:MappingParams}

The earlier version of the mapping condition from \eqref{eq:Mapping_General_2} is recalled below
\begin{align}
0 &= \vva\T \left[ \XTbaz\, \dIb^{(k)}\, \Xbaz \right] \vva && \forall \vva \label{eq:SimplifiedMappingAgain}\\
 &\qquad \quad \qquad \qquad \quad \forall\,k=1,\ldots,10 \nonumber
\end{align}
which we favor over \eqref{eq:mappingFinite} to simplify our development in this section. 

Since this condition is a bit more complicated, we will work on it from the outside moving in, making use of three separate helper functions:
\begin{align}
\vk(\vv) &:= \nabla_{\vpi} \, \frac{1}{2}\,\vv\T\,\vI(\vpi) \, \vv 
\\ 
\vB(\vX) &:= \frac{\partial}{\partial \vpi} \left[\vX\T \,\vI(\vpi)\, \vX  \right]^\vee \\
\vA(\vPhi) &:= \frac{\partial }{\partial \vpi} \left[    (\boldsymbol{\Phi} \times)\T\, \vI(\vpi) + \vI(\vpi)\, (\boldsymbol{\Phi} \times)  \right]^\vee  \label{eq:A_InertialParams}
\end{align}
which we now use in three separate simplifying stages.

For the first stage, we note that $\vk(\vv_i)$ gives the gradient of the kinetic energy for the $i$-th body w.r.t.~its inertia parameters. In this regard, if $ \vk(\vv)\T \dpi_i =0$ for all attainable velocities $\vv \in \Vset{i}$, the linear combination of parameters given via $\dpi_i$ does not appear in the kinetic energy of the mechanism (i.e., that combination is unidentifiable).    

With this motivation, consider the span of the vectors $\vk(\vv_i)$ over all attainable velocities for body $i$:
\[
\mathcal{K}_i = \spn\{ \vk(\vv) ~|~ \vv \in \Vset{i} \}
\]
Analogous to the previous span $\Vspan{i}$, $\mathcal{K}_i$ is a {\em vector} subspace of $\mathbb{R}^{10}$ and thus has a finite basis representation that we set via selecting any matrix $\vK_i$ with ${\rm Range}(\vK_i) = \mathcal{K}_i$. A change $\dpi_i$ to body~$i$ is unidentifiable if and only if $\vK_i\T\, \dpi_i = \bzero$. Turning this around, the columns of $\vK_i$ form a basis for the coefficients of the inertial parameters of body~$i$ that are identifiable themselves or in combination with other bodies. 

Using this matrix, we provide equivalent conditions for the form of the mapping condition in \eqref{eq:SimplifiedMappingAgain} as:
\begin{align}
\bzero &= \vK_{\pred{i}}\T \left[\XJT{i}\, \dI_i^{(k)}\, \XJ{i} \right]^{\vee}\label{eq:Step1_Mapping}\\
 &\qquad \quad \qquad \qquad \quad \forall\, k=1,\ldots,10 \nonumber
\end{align}
which enforces that $\dI_i$ maps in a fixed way onto the identifiable parameters of the parent.

\begin{remark}
To connect this development to the previous section, we could alternatively compute $\mathcal{K}_i$ via the ``quadratic'' velocity span $\Wspan{i}$ as:
\[
\mathcal{K}_i = \spn\left\{ \nabla_{\vpi} \trace(\,\W \,\vI(\vpi)\,) ~|~ \W \in \Wspan{i} \right\}
\]
\end{remark}

Proceeding to our second stage of simplifications, consider the $10 \times 10$ parameter transformation matrix
\begin{equation}
\Bi = \vB( \, \XJ{i}\, )
\end{equation}
that maps parameters to the predecessor.  With this definition, we re-express the transfer assignment \eqref{eq:Transfer_Assignment} as:
\begin{equation}
\dpi_{\pred{i}} = - \Bi\, \dpi_i
\label{eq:parameter_transfer_assignment}
\end{equation}
Via this matrix, \eqref{eq:Step1_Mapping} is equivalent to
\begin{equation}
\vK_{\pred{i}}\T\, \Bi \, \left[\dI_i^{(k)} \right]^{\vee} = \bzero \quad \forall\, k=1,\ldots,10
\label{eq:rework2}
\end{equation}
% \revg{\begin{center} THIS IS THE FIRST TIME WE WRITE $\forall{}k$ VERSUS JUST $k$??  Should we be consistent?\end{center}}
% \rev{
% \begin{center}
% \sc{Good catch. Added $\forall$ for all :) }
% \end{center}
% }

% \canp{For further simplification, let $\vKJi = \Bi\T \, \vK_{\pred{i}} $.}
% \revg{\begin{center}WE DEFINE $\vKJi$ HERE BUT DO NOT USE.  SIMILAR TO $\vVJi$, SHOULD WE DEFINE IT AS PART OF LEMMA 2 (making (43) = (45))??  I feel there are enough definitions here and like (45) having three pieces to match the three helper functions.  
% \end{center}}
% \rev{
% \begin{center}
% \sc{100\% Agreed. Brings more symmetry here and in the lemmas.}
% \end{center}
% }

As our last stage of simplification, we use the final helper function $\vA(\vPhi)$ from \eqref{eq:A_InertialParams}, which gives the rate of change in how parameters map to the parent with changes in joint angle. It follows from the definition of $\vA(\vPhi)$ that $\left[\dI^{(k)}\right]^{\vee} =\vA(\vPhi_i)^k\, \vpi_i$.
This property finally enables refactoring of \eqref{eq:rework2} as
\begin{equation}
\vK_{\pred{i}}\T \, \Bi \, \vA(\vPhi_i)^k \, \dpi_i  = \bzero \quad \forall\, k=1,\ldots,10
\label{eq:Cyea}
\end{equation}

\subsub{Summary} Inertia transfers \eqref{eq:parameter_transfer_assignment} across any joint~$i$ are unobservable to the kinetic energy if they satisfy:
\renewcommand{\myobjscale}{.97}
\begin{empheq}[box=\fbox]{align}
\!\scaleobj{\myobjscale}{\bzero}  & \scaleobj{\myobjscale}{= \MomentumCondition{i}} && \!\!\! \scaleobj{\myobjscale}{\textrm{(Momentum)}}  \label{eq:Momentum_Fixed} \\
\!\scaleobj{\myobjscale}{\bzero}  &\scaleobj{\myobjscale}{= \vK_{\pred{i}}\T \, \Bi \, \vA(\vPhi_i)^k \, \dpi_i } && \!\!\!\scaleobj{\myobjscale}{\textrm{(\Mapping)}}  \label{eq:Mapping_Fixed}\\
    & &&\hspace{-70px}\scaleobj{\myobjscale}{\forall\, {k = 1, \ldots, 10}}\nonumber
\end{empheq}

\section{Kinematic Chains Of Bodies: Recursion, Theorem, and Algorithm}
\label{sec:fb}

With the results of the previous sections in place, a remaining hurdle is how to compute bases for the attainable velocity spans $\Vspan{i}$ and identifiable parameter spans $\mathcal{K}_i$. In this section, recursive application of controllability and observability from linear systems theory is shown to play the key role needed. We develop these steps, and then present our main algorithm. While the focus remains on serial chains, extensions to tree-structure systems and multi-DoF joints are presented in Section \ref{sec:extensions}.

% This section builds further by using the previous results to characterize identifiability for fixed-base kinematic chains. We consider a serial chain of $N$ bodies connected via single-DoF joints. Extensions to tree-structure systems and multi-DoF joints are discussed in Section \ref{sec:extensions}. 
% %The remainder of the section focuses on how to make this theory computable.
% %While the previous sections focused on developing theory, this section focuses on how to make that theory computable. 
% In particular, we consider how to compute bases for the attainable velocity spans $\Vspan{i}$ and identifiable parameter spans $\mathcal{K}_i$. Recursive application of controllability and observability from linear systems theory plays a key role in computing bases for these sets.% starting from the root and proceeding across each joint.

\subsection{Velocity Spans}

Bases $\vV_i$ for the velocity spans can be computed starting from ${\vV_0 = \bzero_{6 \times 1}}$ and proceeding outward, as in Fig.~\ref{fig:Controllability}.

\begin{figure}
\center
\includegraphics[height = .6 \columnwidth]{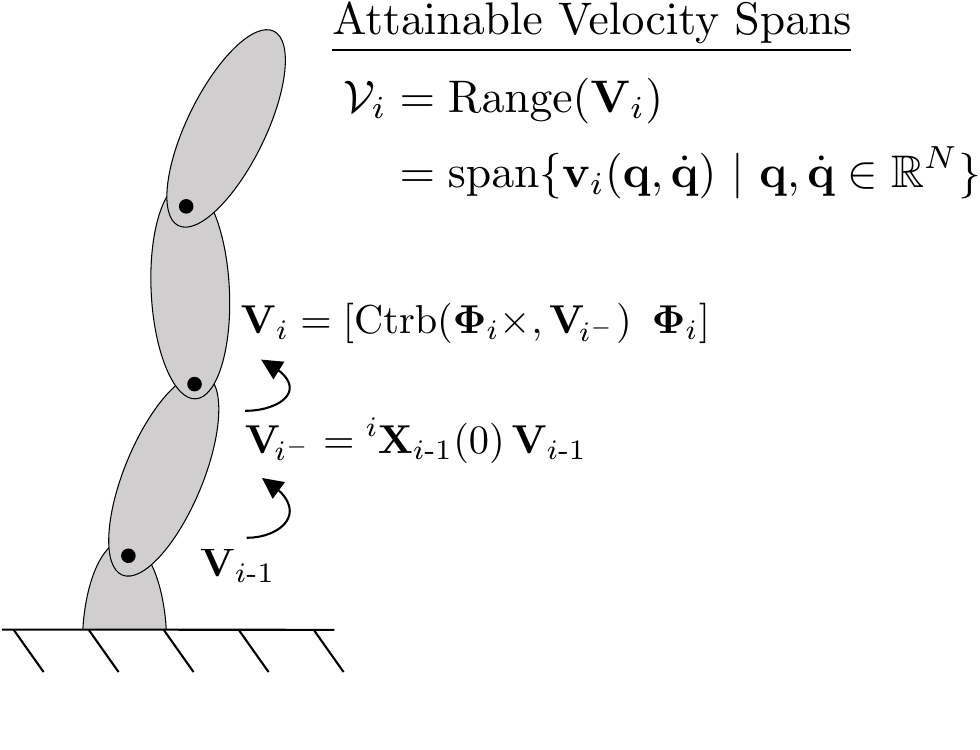}
\vspace{-10px}
\caption{Controllability analysis is applied recursively to obtain a basis $\vV_i$ for the span of the attainable velocities for each body, starting with $\vV_0 = \bzero_{6 \times 1}$.}
\vspace{-10px}
\label{fig:Controllability}
\end{figure}

\begin{lemma}
\label{corollary:velocityProp}
Suppose a matrix $\vV_{\pred{i}}$ such that $\Vspan{\pred{i}} = {\rm Range}(\vV_{\pred{i}})$. Consider the quantities
\begin{align*}
\vVJi &= \XJ{i} \, \vV_{\pred{i}}  \\
\vVim &= {\rm Ctrb}\!\left(\, ({\boldsymbol \Phi}_i \times) ,\vVJi \,\right) \\
      &=[ \vVJi,\, ({\boldsymbol \Phi}_i \times)\,\vVJi ,\ldots , ({\boldsymbol \Phi}_i \times)^5\,\vVJi ]
\end{align*}
where ${\rm Ctrb}\!\left(\, ({\boldsymbol \Phi}_i \times) ,\vVJi \,\right)$ is the controllability matrix \citep{Rugh96} associated with the pair $\left(\, ({\boldsymbol \Phi}_i \times) ,\vVJi \,\right)$.  Then, the matrix
\begin{equation}
\mathbf{V}_i = \left [\, \vVim ~ \boldsymbol{\Phi}_i \, \right] 
\label{eq:ViProp}
\end{equation}
satisfies $\Vspan{i} = {\rm Range}(\vV_i)$, 
\end{lemma}
\begin{proof} See Appendix \ref{sec:app:proof_lemma1}. 
\end{proof}

To provide intuition into this result, $\vVJi$ first transforms $\vV_{\pred{i}}$ {\em across the link}. Then $\vVim={\rm Ctrb}\!\left(\, ({\boldsymbol \Phi}_i \times) ,\vVJi \,\right)$  captures all possible velocity effects from the predecessor following transformation {\em across the joint}. Finally, the second term  $\boldsymbol{\Phi}_i$ in \eqref{eq:ViProp} adds in relative velocities from the joint itself.

\subsection{ Identifiable Parameter Coefficient Spans }
\label{sec:K_Observability}

%Using the parameter transform matrices $\Bi$ and parameter rate matrices $\vA(\vPhi_i)$, 
Using the previous helper function definitions, $\vK_i$ can be computed recursively, starting from $\vK_0= \bzero_{10\times 1}$, and propagating outward via the following lemma.

\newcommand{\overtext}[1]{
\begin{minipage}[t]{1.05in}
\begin{flushright}
#1 
\end{flushright}
\end{minipage}
}
\begin{figure*}
\begin{overpic}[grid=false,tics=5, width=\textwidth]{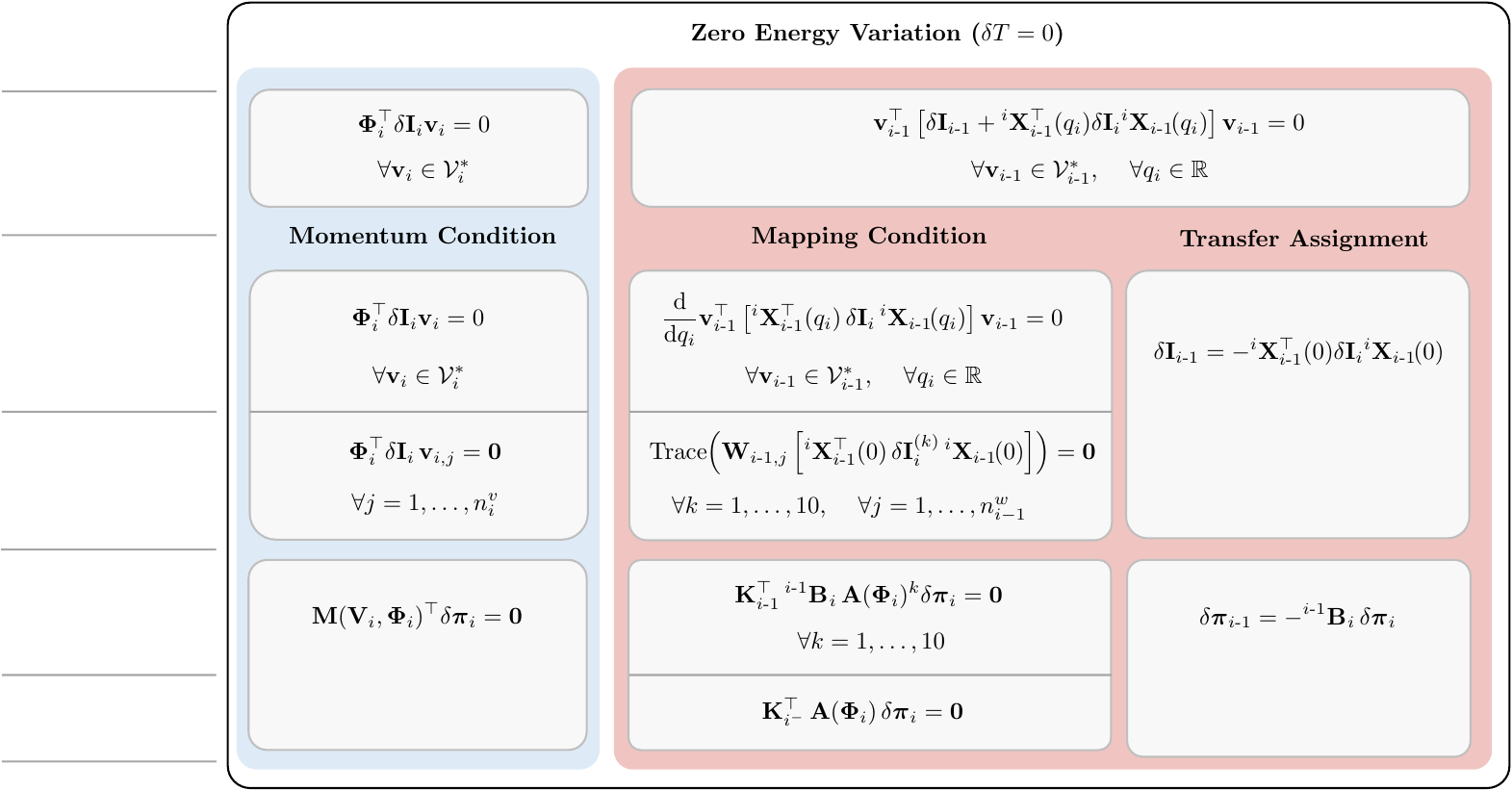}
\put(-3,43.5){\overtext{Infinite Form\\ (Coupled)\\ \bf Section~\ref{sec:Two_Bodies_Coupled}}}
\put(-3,31.5){\overtext{Infinite Form\\ (Decoupled)\\ \bf Section~\ref{sec:Two_Bodies_Simplify}}}
\put(-3,21.5){\overtext{Finite Form\\ \bf Section~\ref{sec:Two_Body_Finite}}}
\put(-3,12.5){\overtext{via Parameters\\ \bf Sec.~\ref{sec:MomentumParams} \& \ref{sec:MappingParams}}}
\put(-3,5.5){\overtext{Alt.~From\\ \bf Section~\ref{sec:K_Observability}}}
\put(36,40){\eqref{eq:twoBodyTcond2}}
\put(94,40){\eqref{eq:twoBodyTcond1}}
\put(35,26){\eqref{eq:Momentum_General}}

\put(69.5,26){\eqref{eq:MappingCondition_derivative}}
\put(94,25){\eqref{eq:Transfer_Assignment}}
\put(35,18){\eqref{eq:MomentumFinite}}
\put(69.5,18){\eqref{eq:mappingFinite}}
\put(35,8.5){\eqref{eq:Momentum_Fixed}}
\put(69.5,8.5){\eqref{eq:Mapping_Fixed}}
\put(93,8.5){\eqref{eq:parameter_transfer_assignment}}
\put(69.5,3.5){\eqref{eq:Mapping_Parameters}}
%, \eqref{eq:Transfer_Assignment}}

\end{overpic}
\caption{Summary of conditions for a parameter transfer at joint $i$ to be unobservable to the kinetic energy.}
\label{fig:summary}
\end{figure*}

\begin{lemma} 
\label{lemma:obs}
Suppose a matrix $\vK_{\pred{i}}$ such that $\mathcal{K}_{\pred{i}} = {\rm Range}(\vK_{\pred{i}})$. Consider the quantities 
\begin{align}
\vKJi &= \Bi\T \, \vK_{\pred{i}} \nonumber \\
\vKim &= {\rm Ctrb}(\, \vA(\vPhi_i)\T, \, \vKJi \,) \label{eq:Ki_minus} \\
      &= \left[ \vKJi,\ \vA(\vPhi_i)\T\,\vKJi ,\ \ldots ,\  \vA(\vPhi_i)^9{}\T\,\vKJi \right] \nonumber
\end{align}
Then, the matrix 
\begin{align*}
\vK_i &= \begin{bmatrix} \vKim  & \MomentumColumns{i}  \end{bmatrix}
\end{align*}
satisfies $\mathcal{K}_i = {\rm Range}(\vK_i)$.
\end{lemma}
\begin{proof} See Appendix \ref{sec:Proof_Observability}, which proves the equivalent characterization in Remark \ref{rem:observe} below.
\end{proof}

Similar to before, $\vKJi = \Bi\T \, \vK_{\pred{i}} $ transforms the identifiable parameters of body $\pred{i}$ {\em across the link},  while $\vKim  = {\rm Ctrb}(\, \vA(\vPhi_i)\T, \, \vKJi \,)$ then captures all possible transformations {\em across the joint}. Finally, the extra columns $\MomentumColumns{i}$ add on additional parameters for body $i$ that can be identified from measuring the torque at joint $i$.

As an additional outcome, the \mapping condition \eqref{eq:Cyea} can be written as
\begin{equation}
\ObservabilityRows{i} \, \vA(\vPhi_i)\, \dpi_i = \bzero\,. \nonumber
%\label{eq:ObservabilityNullSpaceG}
\end{equation}

\begin{remark}
\label{rem:observe}
Due to the duality between controllability and observability \cite{Rugh96}, we can equivalently view Lemma~\ref{lemma:obs} as characterizing the identifiable parameter combinations via an observability matrix:
\[
\vK_i\T = \begin{bmatrix} {\rm Obs}(\,\vKJiT ,\, \vA(\vPhi_i) \,)  \\[1ex]
 						\MomentumRows{i}  \end{bmatrix}
\]
where ${\rm Obs}(\,\vKJiT ,\, \vA(\vPhi_i) \,)$ gives the observability matrix for the pair $( \, \vKJiT ,\, \vA(\vPhi_i) \,)$.
\end{remark}

% Finally, we denote:
% % \begin{equation}
% % \cang{
% % \vO_{i} = {\rm Ctrb}(\, \vA(\vPhi_i)\T, \, \vKJi \,) \nonumber
% %\label{eq:ObservabilityGravPower}
% %}
% %\end{equation}
% %\revg{
% \begin{align*}
% \ObservabilityColumns{i} &= {\rm Ctrb}(\, \vA(\vPhi_i)\T, \, \vKJi \,)
% \\
% &= \left[ \vKJi,\ \vA(\vPhi_i)\T\,\vKJi ,\ \ldots ,\  \vA(\vPhi_i)^9{}\T\,\vKJi \right]
% \end{align*}
%}
%\revg{I ADMIT I DON'T UNDERSTAND THE FOLLOWING STATEMENT.  CAN WE PLEASE TALK THIS THROUGH?\\}
%\canp{since parameter variations satisfying $\vO_{i}\T \dpi_i = \bzero$ imply that $\dpi_i$ never map onto observable parameters of the parent for any configurations of the connecting joint. 
%Parameters $\dpi_i \in {\rm Null}(\vO_i)$ do not modify kinetic or potential energy when transformed to the predecessor for any configuation of joint~$i$. 
%By contrast,} 
%\canp{which enforces that all {\em changes} in how $\dI_i$ is mapped to the parent appear on unobservable parameters of the parent.}

\subsection{Summary}
An inertia transfer $\dpi_{\pred{i}} = - \Bi \, \dpi_i$ across joint~$i$ is unobservable to the kinetic energy if it satisfies:
\begin{empheq}[box=\fbox]{align}
\MomentumCondition{i} = \bzero      && \textrm{(Momentum)   }  \label{eq:Momentum_Parameters} \\
\ObservabilityRows{i}\, \vA(\vPhi_i)\, \dpi_i = \bzero && \textrm{(\Mapping) } \label{eq:Mapping_Parameters}
\end{empheq}
A summary of the steps leading to these final conditions is provided in Figure~\ref{fig:summary}.

%which captures that any changes in the inertial transfer with $q_i$ must be unobservable to the gravitational power.   
% purely using inertial parameters 
% \begin{align}
% \mathcal{T}_i =  \{ \dpi~| ~ \dpi_0 \in &\, \mathbb{R}^{10},~ 
% \dpi_{\pred{i}} = - \Bi \, \dpi_i , \nonumber\\ 
% \bzero &= \vN_i\, \delta \nonumber \vpi_i\\ 
% \bzero &= \dpi_j {\rm{~if~}} j\notin\{i,\pred{i}\} ~\} \nonumber
% \end{align}
 Both conditions can further be written together using a transfer nullspace descriptor $\vN_i$ %\footnote{Recall that the rows of $\vK_i$ characterize the parameters of body $i$ that are identifiable in combination with the parameters of body $i$ and/or previous bodies. By comparison, the rows of $\vN_i$ characterize the parameters of body $i$ that are identifiable in combination with others of body $i$ alone.} 
 given by
%can be formed by collecting the requirements from \eqref{eq:HijConditionViaC} and \eqref{eq:ObservabilityConditionH} for the kinetic energy nullspace with \eqref{eq:VGCondition1} and \eqref{eq:ObservabilityNullSpaceG} for the gravitational nullspace. 
\begin{equation}
\vN_i = \begin{bmatrix} \MomentumRows{i} \\[1ex]
\ObservabilityRows{i} \, \vA(\vPhi_i)\\[1ex]
\end{bmatrix} \nonumber
\end{equation}
In summary, Lemmas \ref{corollary:velocityProp} and  \ref{lemma:obs} provide the main recursive steps for computing the spans $\mathcal{V}_i$ and $\mathcal{K}_i$
%the transfer subspaces across each joint
, with the transfer nullspace descriptors $\vN_i$ then collecting the momentum and mapping conditions together in the single equation $\vN_i \, \dpi_i = \bzero$.

%then characterizing each set of undetectable transfers $\mathcal{T}_i$.

\subsection{Theorem}

\begin{theorem}
\label{thm:main_rev}
(Main Result, Restated) Consider a serial-chain rigid-body system in the absence of gravity, with the following inertia transfer subspaces for each joint ($i \in \{1,\ldots,\nbod\}$):
\begin{align}
\mathcal{T}_i =  \{ \dpi \in \mathbb{R}^{10N}~| ~ \dpi_0 \in &\, \mathbb{R}^{10},~ 
\dpi_{\pred{i}} = - \Bi \, \dpi_i , \nonumber\\ 
\bzero &= \vN_i\, \delta \nonumber \vpi_i\\ 
\bzero &= \dpi_j {\rm{~if~}} j\notin\{i,\pred{i}\} ~\} \nonumber
%\label{eq:finalTi}
\end{align}
% The structurally unobservable parameter subspace $\mathcal{N}$ satisfies
% \begin{equation}
% \mathcal{N} = \bigoplus_{i=1}^{\nbod} \mathcal{T}_i \nonumber
% \end{equation} 
Then, the structurally unobservable parameter subspace $\mathcal{N}$ satisfies $\mathcal{N} = \mathcal{T}_1 \oplus \cdots \oplus \mathcal{T}_N$.
\end{theorem}
\begin{proof} The proof follows from the same logic as Thm.~\ref{thm:main}.
\end{proof}

% \begin{remark}
% Equivalently, the transfer subspace may be written using inertia matrix variations as:
% \begin{align}
% \mathcal{T}_i =  \{ \dpi \in \mathbb{R}^{10 \nbod}\!\!\!&\,\,\,~|~ \dpi_0 \in \mathbb{R}^{10},~ \dI_j = \bzero {\rm{~if~}} j\notin\{i,\pred{i}\} \nonumber\\
% \dI_{\pred{i}} &= - \XJT{i}\, \dI_i \, \XJ{i}, \nonumber\\ 
% \bzero &= \vPhi_i\T \, \dI_i \,\vV_i ,  \nonumber\\
% \bzero &= \vK_{\pred{i}}\T \left[\XJT{i}\, \dI_i^{(k)}\, \XJ{i} \right]^{\vee}\nonumber \\ 
% & ~~~~~~~~~~~~~~~~~~~~\forall {k = 1, \ldots, 10} \nonumber 
%  \nonumber
% %\label{eq:finalTi}
% ~\}
% \end{align}
% \end{remark}

\begin{remark}
The above theorem only applies to the gravity-free case. In fixed-base robots, gravitational forces provide opportunity to identify additional mechanism parameters, decreasing the dimensionality of $\mathcal{N}$. The following section provides a simple method to address these effects within a recursive algorithm.
\end{remark}

\subsection{Addressing Gravity}
\label{sec:gravity}

%Theorem \ref{thm:main} only considers gravity-free systems. However, for fixed-base robots gravity can help to identify additional inertial parameters. 
Within rigid-body dynamics algorithms, effects of gravity are often addressed by fictitiously accelerating the base opposite gravity. This trick is applied in the Recursive-Newton-Euler algorithm \citep{LuhWalkerPaul80b} for inverse dynamics and the articulated-body algorithm for forward dynamics \citep{Featherstone08}. The same approach also works to address gravitational effects for identifiability. By seeding $\vV_0 = {}^0 \va_g$, with ${}^0 \va_g$ the gravity acceleration in the world coordinate, the recursive computations of this section result in modified transfer subspaces $\mathcal{T}_i$ that include gravitational considerations. Intuitively, this modification corresponds to adding a fictitious prismatic joint aligned with gravity at the base whose force is not measured. Appendix~\ref{sec:Conditions_g} rigorously analyzes the role of gravity on identifiability and further justifies this simple modification.    

\subsection{Algorithm Summary}

 Algorithm \ref{alg:RPNA} provides a compact method to recursively compute the parameter nullspace descriptors $\vN_i$. We name this method the Recursive Parameter Nullspace Algorithm (RPNA). As a practical matter, linearly dependent columns of $\vV_i$ or $\vK_i$ can be removed at any step in the algorithm. %These modifications minimize the size of each $\vN_i$ without impact on its nullspace. 
 A {\sc Matlab} implementation of the RPNA is provided open source at the following link: \url{https://github.com/pwensing/RPNA}.

\begin{remark}
The RPNA has described all unobservable parameter combinations through a direct sum of local transfers. The nullspace descriptors $\vN_i$ can also be used to compute bases for the system-wide parameter nullspace $\mathcal{N}$ and its orthogonal complement $\mathcal{N}^\perp$. Details are provided in Appendix~\ref{subsec:ParameterNullspace}. Any basis for $\mathcal{N}$ gives the unidentifiable parameter combinations for the mechanism, while any basis for $\mathcal{N}^\perp$ gives identifiable parameter combinations. 
\end{remark}

\begin{remark}
One might alternatively be interested in considering the parameter nullspace under static experiments.
\[
\mathcal{N}_{\rm static} = \{ \delta \vpi ~|~ \vY(\vq,\bzero,\bzero) \delta \vpi=\bzero, ~~~ \forall \vq \in \mathbb{R}^N\}
\]
Characterizing identifiability in this way would mean only considering the term $\vg$ in \eqref{eq:eom}, which could be of interest for developing a gravity-compensation model.
%Further, a quasi-static robot might reasonably be modeled ignoring parameter combinations in $\mathcal{N}_\vg$. 
An algorithm for this set is obtained by modifying line \ref{line:Vi} of the RPNA to $\vV_i = \vVim$ 
%$\vV_i = \left[{\rm Ctrb}( ({\vPhi_i {\times}}),  \vV_{i^-})\right]$ 
which intuitively removes local joint velocities from consideration. See Appendix \ref{sec:Conditions_g} for justification.
\end{remark}

%\revg{REGARDING REMARK 7: Should we say this means looking only at $\vec{g}(\vec{q})$ in (1) and/or only identifying the gravitation model?}

%\revg{--------------}

\IncMargin{.5em}
\SetInd{.85em}{0em}
\setstretch{1.5}

\let\oldnl\nl% Store \nl in \oldnl
\newcommand{\nonl}{\renewcommand{\nl}{\let\nl\oldnl}}% Remove line number for one line

\begin{algorithm}[t]

$\vV_0 = {}^0 \va_g,~\vK_0 = \bzero_{10\times 1}$ \\[1ex]
\For {$i=1,\ldots, \nbod$} {
    \nonl ~ \\[-.25ex]
    \nonl
    \algdesc{Transfer across the link}\\[1ex]
	$\vVJi= \XJ{i}\, \vV_{\pred{i}} $\\[1.5ex]
    $\vKJi = \Bi\T \, \vK_{\pred{i}}$\\[2.5ex]

    \nonl
    \algdesc{Transfer across the joint}\\[1ex]
    $\vVim = {\rm Ctrb}( ({\vPhi_i {\times}}), \quad\vVJi)$\\[1ex]
	$\vKim = {\rm Ctrb}( \, \vA(\vPhi_i)\T ,\, \vKJi)$ \\[2ex]

    \nonl
    \algdesc{Add the current joint's effect}\\[1ex]
    $\vV_i = \begin{bmatrix}\vVim &\vPhi_i \end{bmatrix} $
    \label{line:Vi} \\[1.5ex]
	$\vK_i = \begin{bmatrix}  \vKim & \vC( \vV_i , \vPhi_i )  \end{bmatrix}$ \\[2.5ex]

    \nonl
    \algdesc{Collect the momentum and mapping conditions}\\[1ex]
	$\vN_i =  \begin{bmatrix} \vC( \vV_i , \vPhi_i )\T \\[1ex]
    \ObservabilityRows{i}\, \vA(\vPhi_i) 
	\end{bmatrix}
	$\\[1.5ex]
	}
 \nonl
\Return {$ \vN_i\,, \quad i = 1,\ldots,\nbod$}\\
  \caption{Recursive Parameter Nullspace Algorithm for Serial-Chains with Single-DoF Joints% - SAME AS BEFORE. JUST TESTING/PLAYING WITH NOMENCLATURE TO MAKE THE SYMMETRY/REMARKS EASIER?  REMARK7:= Line 7: $\vV_i = \vec{C}_i$.  NO TORQUE:= Line 8: $\vS_i=0$ (if correct??).... NOT SURE THIS IS WORTH ANYTHING???? \todo{$K_i+$ and $V_i^+$, etc. Ditch S}
  }
  \label{alg:RPNA}
\end{algorithm}
\DecMargin{.5em}
\setstretch{1}

\section{Extensions for Open-Chain Systems}
\label{sec:extensions}

This section considers extending the RPNA algorithm to tree-structure systems with multi-DoF joints. These extensions build toward the analysis of floating-base systems (e.g., mobile legged systems). We also reflect on our energy analysis so far and revisit how it relates to dynamics, which also provides a new perspective on an important past result of \citet{Ayusawa14}. 

%This section considers a number of extensions to the RPNA algorithm that follow readily from the previous developments. %We first consider a generalization to multi-DoF joints, unifying the treatment of floating-base and grounded cases, as well as opening possibilities for general joint structures. The second subsection comments on generalization to tree-structure systems, while the final-subsection offers preliminary remarks on generalization to closed-loop systems. In particular, we examine the identifiability of motor components, which imposes perhaps both the simplest and most common closed loop found in conventional robots.  

\subsection{Tree-Structure Systems}
\label{sec:extension:branched}
The inertia transfer concept readily generalizes to branched open-chain rigid-body systems. 
%and decomposed their parameter nullspace into a sequence of transfers across joints. This decomposition readily generalizes to the branched case. 
In branched systems, each body has a predecessor, denoted $p(i)$, toward the base. The transfer assignment is re-written as
\[
\dI_{p(i)} = - \XMT{i}{p(i)}\!(0)\, \dI_i \, \XM{i}{p(i)}(0)
\]
All recursive steps of the RPNA generalize to branched systems by likewise replacing $\pred{i}$ with $p(i)$.

\subsection{ Multi-DoF Joints}

Suppose each joint~$i$ has $n_{d_i}$ DoFs with free modes: 
\[
\vPhi_i = \begin{bmatrix} \boldsymbol{\phi}_{i,1} & \cdots & \vphi_{i,n_{d_i}} \end{bmatrix}
\]
where each $\vphi_{i,j} \in \mathbb{R}^6$ is fixed. A fixed free-mode matrix of this form could accommodate, for example, spherical or floating-base joints. 

%With this convention the velocity propagation \eqref{eq:difkin_main} takes the form
%\begin{equation}
%\vv_i = \XM{i}{\pred{i}}(\vq_i)\, \vv_{\pred{i}} + \vPhi_i \vnu_i \nonumber
%\end{equation}
%where  $\vnu_i\in \mathbb{R}^{n_{d_i}}$ gives the generalized joint velocity.
% For a floating-base system, the mapping condition is generalized for all joint rates by enforcing 
% \[
% \bzero = (\vphi_{i,j} \times)\T \, \dI_i + \dI_i \, (\vphi_{i,j}\times) \quad \forall j\in\{1, \ldots, n_{d_i}\} 
% \]
% The transfer assignment and momentum condition generalize directly.% The transfer condition is unaffected by the structure or number of the joint free modes. 

In any case, additional modifications are required to accommodate multi-DoF joints in the RPNA. First, the propagation of the attainable velocity spans must be generalized. For a single-DoF joint, joint kinematics follow a linear system \eqref{eq:change_of_X_main}. In contrast, for a multi-DoF joint: %with $\vPhi_i = [\vphi_{i,1},\ldots, \vphi_{i,n_{d_i}}]$ 
\[
\scaleobj{.95}{\frac{\rm d}{ {\rm d} t} \XM{i}{p(i)}(\vq_{i}) \in {\rm span}\left( (\vphi_{i,1} \times)\XM{i}{p(i)},\, \ldots ,\, (\vphi_{i,n_{d_i}} \times)\XM{i}{p(i)}  \right) }
\] 
As a result, the span
\[
{\rm span}\{   \vv ~|~ \exists \vq_{i} , \vv \in {\rm Range}( \XM{i}{p(i)}(\vq_i)\, \vV) \} 
\]
can be seen as the smallest set containing ${\rm Range}(\XM{i}{p(i)}(\bzero)\, \vV)$ that is invariant under each $(\vphi_{i,k} \times)$. This set is equivalent to the controllable subspace of a switched linear system \citep{Sun02} with pairs $\{\, (\,(\vphi_{i,k} \times), \XM{i}{p(i)}(\bzero) \,\vV\,) \,\}_{k=1}^{n_{d_i}}$. If the controllability matrix in Lemma~\ref{corollary:velocityProp} is replaced by a matrix whose range equals the switched controllable subspace of these pairs, then Lemma~\ref{corollary:velocityProp}  holds more generally. The conditions in Lemma~\ref{lemma:obs} generalize to multi-DoF joints in a similar manner. %using observability for switched linear systems. 
With these developments, Algorithm~\ref{alg:RPNA_multi_dof} provides a modification to the RPNA for open-chain systems with multi-DoF joints. The algorithm relies on an updated definition of the matrix $\vM(\vV,\vPhi)$ as:
\begin{equation}
\vC(\vV, \vPhi) = \left[ \frac{\partial}{\partial \vpi} {\rm Vec}(\vV\T\, [\, \vpi \,]^{\wedge}\, \vPhi) \right]\T
%\label{eq:Cmom}
\end{equation}
where the ${\rm Vec}(\cdot)$ operation stacks the columns of a matrix into a vector.

%Other than the switched controllability matrices, 

\IncMargin{.5em}
\SetInd{.85em}{0em}
\setstretch{1.5}

\begin{algorithm}[t]

%$V \leftarrow \emptyset$; $E \leftarrow \emptyset$\;

$\vV_0 = {}^0 \va_g,~\vK_0 = \bzero_{10\times 1}$ \\[1ex]
%$\mathcal{K}_0 = \{ \bzero_{6\times 6} \}$\\[1ex]
%$\mathcal{P}_0 = \{ \bzero_{6\times 6} \} $\\[1ex]
\For {$i=1,\ldots, \nbod$} {
        \nonl ~ \\[-.25ex]
        \nonl
    
        \algdesc{Transfer across the link}\\[1ex]
	$\vVJi= \XM{i}{p(i)}(0)\, \vV_{p(i)} $\\[1.5ex]
        $\vKJi = {}^{p(i)} \vB_{i^-}\T \, \vK_{p(i)}$\\[2.5ex]
         \nonl
        \algdesc{Transfer across the joint}\\[1ex]
         $\vVim = {\rm Ctrb}\left(\, \left\{\, \left (\, ({\vphi_{i,k} {\times}}), \vVJi\,\right)\,\right\}_{k=1}^{n_{d_i}} \, \right) $ \\[1.5ex]
        $\ObservabilityColumns{i} = {\rm Ctrb}\left( \, \{ \, ( \vA(\vphi_{i,k})\T ,\, \vKJi ) \}_{k=1}^{n_{d_i}} \,\right)$ \\[2.5ex]       
        \nonl
        \algdesc{Add the current joint's effect}\\[1ex]
         $\vV_i = \begin{bmatrix} \vVim &\vPhi_i \end{bmatrix} $\\[1.5ex]
	$\vK_i = \begin{bmatrix}  \ObservabilityColumns{i} & \MomentumColumns{i}  \end{bmatrix}$ \\[2.5ex]
        \nonl
        \algdesc{Collect the momentum and mapping conditions}\\[1ex]
	$\vN_i =  \begin{bmatrix} \MomentumRows{i}  \\[1ex]
    \ObservabilityRows{i}\, \vA(\vphi_{i,1}) \\ \vdots \\ \ObservabilityRows{i}\, \vA(\vphi_{i,n_{d_i}}) 
	\end{bmatrix}
	$\\[1.5ex]
	}
 \nonl
\Return {$ \vN_i\,, \quad i = 1,\ldots,\nbod$}\\
  \caption{Recursive Parameter Nullspace Algorithm for Open-Chain Systems with Multi-DoF Joints}
  \label{alg:RPNA_multi_dof}
\end{algorithm}
\DecMargin{.5em}
\setstretch{1}

\IncMargin{.5em}
\SetInd{.85em}{0em}
\setstretch{1.5}

\begin{algorithm}[t]

%$V \leftarrow \emptyset$; $E \leftarrow \emptyset$\;
%$\vV_0 = {}^0 \va_g,~\vK_0 = \bzero_{1\times 10}$ \\[1ex]
%$\mathcal{K}_0 = \{ \bzero_{6\times 6} \}$\\[1ex]
%$\mathcal{P}_0 = \{ \bzero_{6\times 6} \} $\\[1ex]
$\vN_1 = \bone_{10\times10}$ \\[1ex]
\For {$i=2,\ldots, \nbod$} {
	\textbf{}%$\vVJi= \XM{i}{p(i)}(0)\, \vV_{p(i)} $\\[1.5ex]
        %$\vKJi = {}^{p(i)} \vB_{i^-}\T \, \vK_{p(i)}$\\[1.5ex]
         
         %$\vV_i = \begin{bmatrix}{\rm Ctrb}\left(\, \left\{\, \left (\, ({\vphi_{i,k} {\times}}), \vVJi\,\right)\,\right\}_{k=1}^{n_{d_i}} \, \right) &\vPhi_i \end{bmatrix} $\\[1.5ex]

	%$\vO_{i} = {\rm Ctrb}\left( \, \{ \, ( \vA(\vphi_{i,k})\T ,\, \vKJi ) \}_{k=1}^{n_{d_i}} \,\right)$ \\[1ex]
	%$\vK_i = \begin{bmatrix}  \vO_i & \vC( \vV_i , \vPhi_i )\T  \end{bmatrix}$ \\[1.5ex]
	$\vN_i =  \begin{bmatrix} \vC( \bone_{6\times6} , \vPhi_i )\T  \\[1ex]
    \vA(\vphi_{i,1}) \\ \vdots \\ \vA(\vphi_{i,n_{d_i}}) 
	\end{bmatrix}
	$\\[1.5ex]
	}
 \nonl
\Return {$ \vN_i\,, \quad i = 1,\ldots,\nbod$}\\
  \caption{Recursive Parameter Nullspace Algorithm for Open-Chain Floating-Base Systems}
  \label{alg:RPNA_floating_base}
\end{algorithm}
\DecMargin{.5em}
\setstretch{1}

% The propagation of the outer product spans $\mathcal{K}_i$ generalizes analogously.
%We implement such generalizations in the included {\sc Matlab} code.% without further dwelling on them here in the main text.
                                                                                                                                                                                                                                                                                                                                                                                                                                                                                                                                                                                                                                                                                                                                                                                                                                                                                                                                                                                                                         %

%One can consider the switched observable subspace with pairs $\{\, (\,\vC_i,\, \vA(\, \vphi_{i,k})\,)\, \}_{k=1}^{n_{d_i}}$. This switched unobservable subspace is the largest subspace of ${\rm Null}( \vC_i )$ that is invariant under each $\vA(\, \vphi_{i,k})$. If the observability matrix $\vO_{i}$ is replaced by any matrix whose nullspace coincides with this switched unobservable subspace, then the condition \eqref{eq:ObservabilityConditionH} is generalized as
%\[
%\begin{bmatrix} \vO_{i} \, \vA( \vphi_{i,1}) \\ \vdots \\ \vO_{i}\, \vA( \vphi_{i,n_{d_i}}) \end{bmatrix} \, \dpi_i = 0 
%\]
%Simple algorithms for switched controllability and observability analysis are included in accompanying code.

%\subsubsection{Spherical Joint} Building from Remark \ref{rem:Floating_Multi_Dof}, the conditions on undetectable transfers for a spherical joint ${\vPhi_2 = [\bone_3,\, \bzero_3]\T}$ arise from the intersection of conditions for revolute joints about $x$, $y$, and $z$ axes. As a result, undetectable transfers across a spherical joint have one degree of freedom. This degree of freedom has an interpretation as an exchange of mass $\delta m$ at the joint center. 

%The mathematics generalizes trivially to the case of tree-structure systems. 

\subsection{Floating-Base Systems}
\label{sec:FloatingBaseChains}

With the generalization to multi-DoF joints, Alg.~\ref{alg:RPNA_multi_dof} can be applied directly to floating-base systems. Note, however, that the six degrees of freedom of the floating base can enable one to simplify the algorithm. For example, since the possible velocities of any link satisfy $\mathcal{V}_i^*=\mathbb{R}^6$, basis matrices $\vV_i$ and $\vK_i$ can be chosen as identity matrices $\vV_i=\mathbf{1}_{6\times 6}$, $\vK_i=\mathbf{1}_{10\times10}$, with $\ObservabilityColumns{i}$ correspondingly set to $\ObservabilityColumns{i} = \mathbf{1}_{10\times10}$ without effect on the algorithm. As a result, the null-space descriptors can be constructed directly as:
\[
\vN_i = \begin{bmatrix}
\vM(\mathbf{1}_{6 \times 6}, \vPhi_i)\T \\ 
\vA(\vphi_{i,1}) \\
\vdots\\
\vA(\vphi_{i,n_{d_i}})
\end{bmatrix}
\]
such that $\vN_i \, \delta \vpi_i = \mathbf{0}$ simply combines the original momentum \eqref{eq:Momentum_Floating} and mapping \eqref{eq:Invariance_Floating} conditions from Sec.~\ref{sec:Floating_Two_Body}. For the floating-base link (Link 1), the momentum condition is $\dI_1 = \bzero$, and so the first transfer subspace $\mathcal{T}_1$ is the zero set $ \mathcal{T}_1= \{\bzero\}$. Algorithm \ref{alg:RPNA_floating_base} shows the revised RPNA for floating-base systems.

\begin{remark} Gravity effects provide no additional identifiable parameters in floating-base systems. This result is due to the fact that the floating base can be accelerated opposite gravity to excite the same dynamic effects as does gravity itself. 
\label{rem:gravity}
\end{remark}

\begin{remark}
With the generalizations in this section, our main Theorems~\ref{thm:main} and \ref{thm:main_rev} apply to floating-base systems without motion restrictions, and certify the maximum number of parameter combinations that can be identified from a maximally exciting trajectory. However, for purely floating systems (without actuation on the base), the underactuation and constraints from physics (e.g., conservation of linear and angular momentum) can prevent resulting trajectories from being maximally exciting. For example, if the system undergoes ballistic motion without external forces other than gravity, it will lose at least one identifiable parameter (the total mass \citep{Ayusawa14}), and potentially more depending on the simplicity of the mechanism and its achievable motions. See \citet{Ayusawa14} for additional discussion.
%In this case, the total mass of the system cannot be identified \citep{Ayusawa14}. Once in contact, however, this unidentifiable parameter combination is no longer present due to the influence of interaction forces on the CoM motion. As a result, systems evolving in and out of contact have parameter identifiability that coincides with the case of unrestricted motion. 
%In the case that a system is purely floating (i.e., in free flight without any external contacts), then the total mass of the system is unobservable \cite{Ayusawa14}.  
\end{remark}

%\revg{REGARDING THE PREVIOUS REMARK: Is this because the 6DOF joint of the floating base (for truly free-floating systems) doesn't have a joint torque?  I.e. you don't have the external forces?

%But when you go into contact, you do you have contact forces then?  If not, can you still stay this?  If you do, does this assume you can make contact in arbitrary orientations?   Would love to discuss.}

%\todo{algorithm gives the maximal identifiable set
%but if you are truly ballistic, you can't control base and joint motions independently and hence don't have enough excitation. In that case, you will never be able to identify the total mass, or maybe more, depending on the complexity of the achievable joint motions. }

% \revg{REGARDING THE PREVIOUS REMARK: I'm not sure I believe the above.  Can we please talk this through?

% Ok, I think I've convinced myself of the above.  But that means the momentum condition isn't actually what the joint torque can sense.  Instead, IF WE WANT TO DROP TORQUE MEASUREMENTS, we would need a measurable torque space analogous to the velocity space???  Meaning my above remark about setting M to zero is WRONG!  Would love to discuss to 5 minutes and then probably drop entirely :)}

\newcommand{\half}{\tfrac{1}{2}}
\newcommand{\vvp}{\vv_1}
\newcommand{\vvm}{\vv_\mot}
\newcommand{\vvq}{\vv_2}
\newcommand{\vIp}{\vI_1}
\newcommand{\vIm}{\vI_\mot}
\newcommand{\vIq}{\vI_2}
\newcommand{\dIp}{\dI_1}
\newcommand{\dIm}{\dI_\mot}
\newcommand{\dIq}{\dI_2}

\newcommand{\Izzm}{I_{zz_\mot}}
\newcommand{\dIzzm}{\delta I_{zz_\mot}}

\newcommand{\qq}{q_2}
\newcommand{\qm}{q_\mot}
\newcommand{\qqdot}{\dot{q}_2}
\newcommand{\qmdot}{\dot{q}_\mot}
\newcommand{\vPhiq}{\vPhi_2}
\newcommand{\vPhim}{\vPhi_\mot}

\newcommand{\Xqp}{\XM{2}{1}}
\newcommand{\Xmp}{\XM{\mot}{1}}
\newcommand{\XqpT}{\XMT{2}{1}}
\newcommand{\XmpT}{\XMT{\mot}{1}}

\newcommand{\Xqpq}{\XM{2}{1}(\qq)}
\newcommand{\Xmpq}{\XM{\mot}{1}(\qm)}
\newcommand{\XqpTq}{\XMT{2}{1}\!(\qq)}
\newcommand{\XmpTq}{\XMT{\mot}{1}\!(\qm)}

\newcommand{\Xqpz}{\XM{2}{1}(0)}
\newcommand{\Xmpz}{\XM{\mot}{1}(0)}
\newcommand{\XqpTz}{\XMT{2}{1}\!(0)}
\newcommand{\XmpTz}{\XMT{\mot}{1}\!(0)}

\newcommand{\XqpInvq}{\XM{2}{1}^{-1}(\qq)}
\newcommand{\Xmq}{\XM{\mot}{2}}
\newcommand{\Xmqq}{\XM{\mot}{2}(\qq)}

\subsection{From Energy to Dynamics}
\label{subsec:EnergyToDynamics}
Throughout, we have primarily considered energy analysis to characterize unidentifiable parameters rather than examining the dynamics directly. These two views are equivalent: In the absence of gravity, ensuring a zero variation to the kinetic energy $\delta T(\vq,\vqd) \equiv 0$ is the same as ensuring a zero variation to the mass matrix $\delta \vH(\vq) \equiv \bzero$, which, in turn, is the same as ensuring zero variation to the total generalized force $\delta \btau_{\rm total} = \delta \left[ \vH(\vq) \vqdd + \vc(\vq,\vqd)\right] \equiv \bzero$. 

Comparing these views, however, can help determine parameter identifiability via specific or individual forces/torques on/in the mechanism. Let us consider a two-body floating-base system similar to in Sec.~\ref{sec:Floating_Two_Body}. The system has kinetic energy:
\[
T =  \frac{1}{2} \begin{bmatrix} \vvp \\ \qqdot \end{bmatrix}\T \vH(\qq) \begin{bmatrix} \vvp \\ \qqdot \end{bmatrix}
% \label{eq:TwoBodyMassMatrix}
\]
where the mass matrix takes the form:
\[
\scaleobj{.9}{\vH(\qq) = \begin{bmatrix} \vIp + \XqpT \,\vIq\, \Xqp & \XqpT \, \vIq\, \vPhiq \\[.75ex] \vPhiq\T\, \vIq\, \Xqp & \vPhiq\T\, \vIq\, \vPhiq \end{bmatrix} = \begin{bmatrix} \vH_{11} & \vH_{12} \\ \vH_{21} & H_{22} \end{bmatrix}}
\]
The generalized force in this case would be comprised of the net external force on the base $\vf_1$ and the torque $\tau_2$ at joint $2$, so that the equations of motion take the form:
\[
\begin{bmatrix} \vf_1 \\ \tau_2 \end{bmatrix} = \begin{bmatrix} \vH_{11} & \vH_{12} \\ \vH_{21} & H_{22} \end{bmatrix}  \begin{bmatrix} \dot{\vv}_1 \\ \ddot{q}_2 \end{bmatrix} + \begin{bmatrix} \vc_1 + \vg_1 \\ c_2 + g_2 \end{bmatrix}
\]
In this case, the momentum and mapping conditions for a transfer at joint $2$ would collectively enforce a zero variation $\delta \vf_1$ to the base force, and zero variation $\delta \tau_2$ to the joint torque. %[?Keep?$\rightarrow$] More generally, within an RNEA setting for an articulated floating-base system, a transfer at joint $i$ satisfying these conditions will not disturb the torque $\tau_i$ or the predecessor force $\vf_{p(i)}$ within the algorithm.
%\revg{[YES, but given that we are talking forces already, I might skip the RNEA reference?  How about this:] 
More generally, {\em without any additional motion restrictions}, these two conditions for body $i$ ensure that a transfer with the parent would not disturb the connecting joint torque $\tau_i$ nor the total 6D predecessor force $\vf_{p(i)}$ (acting from body $p(i)$ onto body $i$).

For this two-body system, a zero variation to the mass matrix $\vH$ overall is equivalent to a zero variation in its first row ($\vH_{11}$ and $\vH_{12}$) alone. This result follows due to symmetry ($\vH_{12} = \vH_{21}\T$) and the fact that the lower right block gives a projection of the upper right one ($H_{22} = \vPhi_2\T \XMT{1}{2} \vH_{12}$). Physically, this occurs since any reaction torque $\tau_2$ is balanced by and shows up on the base force $\vf_1$. Overall, the implication is that sequential excitation, i.e., accelerating the floating base in all directions and then accelerating the joint itself, will be maximally exciting. For $\delta \vf_1 \equiv \bzero$, then we necessarily have $\delta \tau_2\equiv0$ as well, such that the joint torque provides no additional information nor can identify any additional parameters not already identifiable via the base force.

We generalize this argument to open-chain floating-base systems in Appendix~\ref{sec:compareToKo}: Under proper excitation and full knowledge of net external forces, measurements of joint torques do not enable the identification of any additional parameters. This is inspired by and equivalent to a key result of \citet{Ayusawa14}.

\newcommand{\mot}{m}
\newcommand{\ngear}[1]{n_{\!R_{#1}} }
\section{Special Case Extension to Closed Kinematic Loops: Joint Motors}
\label{sec:closedChain}

Joint motors represent a simple and common closed kinematic chain,
driving the joint between two connected bodies.  The spinning rotor is
coupled to the joint, but often spins a fixed ratio faster and may
spin along a distinct axis.  Hence, we must treat the rotor as a
separate body.  Toward understanding this case, we consider a
three-body system with a floating base (body~1), a child link
(body~2), and a motor's rotor (body~$\mot$), building on Sec.~\ref{sec:Floating_Two_Body}.

We will also assume the rotor to be rotationally symmetric about the
motor axis, giving it only four inertial parameters versus the general
ten parameters.

\subsub{Dynamics and constraints}
The total kinetic energy of the three bodies is
\begin{equation*}
  T = \half \vvp\T \vIp \vvp
    + \half \vvq\T \vIq \vvq
    + \half \vvm\T \vIm \vvm
\end{equation*}
where, similar to \eqref{eq:two_body_diff_kin},
\begin{equation*}
  \begin{array}{r@{\,=\,}c@{\,}l@{\,+\,}l@{\,}l}
    \vvq & \Xqpq & \vvp & \vPhiq & \qqdot \\[1ex] 
    \vvm & \Xmpq & \vvp & \vPhim & \qmdot
 \end{array}
\end{equation*}  
are the child and rotor body velocities respectively.  The motor
position $\qm = \ngear{} \qq$ and speed $\qmdot = \ngear{} \qqdot$ are
modified by the fixed gear ratio $\ngear{}$.

Using inertia variations $\dI$ and the gear ratio $\ngear{}$, we can
factor the energy variation $\delta T$ as
\begin{equation*}
  \delta T \!=\! \half \!
  \begin{bmatrix}\vvp \\ \qqdot \end{bmatrix}\T
  \!\!
  \begin{bmatrix}
    \dIp^C \!\! &
    \XqpT   \dIq \vPhiq + \!\XmpT   \dIm \vPhim \ngear{}   \\[1ex]
    \cdot &
    \ \ \vPhiq\T \dIq \vPhiq + \ \ \vPhim\T \dIm \vPhim \ngear{}^2
  \end{bmatrix}
  \!\!
  \begin{bmatrix}\vvp \\ \qqdot \end{bmatrix}
\end{equation*}
% \underbrace{...}_{\vH}
where the composite inertia variation $\dIp^C$ is
\begin{equation*}
  \dIp^C \!=\! \dIp \!+\! \XqpTq \dIq \! \Xqpq \!+\! \XmpTq \dIm \! \Xmpq
\end{equation*}

Mirroring the developments of Sec.~\ref{sec:Floating_Two_Body},
enforcing $\delta{}T=0$, we use the top-left composite inertia
$\dIp^C$ to define a fixed transfer assignment analogous to
\eqref{eq:Transfer_Assignment}
\begin{equation*}
  \dIp = - \XqpTz \, \dIq \, \Xqpz - \XmpTz \, \dIm  \Xmpz
\end{equation*}
and uniformly ensure $\dIp^C=\mathbf{0}$ over all configurations by zeroing its
derivative.
\begin{equation*}
  \bzero = \frac{\rm d}{{\rm d} t}
  \Bigl( \XqpTq \dIq \Xqpq \!+\! \XmpTq \dIm \Xmpq \Bigr)
\end{equation*}
Using \eqref{eq:change_of_X_main}, we find the mapping condition
\begin{equation}
\label{eq:mapping_rotor}
  \begin{array}{r@{\big[}r@{}l@{+}r@{}l@{\big]}c@{}l}
  \bzero = 
  \XqpTq & (\vPhiq\times)\T & \dIq & \dIq & (\vPhiq\times) & \Xqpq & \qqdot
  \\[1ex]
  + 
  \XmpTq & (\vPhim\times)\T & \dIm & \dIm & (\vPhim\times) & \Xmpq & \qmdot
  \end{array}
\end{equation}

From the off-diagonal elements for $\delta{}T$, we obtain the momentum
condition
\begin{equation}
\label{eq:momentum_rotor}
  \bzero =
  \vPhiq\T \dIq \Xqpq + \vPhim\T \dIm \Xmpq \, \ngear{}
\end{equation}

Finally, from the bottom-right element determining $\delta{}T$, we obtain a
third condition, which we refer to as the torque condition.
\begin{equation}
  0 = \vPhiq\T \dIq \vPhiq + \vPhim\T \dIm \vPhim \, \ngear{}^2 \label{eq:torque_rotor}
\end{equation}
Note, in this three-body case the momentum condition does not automatically
guarantee the torque condition.\footnote{Relating back to Sec.~\ref{subsec:EnergyToDynamics}, a related implication here is that motor rotors are generally not identifiable from ground forces alone.}

\subsub{Symmetric rotor}
In the common case that the motor's rotor is rotationally symmetric,
these conditions simplify a great deal.  Consider a rotor spinning
about its $z$ axis such that $\vPhim=[0\,0\,1\,0\,0\,0]\T = [\hat{\mathbf{z}}\T \mathbf{0}\T]\T$.  Rotational
symmetry implies $I_{xy}=I_{xz}=I_{yz}=h_x=h_y=0$ and
$I_{xx}=I_{yy}$. Intuitively, variations respecting this symmetry
satisfy
\begin{equation}
  \bzero = (\vPhim\times)\T \dIm + \dIm (\vPhim\times)
  \label{eq:symmetric_inertia}
\end{equation}

This means a rotor only has four parameters:
$I_{zz},I_{xx}=I_{yy},h_z,m$.  Further, from the
example in Sec.~\ref{sec:example_revolute}, we know that the revolute attachment via
the $z$ axis allows the transfer of the last three to the parent body 1.
So, practically, we need only consider changes to the rotor inertia
$\delta I_{zz}$ about its axis.  This further gives
\begin{align}
  \vPhim\T \dIm \phantom{\vPhim} &= \vPhim\T \, \dIzzm \label{eq:mot_IPhi} \\
  \vPhim\T \dIm \vPhim  &= \phantom{\vPhim} \, \dIzzm \label{eq:mot_PhiIPhi}
\end{align}

\newcommand{\Xpqq}{\XM{1}{2}(q_2)}

\subsub{Simplified conditions}
We simplify the mapping condition \eqref{eq:mapping_rotor} using symmetry via \eqref{eq:symmetric_inertia},
acknowledging that $\Xqpq$ is full rank, and allowing $\qqdot$ to take
any value.  For the momentum condition \eqref{eq:momentum_rotor}, we first
post-multiply with $\Xpqq$. We then recognize that
$\vPhim\T\Xmpq=\vPhim\T\Xmpz$ since the motor rotation does not change
the motor axis, and note $\Xmqq=\Xmpz\Xpqq$ is a transform from
joint 2 to the motor.  
Finally, for the torque condition, we use
\eqref{eq:mot_PhiIPhi}.  In summary, the mapping, momentum, and torque
conditions are
\begin{align}
  \bzero &= (\vPhiq\times)\T \dIq + \dIq (\vPhiq\times)
  \label{eq:mot_mapping}\\
  \bzero &= \vPhiq\T \dIq \phantom{\vPhiq}
          + \dIzzm \ngear{} \, \vPhim\T \Xmqq
  \label{eq:mot_momentum}\\
  0 &= \vPhiq\T \dIq \vPhiq  + \dIzzm \ngear{}^2
  \label{eq:mot_torque}
\end{align}
where the mapping condition reverts to the two-body case from Sec.~\ref{sec:Floating_Two_Body}, while the
momentum and torque conditions retain the rotor inertia terms.

%\begin{equation}
%  \dIzzm = - \frac{1}{\ngear{}^2} \ \vPhiq\T \dIq \vPhiq
%\end{equation}

%\begin{equation}
%  \bzero = \vPhiq\T \dIq
%  \bigl( \mathbf{1}_{6\times 6} -
%         \frac{1}{\ngear{}} \ \vPhiq \vPhim\T \Xmqq \bigr)
%\end{equation}

Since $\dIzzm$ is the only parameter appearing for the motor, it can
at most add this one additional identifiable parameter.  However, if
both conditions \eqref{eq:mot_momentum} and \eqref{eq:mot_torque} are
satisfied, $\dIzzm$ can be transferred to $\dIp$ and it becomes
unidentifiable. 

Breaking down $\vPhiq$ into its rotational and linear components as  $\vPhiq = [ \hat{\mathbf{e}}_{\omega,2}\T, \hat{\mathbf{e} }_{v,2}\T ]\T$ and then post-multiplying \eqref{eq:mot_momentum} with $\vPhiq$ then gives a projected version of the momentum condition as
\begin{align}
0 &= \vPhiq\T \, \dIq \, \vPhiq + \ngear{}\,  \dIzzm \,  \vPhim\T\,  \Xmqq\,  \vPhi_2 \nonumber \\
           &= \vPhiq\T \, \dIq \, \vPhiq + \ngear{} \, \dIzzm \, \hat{\mathbf{z}}\T \, {}^m \mathbf{R}_2(q_2)  \,\hat{\mathbf{e}}_{\omega,2}  \label{eq:mom_simplified}
\end{align}
where the structure of the spatial transform $\Xmqq$ (given in Appendix \ref{app:dynamics}) is used in the simplification to \eqref{eq:mom_simplified}.

We see that \eqref{eq:mot_torque} and \eqref{eq:mom_simplified} can simultaneously be satisfied if and only if 
\[
\ngear{}^2 \, \dIzzm  = \ngear{} \, \dIzzm \, \hat{\mathbf{z}}\T \,{}^m \mathbf{R}_2(q_2) \,\hat{\mathbf{e}}_{\omega,2}
\]
or equivalently that:
\begin{equation}
\ngear{} = \hat{\mathbf{z}}\T \, {}^m \mathbf{R}_2(q_2) \, \hat{\mathbf{e}}_{\omega,2} \label{eq:rotorSimultaenousCondition}
\end{equation}

If the joint is purely translational, \eqref{eq:rotorSimultaenousCondition} cannot hold since $\hat{\mathbf{e}}_{\omega,2} = \mathbf{0}$, and so $\dIzzm$ is identifiable in this case

If the joint has a rotational component (e.g., for a revolute or helical joint), we assume $\| \hat{\mathbf{e}}_{\omega,2}\| = 1$. In this case, \eqref{eq:rotorSimultaenousCondition} holds when:
\begin{enumerate}[nosep]
\item the gear ratio $\ngear{}$ is unity and
\item the rotational component of joint 2 ($\hat{\mathbf{e}}_{\omega,2}$) is parallel to the motor axis (making
  $\vPhim\T\Xmqq\vPhiq$ unity).
\end{enumerate}
Note that this allows only a motor without no reduction ($\ngear{}=1$), but that may
be offset from the joint. In all other cases, the rotor adds one
identifiable parameter $\Izzm$.

\subsub{Generalization to fixed bases}
More generally, the fixed-base case requires considerations of motion
restrictions to determine whether the single additional motor
parameter is identifiable at each joint.
%This can be accomplished by considering spans of jointly attainable
%velocities for both the motor rotors and links akin to
%Lemma~\ref{lemma:reachability}. It likewise requires generalizations
%of Lemma~\ref{lemma:observability} that characterize observability
%for transfers of both motor and link parameters to the predecessor.

The main conceptual difference in the fixed-base case is that the motor inertia is identifiable {\bf only if} it can be felt earlier in the chain. This follows since, when $\vv_1$ has motion restrictions, the momentum condition takes the form:
\[
  \bzero =
  \vPhiq\T \dIq \Xqpq \vv_1 + \vPhim\T \dIm \Xmpq \, \ngear{} \vv_1
\]
where $\vPhim\T \dIm \Xmpq \vv_1=\bzero$ when the rotor inertia is not felt along any of the directions $\vv_1$ can take.
For instance, consider the simple manipulators in Fig.~\ref{fig:simpleManips} and suppose each motor rotates along the corresponding joint axis $\hat{z}_i$. The inertia of the first motor is not felt earlier in the chain for either mechanism (since there are no previous joints), and thus it does not add an identifiable parameter. For the system with parallel joints in Fig.~\ref{fig:simpleManips}(a), the rotational inertia of the second motor about its axis does contribute rotational inertia about the first joint axis, and it is found to add an identifiable parameter. For the system with perpendicular joint axes, the rotational inertia of the second motor rotor does not lead to any rotational inertia about the first joint axis. Thus neither motor inertia contributes an identifiable parameter in this case.

The full details of generalizing the RPNA to consider rotors are described in App.~\ref{app:rotors}. They are also implemented in the companion {\sc Matlab} code where $\vV_i$ and $\vK_i$ are generalized to include effects from both the rotor and link associated with joint $i$.  Returning to the original motivation for the RPNA, these generalizations avoid the algorithm having to consider special cases, such as when joints are/aren't parallel, as was necessary in the above example.

\section{Verification and System-Level Examples}
\label{sec:results}

This section provides verification of the Recursive Parameter Nullspace (RPNA) for fixed- and floating-base systems. The RPNA is unique in that it requires only the kinematic parameters of a mechanism as its input,  it is provably correct, and it does not rely on any symbolic manipulations or assumed exciting input data.  %These features offer advantages over the two main existing approaches to determine the  inertial parameters: symbolic \citep{Khalil90,Khalil14} and numeric SVD or QR decompositions applied to random sampling of the regressor \citep{Gautier91b,Atkeson86}. %Symbolic approaches are often recursive in nature, and rely on manual regrouping that require specialized rules of thumb for different joint types \citep{Khalil90,Khalil14}. In contrast, numeric techniques often sample the inverse dynamics regressor or an energy regressor at random points \citep{Gautier91b} and then employ SVD or QR decompositions to determine determine minimal parameters. These approaches are easily generalizable, but are not provably correct.
 We use numerical identification approaches \citep{Atkeson86} to empirically verify the RPNA output.% for classical industrial manipulators.

\begin{figure}[t]
\center
\includegraphics[height=.7\columnwidth]{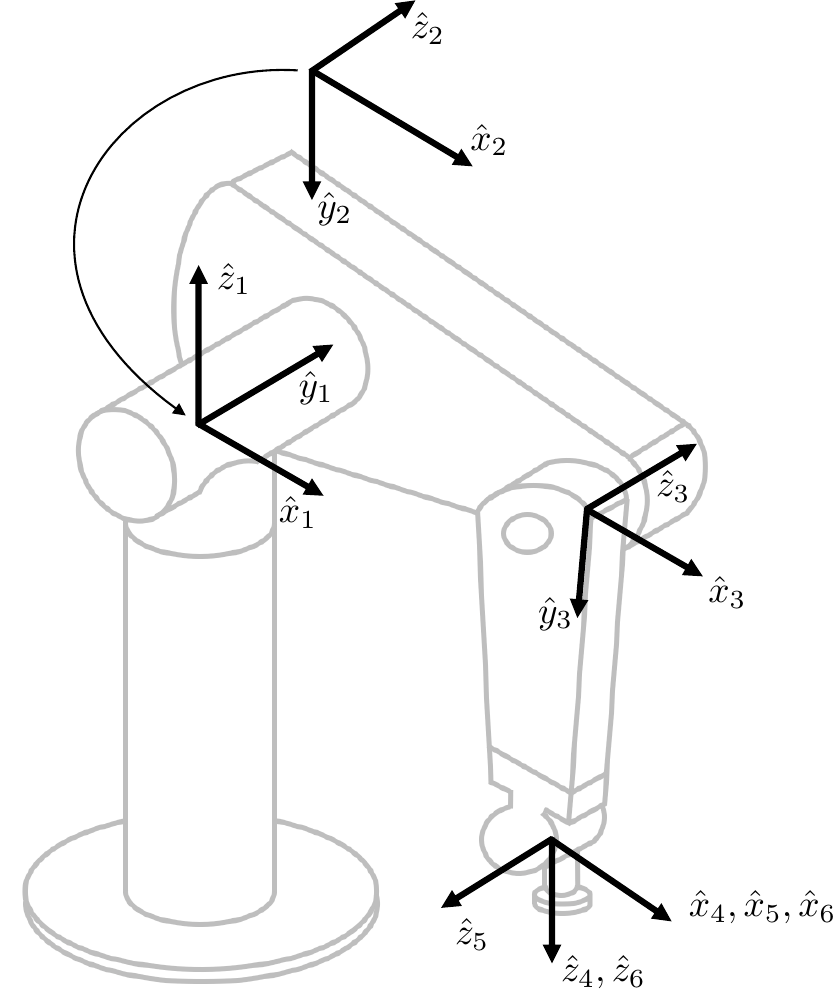}
\caption{PUMA 560 robot and coordinate assignment.}
\label{fig:PUMA}
\end{figure}

\subsection{PUMA 560}

We first consider the classical industrial manipulator PUMA 560 shown in Figure \ref{fig:PUMA}. The mechanism has three joints to position the wrist, followed by three wrist joints with intersecting orthogonal axes.

\newcommand{\cmark}{{\footnotesize\ding{72}}}%
\newcommand{\y}{\cmark}
\renewcommand{\c}{\tikz\draw[black,fill=white] (0,0) circle (.4ex);}
\newcommand{\xmark}{{\tiny\ding{53}}}%
\newcommand{\n}{\xmark}

\newcommand{\cm}{\tikz\draw[black,fill=black] (0,0) circle (.4ex);}

\begin{table}[t]
\center
\begin{tabular}{|c| c| c| c| c| c| c|}
\hline
          & 1 & 2 & 3 & 4 & 5 & 6\\ \hline
 $m$      & \n	& \n  & \c  & \c  & \c  & \c \\
 $m c_x$  & \n	& \cm & \cm & \y  & \y  & \y \\
 $m c_y$  & \n 	& \y  & \cm & \cm & \cm & \y \\
 $m c_z$  & \n 	& \n  & \c  & \c  & \c	& \c \\
 $I_{xx}$ & \n	& \cm & \cm & \cm & \cm	& \cm \\ 
 $I_{yy}$ & \n	& \c  & \c  & \c  & \c	& \c \\ 
 $I_{zz}$ & \cm & \cm & \cm & \cm & \cm & \y \\ 
 $I_{yz}$ & \n	& \y  & \y  & \y  & \y	& \y \\ 
 $I_{xz}$ & \n	& \cm & \y  & \y  & \y	& \y \\
 $I_{xy}$ & \n 	& \y  & \cm & \y  & \y	& \y \\ \hline
 $I_{\mot}$ &\c & \c  & \y  & \y  & \y & \y  \\ \hline
 ${\rm dim}({\mathcal V}_i)$ & 1 (2) & 3 (5) & 6 & 6 & 6 & 6 \\
 ${\rm dim}({\mathcal K}_i)$ & 1 (1) & 6 (8) & 10&10 & 10& 10 \\  
 ${\rm dim}({\mathcal T}_i)$ & 9 & 3 & 3 & 3 & 3 & 3 \\ \hline 
\end{tabular}
\vspace{4px}
\caption{PUMA 560 parameters -- see Table \ref{tab:star_ref} for legend. We find 36/40 base parameters ex-/including motor inertias.  Dimensions in parentheses consider gravitational effects.}
\label{tab:PUMA}
\textbf{}\end{table}

%\caption{PUMA 560 identifiable (\y) and unidentifiable (\n) parameters. A minimal parameter set is indicated (\cm). Unmarked entries are only identifiable in linear combinations with \cm-type parameters. Numbers in parentheses indicate the dimension of the set within the algorithm when considering gravitational effects.}

\begin{table}[t]
\center
\small
\begin{tabular}{c c}
\hline
{\bf Symbol} & {\bf Explanation} \\ \hline
\rowcolor{light-gray} \y   & Identifiable individually\\
\multirow{2}{*}{\cm}  & Identifiable in combination with others\\
     & (Selected as a base parameter) \\
\rowcolor{light-gray} \c   & Identifiable in combination with others \\
\n   & Unidentifiable parameter \\ \hline
\end{tabular}
\caption{Legend for results tables. A minimal set of parameters is denoted by the filled stars and circles. Filled circles are only identifiable in combination with open circles.}
\label{tab:star_ref}
\end{table}

 Table~\ref{tab:PUMA} details the parameter identifiability for this mechanism, with the symbols explained in Table~\ref{tab:star_ref}. We recall that there are three possibilities for each parameter: identifiable by itself, unidentifiable, and identifiable in linear combinations only \citep{Atkeson86}. %An unidentifiable parameter does not affect the measurements at all \rev{(i.e., with the parameter belonging to $\mathcal{N}$).} %\canp{\revg{(contained in $\mathcal{N}$)}}. \revg{[QUESTION: THE FOLLOWING IS COOL, BUT WE DON'T USE THE NOMENCLATURE ANYWHERE ELSE??  COULD WE SIMPLY USE THE "contained in" NOTE??]}  \cang{For body $i$ parameter $k$, this means its corresponding unit vector $\ve_{ik} \in \mathbb{R}^{10 \nbod}$  satisfies $\ve_{ik} \in \mathcal{N}$.} 
 %Relating to inertial transfers, a parameter is unobservable if it can be transferred not just to its predecessor, but all the way to ground. 
 %Likewise, an identifiable parameter is one that is uniquely determined from maximally exciting data \rev{(i.e., with its corresponding unit vector belonging to $\mathcal{N}^\perp$).} \canp{\revg{(contained in $\mathcal{N}^\perp$)}} \cang{, and is characterized mathematically by $\ve_{ik} \in \mathcal{N}^\perp$}. 
 %Again relating to inertial transfers, a parameter is uniquely identifiable if and only if it is not involved in any of the undetectable transfers with its predecessor or successors.
%Remaining parameters (i.e., those not individually identifiable or unidentifiable) are identifiable in linear combinations with other parameters only. For instance, while $I_{xx_6}$ and $I_{yy_6}$ cannot be identified alone, $I_{xx_6}\!-\!I_{yy_6}$ can be \revg{($I_{xx_6}\!-\!I_{yy_6}$ contained in $\mathcal{N}^\perp$, $I_{xx_6}\!+\!I_{yy_6}$ contained in $\mathcal{N}$)}. \rev{PMW: Discuss.}
%consider $I_{xx}$ and $I_{yy}$ for Body~6. Since this last body has no motion restrictions, its parameter transfers coincide with the second body in the free-floating revolute joint example. A transfer of $\delta I_{xx}$ and $I_yy$ satisfying $\delta I_{xx}=\delta I_{yy}$ is a undetectable, thus only the linear combination $I_{xx,6}-I_{yy,6}$ can be identified. 
 %
A minimal set of parameters representing identifiable combinations is indicated with symbols (\cm) in the table. There is some freedom in this assignment since, for example, there is an arbitrary choice as to whether $I_{xx,6}$ or $I_{yy,6}$ is chosen as the base parameter for the identifiable combination $ I_{xx,6}-I_{yy,6}$. To resolve this ambiguity, we always give preference to the parameter that appears first in the table as we move down the columns and right across the rows.
%For instance, choosing $I_{xx,6}$ as a minimal parameter would amount to updating the parameters as $ I_{xx,6}-I_{yy,6} \rightarrow I_{xx,6}$ and $0 \rightarrow I_{yy,6}$. Following this update, the mechanism has equivalent dynamics. The selection of a minimal parameter set is not unique, as the same could be said for using $I_{yy,6}$ as a minimal parameter instead. 
With this designation, the total number of solid entries (\y~or~\cm) in any given column represents the number of 
%identifiable parameter combinations 
base parameters 
contributed by that body. The total number of \n~and \c~symbols for each body indicates the number of degrees of freedom in the inertia transfer with its parent. %the dimension of $\mathcal{N}_{\vY}^\perp$. 
Using the RPNA, the PUMA is found to have 36 identifiable parameter combinations for its bodies. This is consistent with previous symbolic approaches \citep{Mayeda90,Gautier90b} that relied on many special cases. Additional details on the identifiable linear combinations (i.e., full specification of the parameter combinations for the given choice of base parameters) can be obtained from the supplementary {\sc Matlab} code.

Motion restrictions play an important role on the structure of the identifiable parameters for the first two bodies of the PUMA. The true attainable velocity spans have sub-maximal dimensions 1 and 3, and have dimensions 2 and 5 within the algorithm when considering gravity as a fictitious prismatic joint at the base. All remaining bodies have full dimension for $\Vspan{i}$ and $\mathcal{K}_i$. Despite the motion restrictions on body~2, its undetectable transfers coincide with that of an unconstrained body.  The mass $m_2$ is unidentifiable since it is not sensed by joint~$2$, and it is mapped onto the parent parameter $m_1$, which is itself unidentifiable. Likewise, $m c_{z_2}$ is not sensed by torques on joint $2$ and maps to parent parameters $m c_{x_1}$ and $m c_{y_1}$ depending on the value of $q_2$. Both of these parameters of the parent are unidentifiable, and thus so is $m c_{z_2}$. With respect to the motor inertias, the first two joints of the PUMA are perpendicular and thus its first two motor inertias are not identifiable (see the end of Sec.~\ref{sec:closedChain}). %[SHALL WE REFER TO THE MOTOR SECTION?]} \cang{as in the simple system from Section \ref{sec:Simple2}}.

\subsection{SCARA}
The second example considered is a SCARA robot depicted in Fig.~\ref{fig:scara}. The SCARA is a 4-DoF RRPR manipulator traditionally used in pick and place operations. All rotations and translations take place about the local $\hat{z}_i$ axes. Motion restrictions play a key role in parameter identifiability for this robot, as described in Table \ref{tab:SCARA}.

Each of the joints in the SCARA admits more transfer freedoms than in the floating case. The first two links of the SCARA resemble the parallel joint example from Section \ref{sec:Simple1}. As a result, the second revolute joint admits 7 transfer degrees of freedom and contributes 3 identifiable parameter combinations.  Motion restrictions likewise enlarge the undetectable transfers across the prismatic joint. While a free-floating prismatic joint admits 6 transfer degrees of freedom, the  SCARA prismatic joint admits 9 transfer degrees of freedom. 

These extra transfer freedoms for the SCARA prismatic joint can be understood physically from the momentum and \mapping conditions. The momentum condition ${\vPhi_3\T \dI_3 \vV_3=\bzero}$ requires that $\dI_3$ must not modify the linear momentum of body $3$ along $\hat{z}_3$. Motions of joint 3 will create pure linear momentum along $\hat{z}_3$ with magnitude $m_3 \dot{q}_3$, while motions of joints 1 and 2 do not create any linear momentum in this direction. Thus, the momentum condition requires $\delta m_3 =0$.

\begin{figure}[t]
\center
\includegraphics[height=.5\columnwidth]{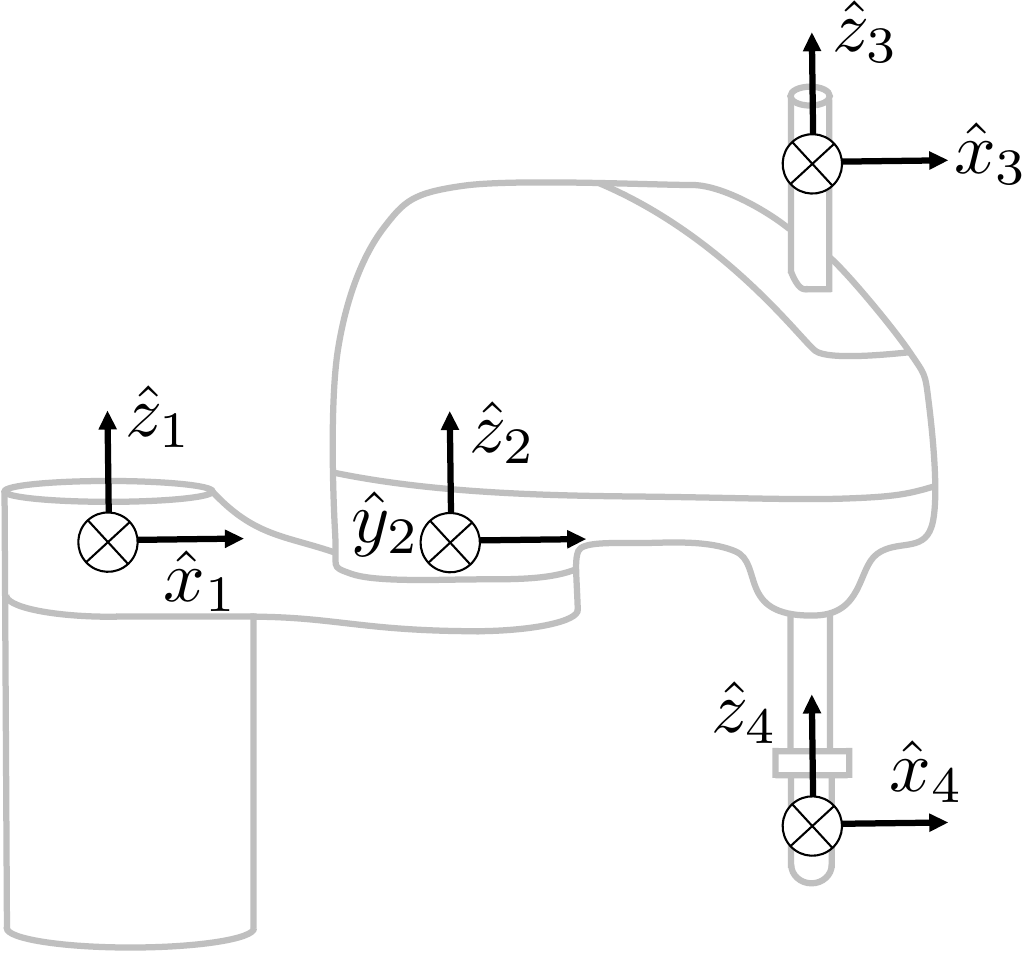}
\caption{SCARA robot and coordinate assignment.}
\label{fig:scara}
\end{figure}

\begin{table}[t]
\center
\begin{tabular}{|c| c| c| c| c| c| c|}
\hline
          & 1(R) & 2(R) & 3(P) & 4(R) \\ \hline %& 5 & 6\\ \hline
 $m$      & \n	& \c  & \cm  & \c \\
 $m c_x$  & \n	& \cm & \c &	\y \\
 $m c_y$  & \n 	& \cm  & \c & \y \\
 $m c_z$  & \n 	& \n  & \n  & \n  \\
 $I_{xx}$ & \n	& \n  & \n  &	\n  \\ 
 $I_{yy}$ & \n	& \n  & \n  & \n  \\ 
 $I_{zz}$ & \cm & \cm & \c &	\y  \\ 
 $I_{yz}$ & \n	& \n  & \n  &	\n  \\ 
 $I_{xz}$ & \n	& \n  & \n  & \n  \\
 $I_{xy}$ & \n 	& \n  & \n  & \n  \\ \hline
 $I_{\mot}$ & \c & \y & \y  & \y \\ \hline
 ${\rm dim}({\mathcal V}_i)$ & 1 (2) & 3 (4) & 4 (4) & 4 (4)  \\[.2ex] 
 ${\rm dim}({\mathcal K}_i)$ & 1 (1) & 4 (4) & 4 (4) & 4 (4)  \\[.2ex]
 ${\rm dim}({\mathcal T}_i)$ & 9 & 7 & 9 & 7  \\ \hline 
\end{tabular}
\vspace{4px}
\caption{SCARA  parameters -- see Table \ref{tab:star_ref} for legend. We find 8/11 base parameters ex-/including motor inertias.  Dimensions in parentheses consider gravitational effects.}
\label{tab:SCARA}
\end{table}

It turns out that the \mapping condition \eqref{eq:Mapping_Parameters} for joint 3 holds without restriction on $\dI_3$. Recall that the \mapping condition considers changes in the way $\dI_3$ maps to parameters of its parent. It holds when any changes in this mapping with $q_3$ appear only on unidentifiable parameters for the parent. Changes in $q_3$ affect the vertical distribution of mass for body $3$ relative to its parent. Yet, any of the parameters affected by the vertical distribution of mass (i.e., $m c_z$, $I_{xx}$, $I_{yy}$, $I_{xz}$, $I_{yz}$) are unidentifiable for body 2. Thus, the \mapping condition holds without any restriction on $\dI_3$. As a result, transfers between body 2 and body 3 need only satisfy the momentum condition $\delta m_3 = 0$, providing 9 transfer freedoms across this joint. Discounting motors, the mechanism has $32$ unobservable parameter combinations and therefore only $8$ identifiable parameter combinations. This was confirmed empirically through an SVD applied to random samples of the regressor $\vY$.

\subsection{Cheetah 3 Leg}

The last example considers the MIT Cheetah 3 robot \citep{bledt2018cheetah} consisting of a torso and four independent legs.  Each leg contains three  rigid bodies, driven by three proprioceptive actuators \citep{Wensing17b}.  It is common to identify the legs separate from the torso \citep{WensingKimSlotine17,Focchi17b}: Fixing the torso as a base, leg swing experiments identify the leg parameters, as depicted in the Figure \ref{fig:Cheetah3}. We explore whether this common setup is appropriate to fully identify the leg, testing two cases: floating vs. fixed torso.  We find that motion restrictions in the fixed experiments prevent exciting all the parameters affecting the floating torso case.

\begin{figure}
\center
\includegraphics[height=.7\columnwidth]{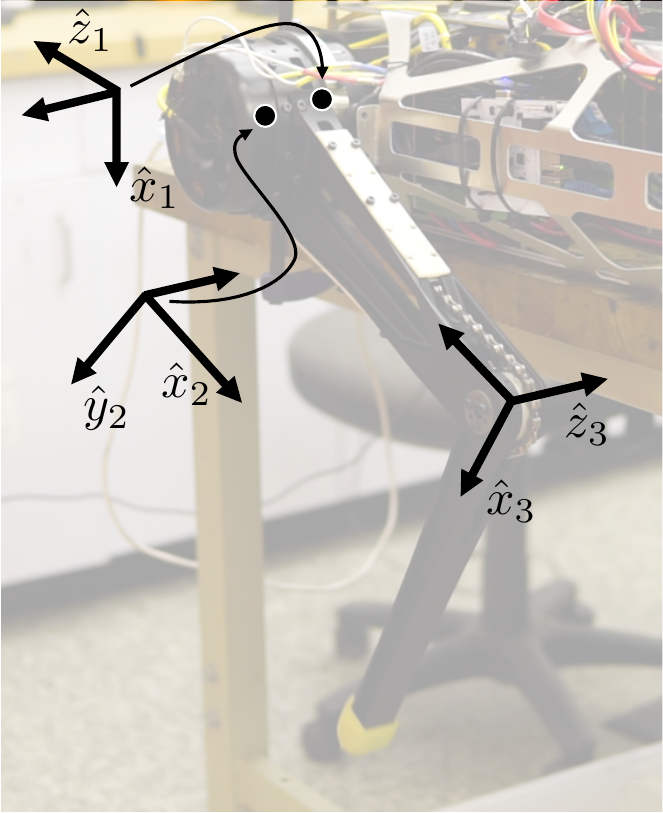}
\caption{Cheetah 3 Leg Coordinates. Table-top experiments like the one shown are often used to identify leg parameters.}
\label{fig:Cheetah3}
\end{figure}

\begin{table}
\center

\begin{tabular}{|c|| c|c|c|| c|c|c|}
\hline
&  \multicolumn{3}{|c||}{Fixed} &  \multicolumn{3}{|c|}{Free} \\ \hline
& 1 & 2 & 3 & 1 & 2 & 3 \\ \hline
 $m$   &    \n	& \c  &   \c   &  \c	& \c  &  \c  \\
 $m c_x$  & \y	& \cm &  \y & \y	& \cm &  \y  \\%& \y	& \cm &  \y  \\
 $m c_y$  & \cm & \y  &  \y & \cm   & \y  &  \y\\%& \cm & \y  &  \y  \\
 $m c_z$  &\n 	& \c  &   \c & \c 	& \c  &  \c \\%& \n 	& \c  &      \\
 $I_{xx}$ &\n	& \cm  &  \cm & \cm	& \cm &  \cm\\%& \n	& \cm  &  \cm\\ 
 $I_{yy}$ &\n	& \c  &  \c    & \c	& \c  &  \c \\%& \n	& \c  &      \\ 
 $I_{zz}$ &\cm  &\cm  &  \y & \cm   &\cm  &  \y \\%& \cm  &\cm  &  \y \\ 
 $I_{yz}$ &\n	& \y  &  \y  & \y	& \y  &  \y\\%& \n	& \y  &  \y  \\ 
 $I_{xz}$ &\n	& \cm  &  \y & \y	& \y  &  \y \\%& \n	& \cm  &  \y \\
 $I_{xy}$ &\n 	& \y  &  \y &  \y	& \y  &  \y \\ \hline %& \n 	& \y  &  \y  \\
 $I_{\mot}$ & \c & \c & \y & \y	& \y  &  \y \\ \hline%& \c & \c & \y   \\ \hline
 ${\rm dim}({\mathcal V}_i)$ &1 (3) & 3  (6) & 6  & 6 & 6 & 6\\%& 1 & 3 & 6   \\[.2ex] 
 ${\rm dim}({\mathcal K}_i)$ &1 (3) & 6 (9) & 10 & 10 & 10 & 10\\%& 1 & 6 & 10   \\[.2ex]
 ${\rm dim}({\mathcal T}_i)$ &7 & 3 & 3 & 3 & 3 & 3 \\ \hline% & 9 & 3 & 3   \\ \hline 
\end{tabular}
\vspace{2px}
\caption{Cheetah 3 Leg parameters -- see Table \ref{tab:star_ref} for legend. We find 17/18 versus 21/24 base parameters in fixed versus floating conditions, ex-/including motor inertias.  Dimensions in parentheses consider gravitational effects.}
\label{tab:Cheetah3Obs}
\end{table}

Table \ref{tab:Cheetah3Obs} compares parameter identifiability in the fixed- vs. floating-base case. Similar to the PUMA and SCARA examples, the Cheetah 3 leg model possesses unidentifiable parameters in the fixed-base case.  As expected, the parameter set is a subset of the floating-base case.
%}  \cang{We will later compare these results to SVD.} 
%The first two joints of the Cheetah are orthogonal, similar to the PUMA. However, unlike the PUMA, the first joint axis is not aligned with gravity in the test configuration. This provides additional identifiable parameters for the first link in comparison to the PUMA. Again, via arguments similar to the PUMA, the rotational inertia of the rotors on the first two joints are only identifiable in combination with the rotational inertia of their associated successor link. Via comparison, in the floating-base case, the addition of coupling moments onto the body allows a disambiguation between theses two effects, as reflected inertia scales with a factor $n_{R}^2$ of the gear ratio $n_R$ on the joint, whereas the associated coupling moments on the body only scale as $n_R$. 

\begin{table*}
\renewcommand{\arraystretch}{1.35}
\center 

\begin{tabular}{|c||c|c c c|}\hline
&  {\bf Fixed-Base Validation} & \multicolumn{3}{c|}{{\bf Floating-Base Validation}} \\
  \bf Identification & Leg Joint Torques & Leg Joint Torques    & Body Torques   & Body Forces \\\hline \hline
  Floating Torso %Floating-Base ID
  & $[0,0,0]$ Nm & $[0,0,0]$ Nm & $[0,0,0]$ Nm & $[0,0,0]$ N  \\ 

  Fixed Torso % Fixed-Base ID
  & $[0,0,0]$ Nm &  $[1.35,    2.39, 0.00]$ Nm                  & $[7.40, 20.24, 4.71]$ Nm  & $[16.31,   12.40,   61.73]$ N   \\\hline
\end{tabular}\vspace{3px}

\caption{RMS validation errors on the galloping dataset using fixed-base vs. floating-base identification. Leg torque vector report ab/ad, hip, and knee residuals. Force and torque residuals are reported in body coordinates with $+x$ forward, $+y$ left, and $+z$ up.}
\label{tab:Validation}
\end{table*}

To analyze the effects of motion restrictions in a concrete situation, consider the scenario of Cheetah~3 executing a transverse gallop. For consistency across cases, motion data is collected in simulation, shown in Figure \ref{fig:CheetahGallop}, and includes the configuration $\vq$, generalized velocity $\vqd \in \mathbb{R}^{18}$, and generalized acceleration $\vqdd \in \mathbb{R}^{18}$.  Motor rotors are modeled as rigid bodies themselves connected to the preceding link via a revolute joint.

In the floating case, we retain the motion as simulated and use inverse dynamics to determine the matching generalized force $\btau_{\rm total} \in \mathbb{R}^{18}$ with effects from active joint torques and ground forces. %Collecting this rich dataset would generally be impractical, as it would require precisely calibrated force plates in strategic locations during galloping. However, measurement of all contact forces would be necessary to fully identify the dynamics.
We estimate all 178 inertial parameters $\hat{\vpi} \in \mathbb{R}^{178}$, stemming from the 13 bodies (10 parameters each) and 12 rotationally symmetric rotors (4 parameters each, as described in Sec.~\ref{sec:closedChain}).  Following \eqref{eq:regressor_main}, the parameters of the system are found by solving a least-squares problem:
\begin{equation}
\min_{\hat{\vpi}} \sum_{j=1}^{N_s} \left \| \btau_{\rm total}^{[j]} - \vY\!\left(\vq^{[j]}, \vqd^{[j]}, \vqdd^{[j]}\right) \hat{\vpi} \right\|^2
\label{eq:least_squares_full}
\end{equation}
where $N_s$ is the number of samples used, the superscript $(\cdot)^{[j]}$ indicates the $j$-th sample of the quantity. Note that the experiment captures both torques for the legs at the joints, as well as associated dynamic coupling forces on the body. %Although the system goes through many contact changes that constrain the robot, the trunk always has full 6D freedom, and thus, swing periods for each leg provide the opportunity for maximal excitation without motion restrictions.

In the fixed case, we lock the torso in place, high above the ground, mimicking the table-top setup in Figure \ref{fig:Cheetah3}. For simplicity, only the front-left (FL) leg is considered. To provide a fair comparison with the floating case, the same swing-leg trajectories are employed, providing configuration $\vq_{FL} \in \mathbb{R}^3$, velocity $\dot{\vq}_{FL}$, and acceleration $ \ddot{\vq}_{FL}$. Again, inverse dynamics determine the
required joint torques $\btau_{FL} \in \mathbb{R}^3$, this time inherently without ground reaction forces.  Equivalent to \eqref{eq:least_squares_full}, we estimate the 42 leg parameters $\hat{\vpi}_{FL} \in \mathbb{R}^{42}$ (3 bodies and rotors).%} \cang{is repeated for the fixed-base identification.} %It is emphasized that this data is synthetic (i.e., not from physical experiments) in order to provide a fair comparison between mock table-top and floating-base identification experiments.  

\begin{figure}[t]
\center
\includegraphics[width=.65 \columnwidth]{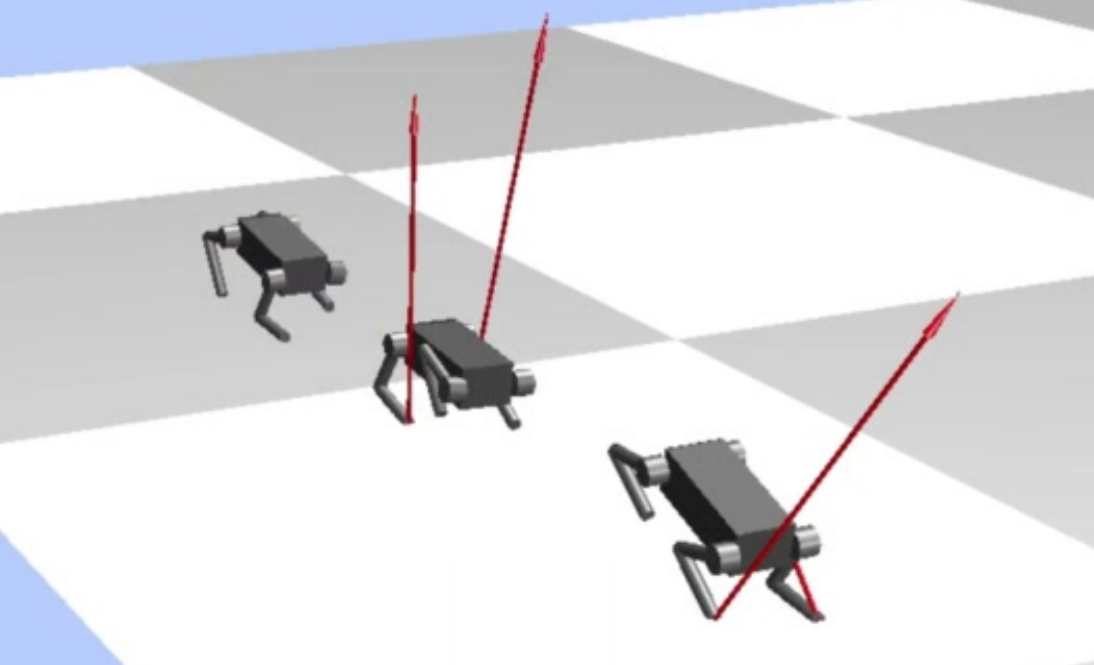}
\caption{Full dynamic simulation of galloping was used to obtain the identification dataset for the MIT Cheetah 3 model.}
\label{fig:CheetahGallop}
\end{figure}

An SVD on the regressor shows that the fixed case includes 18 identifiable combinations (17 from links, and 1 from the knee rotor), while the floating case includes 24 identifiable combinations for each leg (21 from links, and 3 from rotors). The provably-correct output of the RPNA, summarized in Table \ref{tab:Cheetah3Obs}, thus certifies that this motion is maximally exciting for both the fixed- and floating-base cases. Details on the identifiable combinations for both cases are available by running the supplementary MATLAB code.

We computed solutions $\hat{\vpi}$ and  $\hat{\vpi}_{FL}$ to the optimization problems without leveraging prior knowledge by using a pseudo-inverse. %Due to parameter observability considerations, both \eqref{eq:least_squares_full} and \eqref{eq:least_squares_FL} are degenerate least-squares problems with multiple solutions. 
In practice, regularization is often used to include a prior parameter estimate (often from CAD) (e.g., \cite{lee2020geometric}), which also occurs implicitly \citep{boffi2021implicit} in adaptive settings (e.g., \citep{lee2018natural}).

% The accuracy of the estimates are explicitly tied to the accuracy of prior knowledge.

% \revg{QUESTION: DOES THE VALIDATION USE A DIFFERENT MOTION, i.e. DIFFERENT DATA SET?} \rev{Same dataset, different samples.}

Table \ref{tab:Validation} shows the validation error across the two cases. For identification with the floating torso, as expected, validation errors are zero in both the floating-base and fixed-base validation cases. %This is intuitively as expected, since the floating-base dynamics capture a fixed-base constraint as a special setting. 
In contrast, the fixed torso identification only displays favorable generalization when applied to another fixed-base data set. When the identified leg model is used within a full floating-base model, validation errors appear on both the leg torques and the coupling forces/torques.

The parameter indentifiability analysis in Table \ref{tab:Cheetah3Obs} explains the leg torque errors in the floating-base validation. These validation errors occur when motions of the body excite new dynamic effects for the leg that were not captured in the mock table-top experiments. It is observed that these motion restrictions are only on the first two links (ab/ad and hip), while the shank (body 3) can fully excite all its parameters, as signified by ${\rm dim}({\mathcal V}_3)=6$ and ${\rm dim}({\mathcal K}_3)=10$. As a result of this full excitation of the shank parameters in the fixed-base case, the knee joint experiences zero validation errors when generalizing to the free-base case. 

\section{Conclusions}
\label{sec:conclusions}

This paper has introduced the recursive parameter nullspace algorithm (RPNA) to geometrically characterize the identifiability of inertial parameters in a rigid-body system. We have show that unidentifiable parameter combinations have an interpretation as representing a sequence of undetectable inertial transfers across the joints. In arriving at this result, we have transformed the nonlinear parameter identifiability problem for the system as a whole into a sequence of classical linear systems observability problems, proceeding recursively across each joint of the mechanism. 
As a result of these new theoretical advances, the final algorithm is compact (it can be expressed in 10 lines), while generalizing the results of multiple previous authors. 
Extensions have been discussed to handle general multi-DoF joint models, branched kinematic trees, and simple closed loops arising from geared motors. 
The results verify the correctness of the algorithm and illustrate the importance of considering motion restrictions when designing identification strategies for mobile systems. 

%The technical approach here yet has viability to address identifiability in cases of missing or added force/torque measurements on the joints. We hope the supplied source code might help to accelerate these future advances.

%With the algorithm provided, the RPNA provides a new an provably correct method to solve the classical problem of identifiability in robotics.

%Future developments will provide detailed examples of this method applied to a grounded and floating-base systems.

\bibliographystyle{SageH}
%\bibliographystyle{IEEEtran}
% \ifdefined\ARXIV
%     \bibliography{MinimalParameters}
% \else
%    \bibliography{References}
% \fi

%\appendices
\appendix

\section{Rigid-Body Dynamics Details}
\label{app:dynamics}

\noindent{\bf Spatial Velocities:}
The spatial velocity is a 6D velocity that collects traditional 3D rotational and linear velocities. When expressed in its body-fixed coordinate frame the spatial velocity of body $i$ takes the form
\begin{equation}
\vv_i = \begin{bmatrix} {\boldsymbol{\omega}}_i \\ {\mbox{\boldmath $v$}}_i \end{bmatrix}
\end{equation}
where $\boldsymbol{\omega}_i \in \mathbb{R}^3$ is the angular velocity in body coordinates, and $\vel_i \in \mathbb{R}^3$ is the linear velocity of the coordinate origin (given in body coordinates).

\noindent{\bf Spatial Transform:}
The $6 \times 6$ matrix $\XM{i}{\pred{i}}$ is a spatial transformation matrix that converts spatial velocities expressed in frame $\pred{i}$ to frame $i$:
\begin{equation}
\XM{i}{\pred{i}} = \begin{bmatrix} \Rot{i}{\pred{i}} & \bzero \\[1ex] -\Rot{\,i}{\pred{i}}\,\vS\!\left( {}^{\pred{i}}\vp_i \right) & \Rot{i}{\pred{i}} \end{bmatrix}\
\end{equation}
where $\Rot{i}{\pred{i}}\in\mathbb{R}^3$ the rotation matrix from frame $\pred{i}$ to frame $i$, ${}^{\pred{i}}\vp_i$ the vector from the origin of frame $\pred{i}$ to the origin of frame $i$, and $\vS(\vx) \in \mathbb{R}^{3 \times 3}$ is the skew-symmetric 3D cross-product matrix satisfying $\vS(\vx) \vy = \vx \times \vy$ for all $\vx,\vy\in\mathbb{R}^3$. Note that if $\TM{i}{\pred{i}}$ is the homogenous transform between frames $i$ and $\pred{i}$, then  $\XM{i}{\pred{i}}$ is equivalent to the (big) Adjoint matrix ${\rm Ad}_{\TM{i}{\pred{i}}}$.
We note that these transformations satisfy
\begin{equation*}
  \XM{\pred{i}}{i} = \XM{i}{}_{\pred{i}}^{-1} 
\end{equation*}
However, unlike with rotation matrices, $\XMT{\pred{i}}{i} \ne \XM{i}{}_{\pred{i}}^{-1}$.

\noindent{\bf Spatial Cross Product:} The $6 \times 6$ spatial ``cross product'' matrix is given by
\begin{equation}
\begin{bmatrix} {\boldsymbol{\omega}} \\ {\mbox{\boldmath $v$}} \end{bmatrix} \times = \begin{bmatrix} \vS(\boldsymbol{\omega})  & 0 \\ \vS({\mbox{\boldmath $v$}})  & \vS(\boldsymbol{\omega}) \end{bmatrix} \nonumber
\end{equation}
Similar to the standard 3D cross product, the spatial cross product can be used to provide the rate of change in a 6D quantity due to its expression in moving coordinates \citep{Featherstone08}. Note that the $6\times 6$ cross product matrix $(\vv \times)$ is equivalent to the (little) adjoint matrix ${\rm ad}_{\vv}$.

% \noindent{\bf Spatial Inertia:}
% The spatial Inertia  $\vI_i\in \mathbb{R}^{6\time 6}$ for body $i$ is given using body-fixed coordinates by
% \begin{equation}
% \vI_i = \begin{bmatrix} \vIbar_i & m_i \vS(\vc_i) \\ m_i\vS(\vc_i)\T & m_i \bone_3 \end{bmatrix} \nonumber
% \end{equation}
% with $\vc_i \in \mathbb{R}^3$ the vector to the CoM of body $i$ in local coordinates, $m_i\in\mathbb{R}_+$ its mass, and $\vIbar_i \in \mathbb{R}^{3\times3}$ a standard 3D rotational inertia tensor about the coordinate origin
% \begin{equation}
% \vIbar_i = \left[ \begin{smallmatrix} I_{xx} & I_{xy} & I_{xz} \\ I_{xy} & I_{yy} & I_{yz} \\ I_{xz} & I_{yz} & I_{zz} \end{smallmatrix} \right] \nonumber
% \end{equation}

\section{Second Example for Fixed-Based Systems}
\label{app:example2}

For the system in Fig.~\ref{fig:simpleManips}(b), the velocity of body 2 and its span can be given by
\begin{equation}
\vv_2 = \left[\begin{smallmatrix} \ell s_2 \dot{q}_1 \\ \ell c_2 \dot{q}_1 \\ \dot{q_2} \\ 0 \\ 0 \\ -\ell \dot{q}_1 \end{smallmatrix} \right] \quad\quad \vV_2 = \left[\begin{smallmatrix} 1 & 0 & 0 &0\\ 0 &1 &0 &0\\ 0 &0 &1 &0\\ 0 & 0 &0 &0\\ 0 & 0 & 0 &0 \\0 & 0 & 0 &1 \end{smallmatrix}\right]\nonumber
\end{equation}
In comparison to the example system in Fig.~\ref{fig:simpleManips}(a), the span $\vV_2$ has an extra column, representing additional motion freedoms for the second body of this non-planar system. This example also highlights that not all velocities in the span $\Vspan{2}$ are themselves attainable. It is observed that while the last column of $\vV_2$ is in the span of attainable velocities, a pure linear velocity in the $\hat{z}_2$ direction is not possible.

The first body again only has a single identifiable parameter $I_{zz_1}$. Considering a transfer between body~1 and body~2, the momentum condition $\vPhi_{2}\T \, \dI_2\, \vV_2 = \bzero$ enforces
\beq
\delta I_{xz_2} = \delta I_{yz_2} = \delta I_{zz_2} = 0 \label{eq:SimpleExample2Momentum_app}
\eeq
Similar to the previous case, these conditions impose that inertial changes to body $2$ must not create angular momentum about $\hat{z}_2$. However, the change in joint geometries between the examples provides a different set of parameters that are identifiable via the second joint torque.

The \mapping condition:
\[
\vPhi_1\T \,\XMT{2}{1}(0)\, \dI_2^{(k)}  \,\XM{2}{1}(0)\, \vPhi_1 =0\quad \forall k=1,\ldots,10
\]
collectively enforces
\beq
 \delta h_{x_2} = \delta h_{y_2} = \delta I_{xx_2} - \delta I_{yy_2} = \delta I_{xy_2} = 0 \label{eq:SimpleExample2Invariance_app}
\eeq
for $k=1,\ldots,4$. The conditions are redundant for all larger $k$. Again, these conditions ensure that any variations how in $\dI_2$ maps to $\delta I_{zz_1}$ must be zero. Note that the rotational inertia of Body $2$ about $\hat{z}_1$ is
\begin{align}
\vPhi_1\T \,\XMT{2}{1}\, \vI_2\, \XM{2}{1} \, \vPhi_1 &= m_2 \ell^2 + c_2^2 I_{yy_2} + s_2^2 I_{xx_2} \nonumber \\ 
	& ~~~+ 2 c_2 s_2 I_{xy_2} + 2 \ell c_2 h_{x_2} - 2\ell s_2 h_{y_2}  \nonumber
\end{align}
This term staying constant with changes in $q_2$ is equivalent to \eqref{eq:SimpleExample2Invariance_app}, and can be deduced from conditions on its first four derivatives with respect to $q_2$.

%%%%%%%%%%%%%%%%%%%%%%%%%%%%%%%%%%%%%%%%%%%%%%%%%%%%%%%%%%%%%%%%%%%%%%%
\section{Proof of Theorem \ref{thm:main}}
\label{sec:Appendix:varH}

We restate the theorem from the main text:

\begin{customthm}{2}
(Main Result) Consider a serial-chain rigid-body system in the absence of gravity, with the following inertia transfer subspaces for each joint ($i \in \{1,\ldots,\nbod\}$):
\begin{align}
\TsetInertia_i =  \{ \dI_1, \ldots&, \dI_N \in \Iset ~| ~ \exists \dI_0 \in \Iset , \dI_j=\bzero {\rm{~if~}} j\notin\{i,\pred{i}\}, \nonumber\\
\dI_{\pred{i}} &= - \XJT{i}\, \dI_i \, \XJ{i}, \nonumber\\ 
\bzero &= \vPhi_i\T \, \dI_i \,\vV_i ,  \nonumber\\
\bzero &= \trace( \W_{\pred{i},j} \, \XJT{i}\, \dI_i^{(k)}\, \XJ{i} \,)\nonumber \\ 
& ~~~~~~~~~~~~~\forall {k = 1, \ldots, 10}, j={1, \ldots, \Nw{\pred{i}}} \nonumber 
\nonumber
%\label{eq:finalTi}
~\}\nonumber
%\label{eq:finalTi}
~\}
\end{align}
% The structurally unobservable parameter subspace $\mathcal{N}$ satisfies
% \begin{equation}
% \mathcal{N} = \bigoplus_{i=1}^{\nbod} \mathcal{T}_i \nonumber
% \end{equation} 
Then, the set of all structurally unobservable inertia changes is given by $\TsetInertia_1 \oplus \cdots \oplus \TsetInertia_N$, where $\oplus$ denotes a direct sum of vector subspaces.
\end{customthm}

\begin{proof} Consider the set of inertia variations up to body $i$ that don't change the kinetic energy:
\begin{align*}
\scaleobj{.9}{\NsetInertia_i = 
\{ \dI_1, \ldots, \dI_N \in \Iset ~|~ \delta T(\vq,\vqd) = 0~\forall \vq,\vqd , \delta \vI_j=\mathbf{0} {\rm~if~} j>i\}}. %\label{eq:Ni}. 
\end{align*}
We show via induction that for $i=0,\ldots, N$ 
\begin{equation}
\NsetInertia_i=\bigoplus_{j=1}^{i} \TsetInertia_j \label{eq:thingToShow}
\end{equation}

{\em Base case ($i=0$):} \eqref{eq:thingToShow} holds since $\NsetInertia_0$ is the set containing only the zero element.

{\em Inductive step:} We assume \eqref{eq:thingToShow} holds at some $\pred{i}$, and show this implies it holds at $i$. To do so, we consider perturbations  $\dI_1, \ldots, \dI_N$ such that $\dI_j = \bzero$ for all $j>i$. We will show that these perturbations don't change the kinetic energy if and only if they are in the set $\bigoplus_{j=1}^{i} \TsetInertia_j$. Now,  not changing the kinetic energy requires 
$
0=\delta T = \frac{1}{2} \sum_{j=1}^{i} \vv_j\T \dI_j \vv_j\,.
$ 
Then, using \eqref{eq:two_body_diff_kin}, the kinetic energy variation can be expressed as:
\begin{align}
\delta T  =& \frac{1}{2} \sum_{j=1}^{i-2}\left( \vv_j\T \dI_j \vv_j \right) \nonumber\\
			 &~+ \vv_{\pred{i}}\T \XMT{i}{\pred{i}} \dI_i \vPhi_i \qd_i + \frac{1}{2} \vPhi_i\T \dI_i \vPhi_i \qd_i^2\nonumber \\
			 &~ + \frac{1}{2} \vv_{\pred{i}}\T \left[ \dI_{\pred{i}} + \XMT{i}{\pred{i}} \dI_i \XM{i}{\pred{i}} \right] \vv_{\pred{i}} \label{eq:T_transfer}
\end{align}
Consider a linear change of variables for $\dI_{\pred{i}}$:
\begin{equation}
\dI_{\pred{i}} = \dI_{\pred{i}}' - \XJT{i} \dI_i \XJ{i}
\label{eq:ChangeOfVariables}
\end{equation}
which forms $\dI_{\pred{i}}$ via an inertia transfer from the child plus an additional change $\dI_{\pred{i}}'$.
Under this change of variables: %\eqref{eq:T_transfer} %takes the decoupled form:
\begin{align}
\delta T  =& \frac{1}{2} \sum_{j=1}^{i-2}\left( \vv_j\T \dI_j \vv_j \right)+ \frac{1}{2} \vv_{\pred{i}}\T \dI_{\pred{i}}' \vv_{\pred{i}}   \label{eq:Tcond1} \\
			 &~+ \vv_{\pred{i}}\T \XMT{i}{\pred{i}} \dI_i \vPhi_i \qd_i  + \frac{1}{2} \vPhi_i\T \dI_i \vPhi_i \qd_i^2  \label{eq:Tcond2} \\
			 &~ + \frac{1}{2} \vvJi\T  \DI{i}(q_i) \vvJi  \label{eq:Tcond3}
\end{align}
where 
$\scaleobj{.95}{\DI{i}(q_i)= \XMT{i}{\pred{i}}(q_i) \dI_i \XM{i}{\pred{i}}(q_i) - \XMT{i}{\pred{i}}(0) \dI_i \XM{i}{\pred{i}}(0)}$. 

Considering the cases when $q_i=0$ and $\dot{q}_i=0$ shows that $\delta T$ will be identically zero if and only if \eqref{eq:Tcond1}, \eqref{eq:Tcond2}, and \eqref{eq:Tcond3} are all individually identically zero. 

The terms in \eqref{eq:Tcond1} being zero is equivalent to $\dI_1,\ldots,\dI_{i-2}$, $\dI_{\pred{i}}'$ giving $\delta T=0$, which is the same as the perturbations being in $\NsetInertia_{\pred{i}}$.

The terms from \eqref{eq:Tcond2} being zero for all $\vq,\vqd$ is equivalent to 
%\begin{equation}
$\vv_i\T \dI_i \vPhi_i = 0$  holding for all $\vq, \vqd$, 
%\label{eq:AttainableCondition}
%\end{equation}
which, via the argument in the main body, is the same as $\vPhi_i\T \,\delta \vI_i \, \vV_i=0$.

Finally, the terms from \eqref{eq:Tcond3} being zero is equivalent to
\begin{align}
0  &= \frac{\partial}{\partial q_i} \vvJi\T\, [\XMT{i}{\pred{i}}(q_i) \dI_i \XM{i}{\pred{i}}(q_i)]\,\vvJi \quad \forall \vq, \vqd
%\\&= {\rm Tr}( \vvJi \vvJi\T\, {}^{J_i}\DI{i}(q_i) )  
\label{eq:decoupling_condition_arranged}
\end{align}
which, via the argument in the main body, is the same as $0=\trace(\vW_{\pred{i},j} \,\XJT{i}\, \dI_i^{(k)}\, \XJ{i}\,)$ for $k=1,\ldots,10$, $j=1,\ldots,\Nw{\pred{i}}$.

Thus, we've shown that a variation to the inertias $(\dI_1, \ldots, \dI_N)  \in \NsetInertia_i$ iff it can be written as a parameter variation in $\NsetInertia_{\pred{i}}$ added with one in $\TsetInertia_i$. This proves the inductive step.

Considering \eqref{eq:thingToShow} when $i=N$ proves the theorem. \qed
\end{proof}

%\section{Selected Proofs}

\section{Proof of Lemma \ref{corollary:velocityProp}}
\label{sec:app:proof_lemma1}

To prove the lemma in the main text, we begin with a proposition to compute the span of velocities that can be reached after a joint given a span of velocities before it.
\begin{proposition}
\label{prop:ctrb}
Consider a spatial transform as a function of a single angle $q$, denoted $\XM{}{}(q)$. Suppose  $\XM{}{}(0) = \bone$ and that
\begin{equation}
\frac{\partial }{\partial q} \XM{}{}{}(q) = -(\vPhi \times) \XM{}{}(q) \label{eq:XacrossJoint}
\end{equation}
for some $\vPhi\in \mathbb{R}^{6\times1}$. Then, for any $\vV \in\mathbb{R}^{6\times k}$ 
\begin{align}
&{\rm span}\{   \vv ~|~ \exists q \in \mathbb{R}, \vv \in {\rm Range}( \XM{}{}(q)\, \vV) \} \\\nonumber
&~~~~={\rm Range}(\, {\rm{Ctrb}}(\, (\vPhi \times),\, \vV\,)\, ) \label{eq:controllability} \nonumber
\end{align}
where ${\rm{Ctrb}}(\, (\vPhi \times),\, \vV\,)$ is the controllability matrix associated with the pair $(\, (\vPhi \times),\, \vV\,)$ \cite{Rugh96}.  
\end{proposition}

\begin{proof}[Proof of Proposition \ref{prop:ctrb}]
We define
\[
\mathcal{S}(\vPhi, \vV) = {\rm span}\{ \XM{}{}(q) \vv ~|~q\in\mathbb{R},~\vv\in {\rm Range}(\vV) \} 
\]
and recall, from the Lemma statement, that 
\[
\frac{\rm d}{{\rm d} q} \XM{}{}(q) = - (\vPhi \times) \XM{}{}(q) 
\]
From the definition of the matrix exponential for a linear system \citep{Rugh96}:
\[
 \XM{}{}(q) \vV =  {\rm e}^{\,-q(\vPhi \times)} \vV
\]
The Cayley-Hamilton theorem then ensures that 
\[
\mathcal{S}(\vPhi, \vV)\subseteq  {\rm Range}\left([\vV,(\vPhi \times)\vV,\ldots,(\vPhi \times)^5\vV]\right)
\] 
and thus
\[
\mathcal{S}(\vPhi, \vV)\subseteq  \rm{Range}(\, {\rm{Ctrb}}( (\vPhi \times), \vV)\,)
\]
Note, the range of the controllability matrix provides the smallest
%{\rm{Range}}([\vV,\, (\vPhi \times) \vV,\, \ldots,\, (\vPhi \times)^5 \vV ])$. Recognizing ${\rm{Range}}([\vV,\, (\vPhi \times) \vV,\, \ldots,\, (\vPhi \times)^5 \vV ])$ as the controllability subspace associated with the pair $( (\vPhi\times), \vV)$~\citep{Rugh96} proves that it is the smallest 
$(\vPhi \times)$-invariant subspace containing ${\rm Range}(\mathcal{\vV})$. 
Yet, $\mathcal{S}(\vPhi,\vV)$ is invariant under $(\vPhi\times)$ and contains ${\rm Range}(\mathcal{\vV})$. This proves the reverse containment.\qed 
\end{proof}

\begin{proof}[Proof of Lemma \ref{corollary:velocityProp}]
The propagation of the attainable velocity span:
\[
\mathbf{V}_i = \left [\, {\rm Ctrb}\!\left(\, ({\boldsymbol \Phi}_i \times) ,\XJ{i}\,\mathbf{V}_{\pred{i}} \,\right)~ \boldsymbol{\Phi}_i \, \right] 
\]
follows from Proposition \ref{prop:ctrb} and Eqs.~\eqref{eq:two_body_diff_kin} and \eqref{eq:change_of_X_main}.\qed

\end{proof}

\section{Proof of Lemma \ref{lemma:obs}}
\label{sec:Proof_Observability}

Before proving \eqref{lemma:obs}, we introduce a hat operator $[\cdot]^\wedge$ that maps inertial parameters to a 6D inertia matrix:
\[
\left[\vpi\right]^\wedge = \vI(\vpi)
\]
in order to streamline the equations.

The following proposition is key to proving Lemma 2.
\begin{proposition}
\label{obs:prop}
Consider a spatial transform as a function of a single angle $q$, denoted $\XM{}{}(q)$. Suppose  $\XM{}{}(0) = \bone$ and that there exists $\vPhi\in \mathbb{R}^{6\times1}$ such that
\begin{equation}
\frac{\partial }{\partial q} \XM{}{}{}(q) = -(\vPhi \times) \XM{}{}(q) \nonumber
%\label{eq:ChangeInX}
\end{equation}
Then, for any $\mathbf{C} \in \mathbb{R}^{k\times10}$, the following holds
\begin{align}
&\{\vpi \in \mathbb{R}^{10} ~|~  \mathbf{C} \left[ \,\XMT{}{}\!(q)\, [\vpi]^\wedge \, \XM{}{}(q) \right]^{\vee} = \mathbf{0},~ \forall q \in \mathbb{R}\} \nonumber\\
&~~~={\rm Null}\left(\, { \rm Obs}\left(\, \mathbf{C} ,\,  \vA( \vPhi)\,  \right) \,\right) \nonumber
\end{align}
where ${ \rm Obs}\left(\, \mathbf{C},\,  \vA( \vPhi)\,  \right)$ is the observability matrix \citep{Rugh96} associated with the pair $\left(\, \mathbf{C},\,  \vA( \vPhi)\, \right)$.
%\begin{align}
%\nonumber&\mathcal{I}( \vPhi, \vK) = 
%\end{align}
\end{proposition}
%Either set in the equality characterizes the inertial parameters that maintain a zero output with respect to $\vK$  following transformation across a joint with free modes $\vPhi$.

\begin{proof}[Proof of Proposition \ref{obs:prop}]
Let $\vpi_0 \in \mathbb{R}^{10}$ and denote
\[ 
\vpi(q) = [\, \XMT{}{}\!(q)\, [\vpi_0]^\wedge \, \XM{}{}(q) \,]^\vee
\]
%Similar to the main, text, we 

Using \eqref{eq:change_of_X_main}, \eqref{eq:A_InertialParams}, and the fact that $\XM{}{}(q)$ and $(\vPhi\times)$ commute:
%as the inertial parameters after transfer across a joint. 
%Due to the linearity of $[\cdot]^\vee$, it follows that 
%\begin{equation}
%\frac{ {\rm d}} { {\rm d} q} \vpi(q) = -[\, \XMT{}{} \left( (\vPhi\times)\T [\vpi_0]^\wedge + [\vpi_0]^\wedge (\vPhi\times) \right) \, \XM{}{} \,]^\vee
%\label{eq:ChangeInPiq}
%\end{equation}
%Since $\XM{}{}(q) = {\rm e}^{-(\vPhi \times)q}$, it follows that the spatial cross product $(\vPhi\times)$ commutes with $\XM{}{}(q)$.
% \begin{equation}
% (\vPhi\times)\XM{}{}(q) = \XM{}{}(q)(\vPhi \times)
% \label{eq:PhiCommute}
%  \end{equation}
%Thus, \eqref{eq:PhiCommute} provides
\begin{align}
\frac{\rm d}{{\rm d}q} \vpi(q) &= -\left[\, (\vPhi\times)\T\, [\vpi(q)]^\wedge +  [\vpi(q)]^\wedge\, (\vPhi\times) \,\right]^\vee \nonumber\\ 
 &= -\vA(\vPhi)\, \vpi(q) \nonumber
 \end{align}
%and as a result, 
%\[
%\vpi(q) = {\rm e}^{-\vA(\vPhi) q}\, \vpi_0
%\]
Linear systems observability results \citep{Rugh96} then guarantee that the following are equivalent
\begin{align}
\scaleobj{.96}{\vpi_0  \in {\rm Null}\left(\, { \rm Obs}\left(\, \mathbf{C},\,  \vA( \vPhi)\,  \right) \,\right) \,\iff \,\vpi(q) \in {\rm Null}(\mathbf{C}) ~~~\forall q} \nonumber 
\end{align}
\qed
\end{proof}

\begin{proof}[Proof of Lemma \ref{lemma:obs}]
We aim to construct $\vK_i$ so that ${\rm Range}(\vK_i) = \mathcal{K}_i = \spn(\{\nabla_{\vpi} \vv\T \vI(\vpi) \vv \,|\, \vv\in\mathcal{V}_i^*\})$. Equivalently, we construct $\vK_i\T$ so that ${\rm Null}(\vK_i\T) = \mathcal{K}_i^\perp$.

Suppose $\delta \vpi_i \in \mathcal{K}_i^\perp$. That is, we consider a corresponding $\dI_i$ such that $
\vv_i\T \dI_i \vv_i=0 ~~\forall \vq,\vqd$. 
This is equivalent to
\[
\left[\XM{i}{\pred{i}}\vv_{\pred{i}} + \vPhi_i \qd_i\right] \T \dI_i \left[ \XM{i}{\pred{i}} \vv_{\pred{i}} + \vPhi_i \qd_i  \right]=0
\]
for all $\vq,\vqd$. Expanding terms, this implies
\begin{align}
0&=\vv_{\pred{i}}\T \XMT{i}{\pred{i}}(q_i) \dI_i \XM{i}{\pred{i}}(q_i) \vv_{\pred{i}}  \quad \forall \vq,\vqd \label{eq:app_expand1}\\
0&=\vPhi_{i}\T \dI_i \vv_i \quad \forall \vq,\vqd \label{eq:app_expand2}
\end{align}
We take the transformation $\XM{i}{\pred{i}}(q_i)$ from frame $\pred{i}$ to $i$, and instead decompose it using a fixed transformation to a frame $J_i$ right before the joint so that $\XM{J_i}{\pred{i}} = \XM{i}{\pred{i}}(0)$. With this definition, we express $\XM{i}{\pred{i}}(q_i) = \XM{i}{J_i}(q_i)\, \XM{J_i}{\pred{i}}$. The first condition \eqref{eq:app_expand1} above is then equivalent to
\[
\bzero = \vK_{\pred{i}}\T \Bi\left[ \XMT{i}{J_i}(q_i) \dI_i \XM{i}{J_i}(q_i) \right]^{\vee} \quad \forall q_i
\]
and the second \eqref{eq:app_expand2} equivalent to 
\[
\MomentumCondition{i} = \mathbf{0}
\]
where $\vC(\vV_{i}, \vPhi_i)$ is given by \eqref{eq:Cmom}. Using Proposition \ref{obs:prop}, it follows that $\vK_i\T$ can be selected as
\[
\vK_i\T = \begin{bmatrix} {\rm Obs}(\,\vK_{\pred{i}}\T \Bi ,\, \vA(\vPhi_i) \,)  \\[1ex]
 						\MomentumRows{i}  \end{bmatrix}
\] \qed
\end{proof}

\section{Identifiability from Gravity}
\label{sec:Conditions_g}

Similar to the nullspace for the kinetic energy, parameter variations ensuring $\delta \vg \equiv \mathbf{0}$ can be formed via sequences of inertia transfers. The variation $\delta \vg$ to the generalized gravitational force is equal to zero if and only if it does not affect the rate of change in potential energy $\dot{V}$:
\[
0 = \delta \dot{V} = \vqd\T (\delta \vg) = - \sum_{j=1}^{\nbod} \vv_j\T \, \dI_j\, \XM{j}{0}  \, {}^0 \va_g 
\]
where ${}^0 \va_g $ is the spatial acceleration due to gravity. Each entry of the sum characterizes a change in the instantaneous power of the gravitational force on Body $j$. We again assume that body $i$ is the largest body with $\dI_i \ne \bzero$, and follow a similar approach to the kinetic energy nullspace analysis. Following an equivalent derivation to Appendix \ref{sec:Appendix:varH}
%We define $\dI_{\pred{i}}'$ such that
%\beq
%\dI_{\pred{i}}  = \dI_{\pred{i}}' - \XJT{i} \, \dI_i \, \XM{J_i}{p_i} \label{eq:gravCOV}
%\eeq
it can be shown that $\delta \vg = \bzero$ iff
\begin{align}
\bzero & = \vPhi_i\T \,\dI_i\, \XM{i}{0}\,  {}^0 \va_g  \label{eq:vgcond1}\\
0 & = \vvJi\T \,\DI{i}(q_i) \, \XM{\pred{i}}{0}\, {}^0 \va_g   \label{eq:vgcond2} 
\end{align}
for all $\vq,\vqd$, and subsequent changes $\dI_1, \ldots, \dI_{i-2},\dI_{\pred{i}}'$ independently satisfy $\delta \vg=\bzero$. Similar to before, the substitution introducing $\dI_{\pred{i}}'$ via \eqref{eq:ChangeOfVariables} decouples considerations of transfers across joint~$i$ from transfers earlier in the chain.

Condition \eqref{eq:vgcond1} motivates the attainable gravity vector span
\[
\mathcal{G}_i = {\rm span}\left\{ \XM{i}{0}(\vq) \,{}^0 \va_g ~|~ \vq \in \mathbb{R}^{\nbod} \right\}
\]
Analogous to Lemma \ref{corollary:velocityProp}, we seed $\vG_0 = {}^0 \va_g$, let $\wellBeforeJoint{\vG}{i} = \XJ{i}\, \vG_{\pred{i}}$, and recursively apply
\[
\vG_{i} = {\rm Ctrb}(\, (\vPhi_i \times), \wellBeforeJoint{\vG}{i}  \, )
\]
which ensures each ${\rm Range}(\vG_i) = \mathcal{G}_i$.  Intuitively, changes satisfying $\vPhi_i\T \dI_i \vG_i = \bzero$ cannot be detected via the preceding joint torque in static cases.
%With this propagation, \eqref{eq:vgcond1} is equivalent to requiring
%\begin{equation}
%\vA_i^T \dI_i \vPhi_i = \bzero \nonumber
%\end{equation}

% Examining \eqref{eq:vgcond2} in further detail, it holds iff
% \begin{align}
% 0 & =  \vvJi\T \,\DI{i}(q_i) \, {}^{J_i} \va_g  \nonumber  
% %\label{eq:VelocityAccelerationCondition}
% \end{align}
% for any attainable $\vvJi$ and ${}^{J_i}\va_g = \XM{J_i}{0} {}^0 \va_g$.

The second condition \eqref{eq:vgcond2} can be addressed by generalizing the propagation of $\vK_i$ from Lemma~\ref{lemma:obs} to include gravitational effects. This extension can be accomplished by including the new parameters that are identified via torques on each joint:
\[
\vK_i\T = \begin{bmatrix} {\rm Obs}(\,\vK_{\pred{i}}\T \Bi ,\, \vA(\vPhi_i) \,)  \\[1ex]
 						\vC([\vV_{i}~ \vG_i], \vPhi_i)\T  \end{bmatrix}
\]
Comparing the propagation of $\vG_i$ and $\vV_i$
\begin{align}
\vV_0 &= \bzero			&  \mathbf{V}_i &= \left [\, {\rm Ctrb}\!\left(\, ({\boldsymbol \Phi}_i \times) ,\vVJi \,\right)~ \boldsymbol{\Phi}_i \, \right] \\
\vG_0 &= {}^0 \va_g 	&  \vG_i &= {\rm Ctrb}(\, (\vPhi_i \times), \wellBeforeJoint{\vG}{i}  \, )
\end{align}
a union of these bases $\tilde{\vV}_i = [\vV_i ~ \vG_i]$ can thus be propagated together in one operation via:
\begin{align}
\tilde{\vV}_0 &= {}^0 \va_g 			&  \tilde{\mathbf{V}}_i &= \left [\, {\rm Ctrb}\!\left(\, ({\boldsymbol \Phi}_i \times) , \XJ{i} \tilde{\vV}_{\pred{i}}  \,\right)~ \boldsymbol{\Phi}_i \, \right] \nonumber
\end{align}
which simply represents a change in seed for $\vV_0$. Again, this change in seed can be interpreted as adding a prismatic joint at the base that is aligned with gravity, but whose joint force is not measured.

\section{Computing the System Parameter Nullspace}
\label{subsec:ParameterNullspace} 
For each body, consider $\vR_i$ as any full rank matrix such that ${\rm Range}(\vR_i) = {\rm Null}( \vN_i)$. With this local nullspace basis, we construct a block upper-bidiagonal matrix $\vR$ such that
\begin{align}
\vR_{i,i} &= \vR_i \nonumber\\
\vR_{\pred{i}, i} &= -\Bi\, \vR_i \textrm{~~when $\pred{i}>0$} \nonumber
\end{align}
and $\vR_{i,j} = \bzero$ otherwise. Following this construction ${\rm Range}(\vR) = \mathcal{N}$. Similarly, we can use the local nullspace descriptors $\vN_i$ to determine a basis for $\mathcal{N}^\perp$. A system nullspace descriptor $\vN$ is constructed as a block upper-triangular  matrix satisfying
\begin{align}
\vN_{i,i} &= \vN_i  \nonumber \\
\vN_{i, j}&= \vN_{i, \pred{j}} \Bj 	 \quad \forall j < i \nonumber
\end{align}
and ${\vN}_{i,j} = \bzero$ otherwise. 
Following this construction, ${\rm Null}(\vN) = \mathcal{N}$ 
and thus, ${\rm Range}(\vN\T) = \mathcal{N}^\perp$. 
The significance of this property is that each row of $\vN$ describes a linear combination of parameters that can be identified. 
The row-reduced-echelon form of $\vN$ allows identifiable parameters (individually or through regroupings) to be plainly discerned.

\section{Aside: Revisiting \cite{Ayusawa14}}
\label{sec:compareToKo}
\citet{Ayusawa14} provided a powerful result that is summarized as follows. For a floating-base open-chain system, the inertial parameters that are identifiable through measurement of joint torques and external forces are the same as with external forces alone. This result has immediate application to identifying position-controlled robots without torque sensing. To relate these previous results with the approaches taken here, consider a partitioning of the equations of motion \eqref{eq:eom}:
\begin{equation}
\begin{bmatrix} \vH_1(\vq) \\ \vH_*(\vq) \end{bmatrix} \vqdd + \begin{bmatrix} \vc_1(\vq,\vqd) + \vg_1(\vq) \\ \vc_*(\vq,\vqd) + \vg_*(\vq) \end{bmatrix}  = \begin{bmatrix} \bzero \\ \boldsymbol{\tau}_* \end{bmatrix} + \begin{bmatrix} \vJ_1(\vq)\T \vf \\ \vJ_*\T \vf \end{bmatrix}
\label{eq:block_eom}
\end{equation}
where $\vH_1$, $\vc_1$, and $\vg_1$ give the top six rows of $\vH$, $\vc$, and $\vg$ corresponding to the floating base, and $\vJ_1$ gives the left six columns of $\vJ$ likewise corresponding to the floating base. The $*$ entries represent analogous entries for the joints, but won't be needed elsewhere in our development. The main result from \citet{Ayusawa14} is re-phased below with an alternate proof using a few short physics arguments. 

\begin{figure}
    \centering
    \includegraphics[width=\columnwidth]{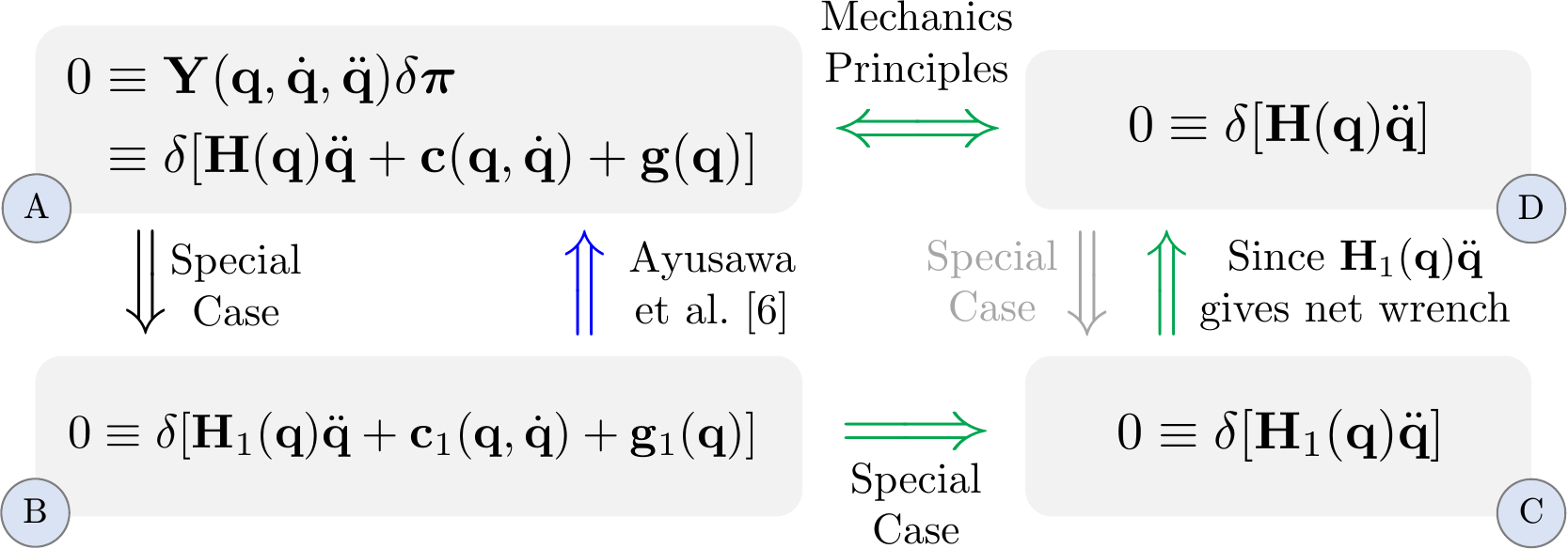}
    \caption{An alternate proof of the main result from \citet{Ayusawa14} uses mechanics principles. The implication from the blue arrow is replaced by those in green.}
    \label{fig:Ayusawa}
\end{figure}

%The floating-base case ends up being covered by our the fixed-base theory with multi-DoF joints, so we focus description on intuition here. Parameter changes do not affect the kinetic energy if and only if they do not affect the system mass matrix $\vH$. 
% For a floating-base system, consider the liberty of partitioning the generalized velocity as $[ \bv_1\T ~ \vqd_{J}\T]\T$ where $\vqd_{J}$ gives the vector of joint velocities. The upper triangle of the symmetric mass matrix can then be given by \cite{Featherstone08}
% \begin{equation}
% \vH = \begin{bmatrix}   \vI_1^C & \XMT{2}{1}\, \vI_2^C\, \vPhi_2 & \cdots  & \XMT{\nbod}{1}\,\vI_{\nbod}^C\, \vPhi_{\nbod} \\[.7ex]
% 						 \cdot & \vPhi_2\T\, \vI_2^C \,\vPhi_2 & \cdots & \vPhi_2\T\, \XMT{{\nbod}}{2}\, \vI_{\nbod}^C\, \vPhi_{\nbod} \\[.7ex]
% 						 \cdot  & \cdot     & \ddots & \vdots \\[.7ex]
% 						 \cdot & \cdot &  \cdot & \vPhi_{\nbod}\T \, \vI_{\nbod}^C \, \vPhi_N \end{bmatrix} 
% \label{eq:fb_h}
% \end{equation}
% where $\vI_i^C$ is the composite rigid-body inertia \citep{FeatherstoneOrin08} of the subchain of bodies rooted at body~$i$. %The transfer assignment $\dI_{\pred{i}} = - \XMT{i}{\pred{i}}(0) \,\dI_{i}\, \XM{i}{\pred{i}}(0)$ 
%will leave the composite inertia $\vI^C_{\pred{i}}$ unchanged when it satisfies the mapping condition. Since only these composite effects appear within earlier rows of $\vH$ in \eqref{eq:fb_h}, any additional undetectable changes to inertias earlier in the chain $\dI_1,\ldots,\dI_{\pred{i}}$ can be determined independently.  The following theorem states this observation formally.

\begin{theorem}[Main result of \cite{Ayusawa14}]
Consider a floating-base open-chain rigid-body system with dynamics \eqref{eq:block_eom}. 
The parameter change $\dpi$ doesn't affect the dynamics if and only if it does not affect the first block rows of \eqref{eq:block_eom}, i.e.,:
\begin{equation}
\delta[ \vH_1(\vq) \vqdd + \vc_1(\vq,\vqd) + \vg_1(\vq) ] = \vzero~~~\forall \vq,\vqd,\vqdd
\label{eq:ko_complex}
\end{equation}
\end{theorem}

\newcommand{\cA}{\rm{A}}
\newcommand{\cB}{\rm{B}}
\newcommand{\cC}{\rm{C}}
\newcommand{\cD}{\rm{D}}

\begin{proof}
($\Rightarrow$) The forward implication ($\cA\Rightarrow \cB$ in Fig.~\ref{fig:Ayusawa}) is immediate since the conditions in \eqref{eq:ko_complex} are a subset of those required for $\dpi \in \mathcal{N}$. 

($\Leftarrow$) Toward proving the reverse implication ($\cB\Rightarrow \cA$ in Fig.~\ref{fig:Ayusawa}), we will instead prove that $\cB\Rightarrow \cC$, $\cC\Leftrightarrow \cD$, and $\cD\Leftrightarrow \cA$.

The implication $\cB\Rightarrow \cC$ is immediate since $\cC$ is a special case of $\cB$.

We now wish to show that $\cC\Leftrightarrow \cD$. The direction $\cC\Leftarrow \cD$ is again immediate since $\cC$ is a special case of $\cD$. To show the reverse direction, we will make use of the form of the mass matrix \citep{Featherstone08}
\begin{equation}
\scaleobj{.9}{\vH = \begin{bmatrix}   \vI_1^C & \XMT{2}{1}\, \vI_2^C\, \vPhi_2 & \cdots  & \XMT{\nbod}{1}\,\vI_{\nbod}^C\, \vPhi_{\nbod} \\[.7ex]
						 \cdot & \vPhi_2\T\, \vI_2^C \,\vPhi_2 & \cdots & \vPhi_2\T\, \XMT{{\nbod}}{2}\, \vI_{\nbod}^C\, \vPhi_{\nbod} \\[.7ex]
						 \cdot  & \cdot     & \ddots & \vdots \\[.7ex]
						 \cdot & \cdot &  \cdot & \vPhi_{\nbod}\T \, \vI_{\nbod}^C \, \vPhi_N \end{bmatrix}} 
\label{eq:fb_h}
\end{equation}
where $\vI_i^C$ is the composite rigid-body inertia \citep{FeatherstoneOrin08} of the subchain of bodies rooted at body~$i$. Indeed, one may be able to see that if a variation in inertias does not change the first block row of $\vH$, then it will not change $\vH$ overall. We proceed to justify this point alternatively via physical arguments. 

Let us consider the case when only joint $j$ is moving. In that case,  we have from \eqref{eq:fb_h} that $\vH_1\vqdd = \XMT{j}{1} \vI_j^C \vPhi_j \qdd_j$, which represents the total wrench required to move all bodies in the system. By comparison, consider the joint torque at any joint $i$ earlier in the chain ($i < j$), given by 
\[
\tau_i = \vH_{ij} \qdd_j = \vPhi_i\T \XMT{j}{i} \vI_j^C \XMT{j}{1} \vPhi_j \qdd_j = (\XM{1}{i} \,\vPhi_i )\T  \vH_1\vqdd
\]
In this case, $\tau_i$ is simply the projection of the net wrench onto the $i$-th joint. As such, the $i$-th joint torque carries strictly less information than $\vH_1 \vqdd$. By this logic, if 
\begin{equation}
\delta[ \vH_1(\vq) \vqdd ] = \vzero~~~\forall \vq,\vqdd
\label{eq:ko_simple}
\end{equation}
then the upper triangle of $\delta \vH(\vq)$ will be zero for all $\vq$. Via symmetry of $\vH$, this implies that $\delta \vH(\vq)=\mathbf{0}$ for all $\vq$. This reasoning shows that $\cC$ implies $\cD$.\footnote{Note that this reasoning for $\cC \Rightarrow \cD$ no longer holds for systems that have closed kinematics loops, as is observed in Sec.~\ref{sec:closedChain}.}

We finally argue the equivalence of $\cA$ and $\cD$ in the figure.
Per Remark~\ref{rem:gravity}, we can ignore gravity, so $\dpi$ leaves the dynamics unchanged if and only if:
\begin{equation}
\delta[ \vH(\vq) \vqdd + \vc(\vq,\vqd) ] = \vzero~~~\forall \vq,\vqd,\vqdd \label{eq:noEnergyVariation}
\end{equation}
Since the Coriolis terms are determined uniquely from the form of the mass matrix (e.g., via Christoffel symbols \cite{siciliano2008book,echeandia2021numerical}), \eqref{eq:noEnergyVariation} is equivalent to requiring that
\begin{equation}
\delta[ \vH(\vq) \vqdd ] = \vzero~~~\forall \vq,\vqdd
\label{eq:no_gravity_condition}
\end{equation}
This proves the equivalence of $\cA$ and $\cD$ in the figure, which completes the proof. %\footnote{The condition \eqref{eq:no_gravity_condition} is also equivalent to enforcing that $\delta \vH(\vq) \equiv 0$ as was considered in  Sec.~\ref{sec:FloatingBaseChains}.}
%Zooming out, we have argued that $\dpi$ doesn't affect the dynamics if and only if \eqref{eq:ko_simple} holds. Since condition \eqref{eq:ko_complex} implies \eqref{eq:ko_simple}, our proof of the reverse implication is complete. 
\qed %\\
%~\hfill \qedsymbol{}
\end{proof}
We hope this conceptual but physically grounded proof will help enhance the appeal of the original main result in \cite{Ayusawa14} for a broader range of readers. %This new proof implicitly rests on the fact that the floating-base rows of \eqref{eq:eom} embed the momentum dynamics for the system as a whole (c.f., \cite{Wieber2006, wensing2016improved}), which provide. %(e.g., leading to \eqref{eq:ko_simple}). %In this sense, the main results of \cite{Ayusawa08} can be understood as following from the fact that 
%the undetectable parameter changes are precisely those that do not affect the mass matrix. We note that  $\delta \vH_1 = 0$ implies that $\delta \vH = 0$. To see, this, note that at any given $\q$ since for any $j\ge i$
%\[
%\vH_{ij} = \vPhi_i\T \XMT{1}{i} \vH_{1j} 
%\]
% determines the rest of $\vH$ by kinematics, it follows that undetectable parameter changes are precisely those satisfying $\delta \vH_1(\vq)=0$ for all $\vq$. Equivalently
% \[
% \mathcal{N} =  \{\dpi \in \mathbb{R}^{10 \nbod}~|~ \delta \vH_1(\vq)\vqdd  = \bzero, ~\forall\,\vq,\vqdd \} \nonumber
% \]

% While Ayusawa et al.~\cite{Ayusawa14} focused on floating-base systems, our next set of results characterizes identifiability with a proof of correctness for fixed-base systems.
% Considering the first block row of $\vH$, it can be observed that this first row uniquely determines $\vH$ in whole. That is, the rest of $\vH$ can be computed from kinematics. 
% With this in mind, a powerful recent theoretical result from  may be understood perhaps more directly.
%\end{remark}
%
%

%\vspace{-10px}

\section{Generalization of RPNA to Address Rotors}
\label{app:rotors}

\newcommand{\gen}[1]{\underline{#1}}
\newcommand{\gV}{\gen{\vV}}
\newcommand{\gPhi}{\gen{\vPhi}}
\newcommand{\gpi}{\gen{\vpi}}
\newcommand{\gX}{\gen{\vX}}
\newcommand{\gB}{\gen{\vB}}
\newcommand{\gI}{\gen{\vI}}
\newcommand{\gk}{\gen{\vk}}
\newcommand{\gK}{\gen{\vK}}
\newcommand{\gA}{\gen{\vA}}
\newcommand{\gC}{\gen{\vC}}
\newcommand{\gN}{\gen{\vN}}

\newcommand{\gv}{\gen{\vv}}

The momentum and torque conditions \eqref{eq:momentum_rotor} and \eqref{eq:torque_rotor} are generalized here in the case of a serial chain:
\begin{align}
%\label{eq:momentum_rotor}
  0 &=
  \left[ \vPhi_i\T \dI_i \XM{i}{\pred{i}}(q_i) + \ngear{} \vPhi_{m_i}\T \dI_{m_i} \XM{m_i}{\pred{i}}(q_i) \right] \vv_{\pred{i}} \nonumber \\
  & \qquad \qquad \qquad \qquad \qquad ~~\forall {\vv_{\pred{i}} \in \mathcal{V}_{\pred{i}}^*}, q_i \in \mathbb{R}\\
  0 &= \vPhi_i\T \dI_i \vPhi_i + \vPhi_{m_i}\T \dI_{m_i} \vPhi_{m_i} \, \ngear{}^2
\end{align}
We can collect these two conditions together into a larger equation:
\begin{equation}
\scaleobj{.9}{
\mathbf{0} =\begin{bmatrix} \vPhi_i \\ \ngear{} \vPhi_{m_i} \end{bmatrix}\T \begin{bmatrix} \dI_i & \mathbf{0} \\ \mathbf{0} & \dI_{m_i} \end{bmatrix} \left( \begin{bmatrix} \XM{i}{\pred{i}} \\ \XM{m_i}{\pred{i}} \end{bmatrix} \vv_{\pred{i}} + \begin{bmatrix} \vPhi_i \\ \ngear{} \vPhi_{m_i} \end{bmatrix} \dot{q}_i \right)}
\end{equation}
for all $\vv_{\pred{i}}$ and $\dot{q}_i$. To simplify these conditions, we define:
\begin{align}
\gPhi_i &=\begin{bmatrix} \vPhi_i \\ \ngear{} \vPhi_{m_i} \end{bmatrix} \\[1ex]
\delta \gen{\vI}_i & = \begin{bmatrix} \dI_1 &  \mathbf{0} \\ \mathbf{0} & \dI_{m_i} \end{bmatrix}  \\[1ex] {}^i\gX_{\pred{i}}(q_i) &= \begin{bmatrix} \XM{i}{\pred{i}}(q_i) \\ \XM{m_i}{\pred{i}}(q_i) \end{bmatrix}
\end{align}
and assume a matrix $\gV_i$ so that
\begin{align}
{\textrm{Range}}(\gV_i) = \spn( \{ &{}^i \gX_{\pred{i}}(q_i)  \vv_{\pred{i}} + \gPhi_i \dot{q}_i  ~|  \\
 & ~~~~~~ \vv_{\pred{i}} \in \mathcal{V}_{\pred{i}}^*, q_i,\dot{q}_i \in \mathbb{R} \} ) 
\end{align}
With these definitions, we re-express the momentum and torque conditions in a combined form as:
\[
\gPhi{}_i\T \, \delta \gen{\vI}_i \, \gV_i = \mathbf{0}
\]
Additionally, we can recover a choice of $\vV_i$ as:
\begin{equation}
\vV_i = \begin{bmatrix} \mathbf{1}_{6\times6} & \mathbf{0}_{6\times6} \end{bmatrix} \gV_i
\label{eq:newV_rotor}
\end{equation}

A remaining question regards how to construct $\gV_i$, however, this can be done using controllability analysis similar to how $\vV_i$ was constructed in the main text. 
\begin{align*}
\wellBeforeJoint{\gV}{i} &= {}^i \gen{\mathbf X}_{\pred{i}}(0) \, \vV_{\pred{i}} \\
\beforeJoint{\gV}{i} &={\rm Ctrb}\left( \begin{bmatrix} (\vPhi_i \times) & \bzero \\ \bzero & \ngear{} (\vPhi_{m_i} \times) \end{bmatrix} , \, \wellBeforeJoint{\gV}{i} \, \right) \\
\gV_i &=  \begin{bmatrix} \beforeJoint{\gV}{i} & \gPhi_i \end{bmatrix}
\end{align*}
which motivates the definition:
\[
\gPhi_i\times =\begin{bmatrix} (\vPhi_i \times) & \bzero \\ \bzero & \ngear{} (\vPhi_{m_i} \times) \end{bmatrix}
\]

While we used the rotational symmetry of the rotor in the main text to simplify the mapping condition,  we can discharge that assumption to improve the generality of the RPNA with rotors. Specifically, we consider $\gpi \in \mathbb{R}^{20}$ to represent both rotor and link inertia parameters. We then extend the definitions in the main text using $\gV \in \mathbb{R}^{12\times m}$, $\gen{\vI}( \gpi) \in \mathbb{R}^{12\times 12}$,  $\gv\in \mathbb{R}^{12}$, $\gX  \in \mathbb{R}^{12\times 6}$ as:
\begin{align*}
\gC(\gV, \gPhi ) &:= \left[ \frac{\partial}{\partial \gpi} \gV\T\,  \gen{\vI} \, \gPhi \right]\T  \\
\gk(\gv) &:= \nabla_{\gpi} \, \frac{1}{2}\,\gv\T\,\gen{\vI}(\gpi) \, \gv 
 \\ 
\gB(\gX) &:= \frac{\partial}{\partial \gpi} \left[\gX\T \,\gen{\vI}(\gpi)\, \gX  \right]^\vee \\
\gA(\gPhi) &:= \frac{\partial }{\partial \gpi} \left[    (\gPhi\times)\T\, \gen{\vI}(\gpi) + \gen{\vI}(\gpi)\, (\gPhi \times)   \right]^\vee  %\label{eq:A_InertialParams}
\end{align*}
We consider a matrix $\gK_i$ so that 
\begin{align*}
\textrm{Range}(\gK_i) = \spn(\,\{ &\gk(\gv_i) ~|~ \gv_i = {}^i\gX_{\pred{i}}(q_i) \, \vv_{\pred{i}} + \gPhi_i \dot{q}_i ~| \nonumber \\
& ~ \vv_{\pred{i}}\in \mathcal{V}_{\pred{i}}^*, q_i, \dot{q}_i \in \mathbb{R} \} )
\end{align*}

With these definitions, the momentum and torque conditions take the form:
\[
\gC(\gV_i, \gPhi_i )\T \delta \gpi_i=0 
\]
while the mapping condition takes the form:
\[
\gK_i\T \, \vB( {}^i\gX_{\pred{i}}(0) ) \, \gA(\gPhi_i)^k \, \delta \gpi_i = \mathbf{0} , k=1,\ldots,10
\]
Mirroring the development in the main text, we can consider:
\begin{align*}
\wellBeforeJoint{\gK}{i} &= \gB( {}^i\gX_{\pred{i}}(0) )\!\T \,\vK_{\pred{i}} \\
\beforeJoint{\gK}{i} &= {\rm Ctrb}( \, \gA(\gPhi_i)\T ,\, \wellBeforeJoint{\gK}{i} \,) \\[1ex]
\gK_i &= \begin{bmatrix}  \beforeJoint{\gK}{i} & \gC( \gV_i , \gPhi_i )  \end{bmatrix} \\[1.5ex]
\gN_i &=  \begin{bmatrix} \gC( \gV_i , \gPhi_i )\T  \\[1ex]
   \beforeJoint{\gK}{i}\T\, \gA(\gPhi_i)  \end{bmatrix}
\end{align*}
where we then compute 
\begin{align}
\vK_i &= \begin{bmatrix} \mathbf{1}_{10\times 10} & \mathbf{0}_{10\times 10} \end{bmatrix} \gK_i \label{eq:newK_rotor} 
\end{align}
for the link alone. Aside from the extra steps \eqref{eq:newV_rotor} and \eqref{eq:newK_rotor}, the rest of these computations take the same structure as the RPNA in the main text, allowing us to quickly generalize its operation to accommodate single-DoF joints with rotors.

\begin{remark}
This approach can be directly extended to the case when $m$ bodies ($m\ge2$) move with $q_i$ as long as their velocities $\gv_i \in \mathbb{R}^{6m}$ can be represented in the form
\[
\gv_i = {}^i\gX_{\pred{i}}(q_i)  \vv_{\pred{i}} + \gPhi_i \dot{q}_i
\]
for some fixed $\gPhi_i $.
\end{remark}

% that's all folks
\end{document}